\documentclass[10pt]{article}

\usepackage{mystyle}
\usepackage{colortbl}
\RequirePackage[top=2.25cm, bottom=2.5cm, left=2.5cm, right=2.5cm, columnsep=0.65cm]{geometry}

\title{\LARGE Linear Programming Representations and Strongly Polynomial Algorithms for Robust Markov Decision Processes}
\author{
  Han Zhong\thanks{Shanghai Jiao Tong University. Email: \texttt{han.zhong@sjtu.edu.cn}.}
  \and
  Yinyu Ye\thanks{Shanghai Jiao Tong University, SIMIS, and Stanford University. Email: \texttt{yinyu-ye@stanford.edu}.}
}
\date{}

\begin{document}
\maketitle

\begin{abstract}
We study linear programming (LP) representations and strongly polynomial
algorithms for robust Markov decision processes (RMDPs) with rational
polyhedral state-action rectangular uncertainty in rewards and transitions.
By encoding a finite sequence of robust policy-iteration steps,
we construct a single LP whose optimal solutions recover
the robust optimal value and all optimal stationary randomized policies.
At fixed discount, the LP has polynomial dimension and encoding length
and can be constructed in strongly polynomial time.
We also develop a general complexity analysis of robust policy iteration
that combines the cost of minimizing over uncertainty sets with the
number of iterations needed to evaluate a policy.
For a fixed discount factor, we use this analysis to improve the known
complexity bounds for $\ell_1$ and $\ell_\infty$ RMDPs and establish
new strongly polynomial bounds for general interval, weighted $\ell_1$,
and Wasserstein RMDPs, as well as turn-based stochastic games with
these uncertainty sets.
\end{abstract}


\section{Introduction}
\label{sec:introduction}

Markov decision processes (MDPs) are a fundamental model of sequential
decision making, with applications in operations research, control,
and reinforcement learning. A policy chooses actions as the state
evolves, seeking to maximize the expected cumulative reward.
Although this return generally depends nonlinearly on the policy,
a finite discounted MDP admits a linear programming (LP)
representation \citep{Manne1960,Ye2011}. This LP determines the
optimal value and identifies the optimal actions, connecting dynamic
programming with linear optimization.

General LP algorithms therefore solve rational discounted MDPs in time
polynomial in the numbers of states and actions and the bit lengths of
the model parameters \citep{Schrijver1986}. This dependence on numerical
precision motivates the study of strongly polynomial algorithms, for
which the number of arithmetic operations and comparisons is bounded
by a polynomial in the numbers of states and actions alone.
For a fixed discount factor, \citet{Ye2005} develops such an algorithm
using an interior-point method that exploits the structure of the LP
for MDPs. \citet{Ye2011} establishes strongly polynomial bounds for
the simplex method with Dantzig's rule \citep{Dantzig1963} and Howard's
policy iteration \citep{Howard1960}.

The MDP formulation and the algorithms for solving it
assume that rewards and transition probabilities are known.
In applications, however, these parameters may be uncertain because of
noisy data or model misspecification. A policy optimized for a nominal
model can then perform poorly when the true model differs.
Robust Markov decision processes (RMDPs) address this uncertainty
by maximizing the worst-case expected return
over uncertainty sets for rewards and transitions
\citep{Iyengar2005,NilimElGhaoui2005}.

The additional worst-case optimization complicates the connection
with linear programming. Direct extensions of the LP for MDPs
through robust Bellman inequalities can produce nonconvex optimization
problems \citep{Iyengar2005,GrandClementPetrik2025}. This nonconvexity
does not rule out a different convex formulation, or even an LP
representation. \citet{GrandClementPetrik2025} obtain convex formulations
of entropy-regularized RMDPs through a change of variables.
These formulations approximate the original RMDP, leaving open the
question of a tractable exact formulation. This motivates our first question.
\begin{center}
\textbf{Q1:} Can a single LP exactly represent the optimal value and optimal policies
of an RMDP?
\end{center}

Recent work establishes strongly polynomial algorithms at fixed discount
for RMDPs with $\ell_1$ and $\ell_\infty$ uncertainty
\citep{AsadiEtAl2026,AsadiChatterjeeMontaseriShafiee2026}.
It remains to determine whether these complexity bounds can be sharpened
and whether such guarantees extend to broader uncertainty models,
such as weighted $\ell_1$ balls
\citep{HoPetrikWiesemann2018} and Wasserstein balls
\citep{RamaniGhate2022}.
\begin{center}
\textbf{Q2:} Can a unified analysis yield strongly polynomial guarantees
for broader classes of RMDPs?
\end{center}

\subsection{Our contributions}
\label{subsec:contributions}

We formulate a general polyhedral framework for RMDPs under state-action
rectangularity \citep{Iyengar2005,NilimElGhaoui2005}, with uncertainty in
rewards and transitions specified by rational linear constraints with auxiliary variables
(Section~\ref{subsec:general-sa-rectangular}).
The framework covers $\ell_1$ and weighted $\ell_1$ uncertainty,
interval uncertainty including weighted $\ell_\infty$ balls,
Wasserstein uncertainty, and turn-based stochastic games
(Examples~\ref{ex:halfspace-rmdp}--\ref{ex:turn-based-games}).
We address both questions through robust policy iteration
\citep{Iyengar2005}, which alternates between computing the worst-case
value of the current policy and updating its actions.
Our main contributions are as follows.

\begin{itemize}[leftmargin=*]
  \item \textbf{LP representation.}
  By encoding a finite sequence of policy evaluations and greedy
  improvements, we construct a single LP whose optimal solutions recover the
  robust optimal value and all optimal stationary randomized policies
  (Theorem~\ref{thm:exact-lp}). At fixed discount, the LP has $O(AN)$
  variables and constraints and total encoding length polynomial
  in the input length, and can be constructed directly from the input
  in strongly polynomial time.
  Here $A$ is the number of state-action pairs and $N$ is the
  input size defined in Section~\ref{sec:model}.
  To our knowledge, this gives the first strongly polynomial construction
  of an exact LP representation for this class. It answers the convex formulation
  question raised by \citet{GrandClementPetrik2025}
  without regularization or approximation.
  \looseness=-1 With this representation, general LP algorithms can solve RMDPs,
  and advances in linear programming can yield new RMDP algorithms
  and complexity bounds.

  \item \textbf{Strongly polynomial algorithms.}
  We develop a unified complexity analysis of robust policy iteration.
  We give conditions that bound the number of iterations
  required to evaluate a policy and combine these bounds with the cost
  of minimizing over the uncertainty sets. Under these conditions,
  Theorem~\ref{thm:residual-criterion} bounds the total number of arithmetic
  operations and comparisons for robust policy iteration on an RMDP
  with $S$ states at fixed discount by
  \begin{equation*}
    O\left((A-S+1)\left[(1+\mathsf G)(S^3+\mathsf C_{\rm in})
                       +\mathsf C_{\rm out}\right]\right).
  \end{equation*}
  Here, $\mathsf C_{\rm in}$ and $\mathsf C_{\rm out}$ bound the costs
  of minimizing over the uncertainty sets for one action per state
  and for all state-action pairs, respectively. The parameter
  $\mathsf G$ controls the number of iterations used to compute the
  worst-case value of a fixed policy.
  Polynomial bounds in $N$ on these three parameters therefore yield
  strongly polynomial algorithms. By deriving explicit bounds on these
  parameters, we improve
  the complexity bounds of \citet{AsadiChatterjeeMontaseriShafiee2026}
  for $\ell_1$ and $\ell_\infty$ RMDPs and obtain strongly polynomial
  guarantees for general interval uncertainty, including weighted $\ell_\infty$ balls
  (Corollary~\ref{cor:norm-planning}). We also obtain explicit strongly
  polynomial bounds for weighted $\ell_1$ and Wasserstein RMDPs
  (Corollary~\ref{cor:budget-models}). For turn-based stochastic games,
  we establish complexity bounds in terms of the sizes of their
  polyhedral descriptions (Theorem~\ref{thm:game-progress-criterion}).
  Table~\ref{tab:result-summary} summarizes the general and
  model-specific bounds.
\end{itemize}

We also establish representation and complexity boundaries. General LP
reduces to our RMDP model even at fixed discount
(Proposition~\ref{prop:general-lp-hardness}), so a strongly polynomial
algorithm for the full class would resolve the longstanding strongly
polynomial LP problem \citep[Problem~9]{Smale1998}.
When the discount is part of the input, we use the reduction of
\citet{SuilenPerez2026} to show that a polynomial-size exact LP
representation with polynomial-time construction and recovery of the
optimal value or an optimal policy would yield a polynomial-time
algorithm for parity games
(Corollary~\ref{prop:parity-boundary}), resolving a longstanding open
problem \citep{CaludeEtAl2017}. Under state rectangularity
\citep{WiesemannKuhnRustem2013},
a rational instance can have an irrational optimal value and a unique
irrational optimal policy, precluding rational affine recovery from a
rational LP (Proposition~\ref{prop:s-rectangular-irrationality}).
Figure~\ref{fig:result-overview} summarizes these results and their scope.

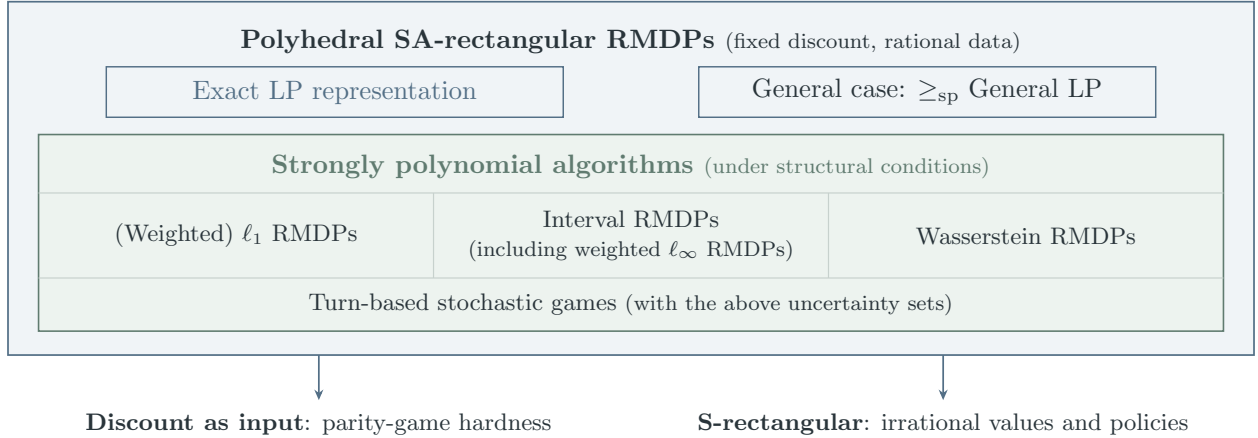
\begin{figure}[t]
  \centering
\begingroup
\usetikzlibrary{arrows.meta}
\definecolor{overviewInk}{HTML}{29343A}
\definecolor{overviewBlue}{HTML}{536F86}
\definecolor{overviewBlueFill}{HTML}{F0F4F7}
\definecolor{overviewGreen}{HTML}{607B70}
\definecolor{overviewGreenFill}{HTML}{EDF3EF}
\begin{tikzpicture}[
  x=1pt,y=1pt,
  every node/.style={text=overviewInk,font=\fontsize{9}{11}\selectfont,
    inner sep=0pt,align=center},
  overview arrow/.style={-{Stealth[length=4pt,width=3.4pt]},line width=.65pt},
  overview heading/.style={font=\fontsize{10}{12}\selectfont\bfseries},
  overview result/.style={draw=overviewBlue,line width=.55pt,
    minimum width=172pt,minimum height=19pt,
    font=\fontsize{10}{12}\selectfont}
]
\path[use as bounding box] (0,60) rectangle (470,231);

\draw[draw=overviewBlue,fill=overviewBlueFill,line width=.65pt]
  (0.5,94) rectangle (469.5,227);
\node[overview heading] at (235,212)
  {Polyhedral SA-rectangular RMDPs
   {\normalfont\fontsize{8}{10}\selectfont (fixed discount, rational data)}};
\node[overview result,text=overviewBlue] at (123.5,193)
  {Exact LP representation};

\node[overview result] at (346.5,193)
  {General case: $\ge_{\mathrm{sp}}$ General LP};

\draw[draw=overviewGreen,fill=overviewGreenFill,line width=.6pt]
  (12,103) rectangle (458,177);
\node[overview heading,text=overviewGreen] at (235,165)
  {Strongly polynomial algorithms
   {\normalfont\fontsize{8}{10}\selectfont (under structural conditions)}};
\draw[overviewGreen!35,line width=.45pt] (12,155) -- (458,155);
\draw[overviewGreen!35,line width=.45pt] (12,123) -- (458,123);
\draw[overviewGreen!35,line width=.45pt] (160.667,123) -- (160.667,155);
\draw[overviewGreen!35,line width=.45pt] (309.333,123) -- (309.333,155);
\node at (86.333,139) {(Weighted) $\ell_1$ RMDPs};
\node at (235,139) {Interval RMDPs\\
  {\fontsize{8}{10}\selectfont (including weighted $\ell_\infty$ RMDPs)}};
\node at (383.667,139) {Wasserstein RMDPs};
\node at (235,113) {Turn-based stochastic games
  {\normalfont\fontsize{8}{10}\selectfont (with the above uncertainty sets)}};

\draw[overview arrow,overviewBlue] (117.5,94) -- (117.5,77);
\draw[overview arrow,overviewBlue] (352.5,94) -- (352.5,77);
\node at (117.5,68) {\textbf{Discount as input}: parity-game hardness};
\node at (352.5,68) {\textbf{S-rectangular}: irrational values and policies};

\end{tikzpicture}
\endgroup

  \caption{We give an exact LP representation for the blue class and
  strongly polynomial algorithms for the green classes under structural
  conditions (Sections~\ref{sec:exact-lp}--\ref{sec:applications}).
  Here $\ge_{\mathrm{sp}}$ denotes a strongly polynomial reduction from
  general LP to RMDPs. The arrows summarize the boundaries in
  Section~\ref{sec:boundaries}.}
  \label{fig:result-overview}
\end{figure}

\subsection{Related work}
\label{subsec:related-work}

\paragraph{LP representation for MDPs and RMDPs.}
Classical LP representations underlie algorithms and complexity
results for MDPs \citep{Manne1960,PapadimitriouTsitsiklis1987}.
\citet{GrandClementPetrik2025} establish convex formulations for
entropy-regularized RMDPs with state or state-action rectangular
uncertainty.
\citet{SuilenPerez2026} use LP duality to obtain a polynomial-size LP
for evaluating a prescribed deterministic policy under state-action
rectangular polyhedral transition uncertainty.
For the more challenging task of finding optimal values and policies,
we give a single LP that represents the robust optimal value and all
optimal stationary randomized policies in our more general polyhedral
framework.
Encoding the steps of robust policy iteration in an LP relates our
construction to LP representations of dynamic
programs \citep{MartinRardinCampbell1990} and the
algorithm-to-LP constructions of \citet{AvisEtAl2019}.
At fixed discount, our LP representation has polynomial dimension and
can be constructed in strongly polynomial time.

\paragraph{Strongly polynomial algorithms for MDPs and RMDPs.}
For MDPs at fixed discount, \citet{Ye2005,Ye2011} establishes
strongly polynomial bounds for interior-point, simplex, and
policy-iteration methods.
When the discount is part of the input, the simplex method with
Dantzig's rule remains strongly polynomial for deterministic MDPs
\citep{PostYe2015}, whereas Howard's policy iteration can require
exponentially many iterations on MDPs with stochastic or deterministic
transitions
\citep{HollandersDelvenneJungers2012,ZhongYe2026}.
For MDPs with rational polyhedral action sets and rational linear reward
and transition functions, \citet{FercoqEtAl2013} establish polynomial-time
solvability.
In the robust setting, \citet{SuilenPerez2026} establish polynomial-time
solvability of polyhedral state-action rectangular RMDPs at fixed
discount by using their policy-evaluation LP within robust policy iteration.
For $\ell_1$ and $\ell_\infty$ uncertainty,
\citet{AsadiEtAl2026,AsadiChatterjeeMontaseriShafiee2026} establish
strongly polynomial bounds for robust policy iteration at fixed discount.
Our analysis improves the total arithmetic complexity at fixed discount
for $\ell_1$ and $\ell_\infty$ RMDPs and establishes strongly polynomial
guarantees for general interval uncertainty, including weighted
$\ell_\infty$ balls.
For weighted $\ell_1$ uncertainty, \citet{HoPetrikWiesemann2018}
develop fast exact Bellman updates.
\citet{HoPetrikWiesemann2021} develop partial policy iteration
with linear convergence guarantees.
\looseness=-1 We bound both the cost of computing minimizing transitions and the
number of iterations required to evaluate a policy, obtaining strongly
polynomial algorithms for weighted $\ell_1$ and Wasserstein RMDPs.

\paragraph{Distributionally robust optimization and learning in RMDPs.}
Distributionally robust optimization optimizes worst-case expected
performance over a set of probability distributions, as in the inventory
model of \citet{Scarf1958} and the moment-based formulations of
\citet{DelageYe2010}.
RMDPs apply this criterion to sequential decision making, seeking policies
that maximize worst-case expected cumulative rewards.
When the nominal model is unknown, robust reinforcement learning studies
how to learn these policies from data.
This learning problem has been studied with access to a generative model
\citep{YangZhangZhang2022,PanagantiKalathil2022,LiuEtAl2022,XuPanagantiKalathil2023,WangSiBlanchetZhou2024,ShiEtAl2026},
from offline data
\citep{ZhouEtAl2021,PanagantiEtAl2022,BlanchetEtAl2023,ShiChi2024,LiuXu2024Offline},
and through interactive data collection
\citep{LuEtAl2024,LiuXu2024Online}.
Our work studies LP representations and the arithmetic complexity of
solving RMDPs whose uncertainty sets are specified by rational linear
constraints.

\paragraph{Stochastic games.}
At fixed discount, \citet{HansenMiltersenZwick2013} establish strong
polynomiality of strategy iteration for finite-action turn-based
stochastic games. \citet{AkianGaubert2013} sharpen the bound on the
number of strategy updates.
\citet{JiaWenYe2020} develop simplex strategy-iteration variants
and establish their strong polynomiality at fixed discount.
\citet{MeiSunYe2026} obtain discount-independent strongly polynomial
bounds for deterministic turn-based games in which no directed cycle
contains states controlled by both players.
Our analysis gives a general arithmetic complexity bound for games
with polyhedral action sets and yields strongly polynomial
algorithms at fixed discount for games with $\ell_1$, interval,
weighted $\ell_1$, or Wasserstein transition uncertainty.

\paragraph{Notation.}
We use lowercase boldface letters for vectors and uppercase boldface
letters for matrices. All vectors are column vectors, and vector
inequalities are componentwise. We write $[\mathbf x]_i$ for the $i$-th
component of a vector and $[\mathbf M]_{ij}$ for the $(i,j)$ entry of a
matrix. For index sets $\mathcal I,\mathcal J$,
$[\mathbf x]_{\mathcal I}$ denotes the subvector on $\mathcal I$, and
$[\mathbf M]_{\mathcal I,\mathcal J}$ denotes the submatrix with rows in
$\mathcal I$ and columns in $\mathcal J$.
A colon in place of an index set selects all indices.
For a finite set $\mathcal B$, we write
$\Delta(\mathcal B)$ for its probability simplex.  For a positive integer
$n$, we write $[n]=\{1,\ldots,n\}$.
We write $\mathbf0$ and $\mathbf1$ for the zero and all-one vectors,
$\mathbf e_i$ for the $i$-th standard basis vector, and $\mathbf I$ for
the identity matrix, with dimensions determined by context.

\section{Preliminaries}
\label{sec:model}

\subsection{Polyhedral SA-rectangular RMDPs}
\label{subsec:general-sa-rectangular}

We represent a discounted RMDP by the tuple
$(\mathcal S,\{\mathcal A_s\}_{s\in\mathcal S},\gamma,
\{\mathcal Q_{sa}\}_{s,a})$.
Here $\mathcal S=[S]$ is the state space, $\mathcal A_s$ is the finite,
nonempty action set available at state $s$, and $\gamma\in[0,1)$ is the
discount factor.  We write $A=\sum_s|\mathcal A_s|$ for the total number
of state-action pairs.  The set $\mathcal Q_{sa}$ describes the joint
uncertainty in the reward and transition probabilities at $(s,a)$ through
the extended linear formulation
\begin{equation}
  \mathcal Q_{sa}
  =\left\{\mathbf q=(r,\mathbf p,\mathbf u):
    \mathbf D_{sa}\mathbf q\le\mathbf b_{sa},\quad
    \mathbf E_{sa}\mathbf q=\mathbf f_{sa}\right\}.
  \label{eq:local-ambiguity}
\end{equation}
Let $n_{sa}$ and $m_{sa}$ denote the numbers of variables and
constraints, respectively.
The vector $\mathbf q\in\mathbb R^{n_{sa}}$ consists of the expected
one-period reward $r$, the transition distribution $\mathbf p\in\mathbb R^S$,
and auxiliary variables $\mathbf u\in\mathbb R^{n_{sa}-S-1}$.
The rational matrices $\mathbf D_{sa}$ and $\mathbf E_{sa}$ specify the inequality
and equality constraints, with respective rational right-hand-side vectors
$\mathbf b_{sa}$ and $\mathbf f_{sa}$, where
$(\mathbf b_{sa}^{\top},\mathbf f_{sa}^{\top})^{\top}\in\mathbb R^{m_{sa}}$.
We denote the set of admissible reward-transition pairs at $(s,a)$ by
\begin{equation*}
  \mathcal M_{sa}
  :=\{(r,\mathbf p):\exists\,\mathbf u,
                      (r,\mathbf p,\mathbf u)\in\mathcal Q_{sa}\}.
\end{equation*}
We assume that $\mathcal M_{sa}$ is a nonempty bounded subset of
$\mathbb R\times\Delta(\mathcal S)$.  As a polyhedral projection,
$\mathcal M_{sa}$ is closed and therefore compact.  The auxiliary variables
$\mathbf u$ may be unbounded even for a fixed pair $(r,\mathbf p)$.
The local constraints may couple rewards and transitions. For example,
the reward may be an affine function of the transition probabilities,
so changes in the two need not be independent.

\paragraph{Policies and robust values.}
A stationary randomized policy $\pi=(\boldsymbol\pi_s)_{s\in\mathcal S}$ assigns an action
distribution $\boldsymbol\pi_s\in\Delta(\mathcal A_s)$ to each state $s$.
The component $[\boldsymbol\pi_s]_a$ is the probability of choosing action $a$ in state $s$.
A policy is deterministic if each $\boldsymbol\pi_s$ assigns probability one to a single action.
We denote this action at state $s$ by $\pi(s)$.
We assume state-action rectangularity, meaning that the global ambiguity
set is the product of the local sets $\mathcal M_{sa}$
\citep{Iyengar2005,NilimElGhaoui2005}.
We define the robust value of $\pi$ by
\begin{equation*}
  [\mathbf v^\pi]_s
  :=\min_{\substack{(r_{ia},\mathbf p_{ia})\in\mathcal M_{ia}\\
                   \forall i\in\mathcal S,\ a\in\mathcal A_i}}
       \mathbb E\left[
         \sum_{t=0}^{\infty}\gamma^t r_{s_ta_t}
         \,\middle|\,s_0=s
       \right].
\end{equation*}
The minimum is over models with fixed reward-transition pairs for all
state-action pairs. Under each model, $a_t$ is drawn from
$\boldsymbol\pi_{s_t}$, and $s_{t+1}$ is drawn from $\mathbf p_{s_ta_t}$.
We write
$[\mathbf v^\star]_s=\sup_{\pi}[\mathbf v^{\pi}]_s$
for the optimum over stationary randomized policies.
We define the robust Bellman operators by
\begin{equation}
  \mathscr T_{sa}(\mathbf v)
  :=\min_{(r,\mathbf p,\mathbf u)\in\mathcal Q_{sa}}
       \{r+\gamma\mathbf p^\top\mathbf v\},
  \qquad
  \mathscr T_s(\mathbf v)
  :=\max_{a\in\mathcal A_s}\mathscr T_{sa}(\mathbf v).
  \label{eq:bellman-operator}
\end{equation}
The minimum is attained because the objective depends only on
$(r,\mathbf p)$ and $\mathcal M_{sa}$ is compact.
We define $\mathscr T$ by
$[\mathscr T(\mathbf v)]_s=\mathscr T_s(\mathbf v)$ for every $s\in\mathcal S$.
The robust value $\mathbf v^{\pi}$ is
the unique fixed point of the policy operator
\begin{equation}
  [\mathscr T_{\pi}(\mathbf v)]_s
  =\sum_{a\in\mathcal A_s}[\boldsymbol\pi_s]_a\mathscr T_{sa}(\mathbf v),
  \qquad
  \mathbf v^{\pi}
  =\mathscr T_{\pi}(\mathbf v^{\pi}).
  \label{eq:policy-operator}
\end{equation}
The robust optimal value $\mathbf v^\star$ is the unique solution of
$\mathbf v^\star=\mathscr T(\mathbf v^\star)$
\citep{Iyengar2005}.
The set of policies optimal from every state is
\begin{equation}
  \Pi^\star=\prod_{s\in\mathcal S}\Delta(\mathcal A_s^\star),
  \qquad
  \mathcal A_s^\star
  =\{a\in\mathcal A_s:\mathscr T_{sa}(\mathbf v^\star)=[\mathbf v^\star]_s\}.
  \label{eq:optimal-actions}
\end{equation}

\paragraph{Input size and complexity.}
Using the variable counts $n_{sa}$ and constraint counts $m_{sa}$ in~\eqref{eq:local-ambiguity}, we define the \emph{input size} as
\begin{equation*}
  N:=S+A+\sum_{s,a}(n_{sa}+m_{sa}).
\end{equation*}
All input coefficients are rational and encoded by an integer numerator
and a positive integer denominator, both written in binary.
We write $L$ for the total input bit length, including the $L_\gamma$ bits
used to encode the discount factor $\gamma$.
We use the following complexity conventions.
\begin{itemize}
  \item \emph{Polynomial time.} The number of bit operations is bounded
  by a polynomial in $N$ and $L$.
  \item \emph{Strong polynomiality.} The number of rational arithmetic
  operations and comparisons is bounded by a polynomial in $N$ alone.
  Every intermediate rational number has bit length polynomial in $L$.
\end{itemize}

\subsection{Examples}
\label{subsec:model-examples}

The following uncertainty models and turn-based stochastic games
fall within the framework of Section~\ref{subsec:general-sa-rectangular}.
For simplicity, we present the uncertainty-set examples with fixed rewards
$r_{sa}$ and focus on transition uncertainty, so
$\mathcal M_{sa}=\{r_{sa}\}\times\mathcal U_{sa}$.
Analogous polyhedral sets for joint reward-transition uncertainty
also fit the representation in~\eqref{eq:local-ambiguity}.

\begin{example}[Polyhedral uncertainty, \citealp{WhiteEldeib1994}]
\label{ex:halfspace-rmdp}
For a rational matrix $\mathbf B_{sa}$ and vector
$\mathbf c_{sa}$, the polyhedral transition uncertainty set is
\begin{equation*}
  \mathcal U_{sa}
  =\{\mathbf p\in\Delta(\mathcal S):
       \mathbf B_{sa}\mathbf p\le\mathbf c_{sa}\}.
\end{equation*}
It is obtained from~\eqref{eq:local-ambiguity} by taking
$\mathcal Q_{sa}=\{(r_{sa},\mathbf p):\mathbf p\in\mathcal U_{sa}\}$,
with no auxiliary variables.
\end{example}

\begin{example}[$\ell_1$ uncertainty, \citealp{Iyengar2005}]
\label{ex:norm-ambiguity}
\label{ex:tv-ambiguity}
An $\ell_1$ uncertainty set consists of transition distributions
within a prescribed distance of a nominal distribution
$\overline{\mathbf p}_{sa}\in\Delta(\mathcal S)$ and is given by
\begin{equation}
  \mathcal U_{sa}
  =\{\mathbf p\in\Delta(\mathcal S):
       \|\mathbf p-\overline{\mathbf p}_{sa}\|_1\le2\sigma_{sa}\}.
  \label{eq:norm-ambiguity}
\end{equation}
Here $\sigma_{sa}\ge0$ is the total-variation radius, so the $\ell_1$
radius is $2\sigma_{sa}$.
The corresponding reward-transition set has the extended linear formulation
\begin{equation}
  \mathcal Q_{sa}
  =\left\{(r_{sa},\mathbf p,\mathbf u):
    \mathbf p\in\Delta(\mathcal S),\quad \mathbf u\ge\mathbf0,\quad
    {-}\mathbf u\le\mathbf p-\overline{\mathbf p}_{sa}\le\mathbf u,\quad
    \mathbf1^\top\mathbf u\le2\sigma_{sa}\right\}.
  \label{eq:tv-lift}
\end{equation}
\end{example}

\begin{example}[Weighted $\ell_1$ uncertainty, \citealp{HoPetrikWiesemann2018,HoPetrikWiesemann2021}]
\label{ex:weighted-l1-ambiguity}
\label{ex:budget-models}
For a nominal distribution
$\overline{\mathbf p}_{sa}\in\Delta(\mathcal S)$, weights $\mathbf w_{sa}\ge\mathbf0$,
and budget $\rho_{sa}\ge0$,
the weighted $\ell_1$ set is
\begin{equation}
  \mathcal U_{sa}
  =\left\{\mathbf p\in\Delta(\mathcal S):
       \sum_{i\in\mathcal S}[\mathbf w_{sa}]_i|[\mathbf p]_i-[\overline{\mathbf p}_{sa}]_i|
       \le\rho_{sa}\right\}.
  \label{eq:weighted-l1-set}
\end{equation}
We obtain $\mathcal Q_{sa}$ by
replacing the last constraint in~\eqref{eq:tv-lift} with
$\mathbf w_{sa}^\top\mathbf u\le\rho_{sa}$.
\end{example}

\begin{example}[Interval and weighted $\ell_\infty$ uncertainty, \citealp{GivanLeachDean1997,GivanLeachDean2000}]
\label{ex:interval-ambiguity}
Interval uncertainty places lower and upper bounds
$\mathbf l_{sa}$ and $\mathbf h_{sa}$ on the transition probabilities, giving
\begin{equation}
  \mathcal U_{sa}
  =\{\mathbf p\in\Delta(\mathcal S):
       \mathbf l_{sa}\le\mathbf p\le\mathbf h_{sa}\}.
  \label{eq:interval-ambiguity}
\end{equation}
The bounds satisfy
$\mathbf0\le\mathbf l_{sa}\le\mathbf h_{sa}\le\mathbf1$ and
$\mathbf1^\top\mathbf l_{sa}\le1\le\mathbf1^\top\mathbf h_{sa}$.
In particular, the weighted $\ell_\infty$ ball
$\max_{i\in\mathcal S}[\mathbf w_{sa}]_i
|[\mathbf p]_i-[\overline{\mathbf p}_{sa}]_i|\le\delta_{sa}$, with nominal
distribution $\overline{\mathbf p}_{sa}\in\Delta(\mathcal S)$,
weights $\mathbf w_{sa}\ge\mathbf0$, and radius
$\delta_{sa}\ge0$, is obtained by setting
\begin{equation*}
  [\mathbf l_{sa}]_i=\max\left\{0,[\overline{\mathbf p}_{sa}]_i-
    \frac{\delta_{sa}}{[\mathbf w_{sa}]_i}\right\},
  \qquad [\mathbf h_{sa}]_i=\min\left\{1,[\overline{\mathbf p}_{sa}]_i+
    \frac{\delta_{sa}}{[\mathbf w_{sa}]_i}\right\},
  \quad [\mathbf w_{sa}]_i>0.
\end{equation*}
For $[\mathbf w_{sa}]_i=0$, we set
$[\mathbf l_{sa}]_i=0$ and $[\mathbf h_{sa}]_i=1$.
The choice $\mathbf w_{sa}=\mathbf1$ gives the $\ell_\infty$ ball.
\end{example}

\begin{example}[Wasserstein transition uncertainty, \citealp{RamaniGhate2022,NeufeldSester2024}]
\label{ex:wasserstein-ambiguity}
For a metric $d$ on $\mathcal S$ and $\ell\ge1$,
the $\ell$-Wasserstein distance between
$\overline{\mathbf p},\mathbf p\in\Delta(\mathcal S)$ \citep{PeyreCuturi2019}
and its uncertainty ball are defined by
\begin{equation*}
  \mathbb W_{\ell,d}(\overline{\mathbf p},\mathbf p)^\ell
  :=\min_{\boldsymbol\Pi\ge0,\;\boldsymbol\Pi\mathbf1=\overline{\mathbf p},\;
        \boldsymbol\Pi^\top\mathbf1=\mathbf p}
      \langle\mathbf C,\boldsymbol\Pi\rangle,
  \qquad
  \mathcal U_{sa}
  =\{\mathbf p\in\Delta(\mathcal S):
       \mathbb W_{\ell,d}(\overline{\mathbf p}_{sa},\mathbf p)^\ell\le\rho_{sa}\},
\end{equation*}
where $[\mathbf C]_{ij}=d(i,j)^\ell$ are rational transport costs,
$\overline{\mathbf p}_{sa}\in\Delta(\mathcal S)$ is the nominal distribution,
and $\rho_{sa}\ge0$ is the budget.
The corresponding reward-transition set has the extended linear formulation
\begin{equation}
  \mathcal Q_{sa}
  =\left\{(r_{sa},\mathbf p,\boldsymbol\Pi):
    \mathbf p\in\Delta(\mathcal S),\quad \boldsymbol\Pi\ge0,\quad
    \boldsymbol\Pi\mathbf1=\overline{\mathbf p}_{sa},\quad
    \boldsymbol\Pi^\top\mathbf1=\mathbf p,\quad
    \langle\mathbf C,\boldsymbol\Pi\rangle\le\rho_{sa}
  \right\}.
  \label{eq:finite-wasserstein-lift}
\end{equation} 
Here $[\boldsymbol\Pi]_{ij}$ is the mass transported from state $i$
to state $j$.
The entries of $\boldsymbol\Pi$ form
$\mathbf u$ in~\eqref{eq:local-ambiguity}.
\end{example}

\begin{example}[Turn-based stochastic games{, \citealp{Shapley1953}}]
\label{ex:turn-based-games}
The state space is partitioned into maximizing and
minimizing states, $\mathcal S=\mathcal S_{\max}\cup\mathcal S_{\min}$.
At state $s$, the acting player chooses an action
$a$ from a finite set $\mathcal A_s$, with rational reward-transition
pair $(r_{sa},\mathbf p_{sa})$.
The players respectively maximize and minimize the
expected total discounted reward, with rational discount factor $\gamma\in[0,1)$.

This game fits the polyhedral RMDP framework of
Section~\ref{subsec:general-sa-rectangular}. At maximizing states,
the RMDP has the same actions, with singleton uncertainty set
$\{(r_{sa},\mathbf p_{sa})\}$ for each action $a$. At each minimizing state $s$, it has
a single action with extended uncertainty set
\begin{equation*}
  \mathcal Q_s
  =\left\{(r,\mathbf p,\mathbf u):
     \mathbf u\in\Delta(\mathcal A_s),\quad
     (r,\mathbf p)=\sum_{a\in\mathcal A_s}
        [\mathbf u]_a(r_{sa},\mathbf p_{sa})\right\}.
\end{equation*}
Here $\mathbf u$ gives the action probabilities.
By linearity, minimizing the Bellman objective over
$\mathcal Q_s$ is equivalent to minimizing over the original actions.
The RMDP has the same Bellman operator and value as the game.
\end{example}

\section{LP representation for RMDPs}
\label{sec:exact-lp}

We aim to construct a single rational LP whose optimal face represents the
robust optimal value and the full set of optimal stationary randomized
policies.  The usual Bellman inequalities do not immediately give such
an LP.  Since the minimum defining $\mathscr T_{sa}(\mathbf v)$ is attained,
\begin{equation}
\begin{aligned}
 [\mathbf v]_s\ge\mathscr T_s(\mathbf v)
 &\iff \forall a\in\mathcal A_s\ \exists(r,\mathbf p,\mathbf u)\in\mathcal Q_{sa}:
       [\mathbf v]_s\ge r+\gamma\mathbf p^\top\mathbf v,\\
 [\mathbf v]_s\le\mathscr T_s(\mathbf v)
 &\iff \exists a\in\mathcal A_s\ \forall(r,\mathbf p,\mathbf u)\in\mathcal Q_{sa}:
       [\mathbf v]_s\le r+\gamma\mathbf p^\top\mathbf v.
\end{aligned}
 \label{eq:bellman-quantifiers}
\end{equation}
In the first line, we must choose a reward-transition pair for each
action. Both $\mathbf p$ and $\mathbf v$ then become decision variables,
making $\mathbf p^\top\mathbf v$ bilinear.
In the second line, fixing an action $a$ leaves one linear inequality
in $\mathbf v$ for each fixed $(r,\mathbf p,\mathbf u)\in\mathcal Q_{sa}$.
LP duality replaces this family of inequalities with finitely many
linear constraints. At each state, however, the constraints for at
least one action must hold, so the feasible value vectors form a union
of polyhedra, one for each combination of actions across states.
This union need not be convex.

One possible construction is to enumerate all deterministic stationary
policies and combine their evaluation LPs into a single LP.
This requires one copy for each of the $\prod_s|\mathcal A_s|$
deterministic policies, so the number of copies can grow exponentially
with the number of states. The following theorem gives an LP
representation with polynomially many variables and constraints at
fixed discount. For general $\gamma$, we define
\begin{equation}
  T_\gamma:=\min\{k\in\mathbb Z_{\ge1}:\gamma^k<(1-\gamma)/2\}.
  \label{eq:discount-block}
\end{equation}
For fixed $\gamma$, $T_\gamma$ is constant, and more generally,
\begin{equation}
  T_\gamma=O\left(1+\frac{1}{1-\gamma}
                         \log\frac{2}{1-\gamma}\right).
  \label{eq:discount-horizon}
\end{equation}
For $0<\gamma<1$, this follows from
$T_\gamma\le1+\frac{\log(2/(1-\gamma))}{-\log\gamma}$ and
$-\log\gamma\ge1-\gamma$. At $\gamma=0$, $T_\gamma=1$.
\begin{theorem}[Exact LP representation of RMDPs]
\label{thm:exact-lp}
Under the standing input model, one can construct a continuous rational LP
with value variables $\mathbf v\in\mathbb R^S$, policy variables
$\pi\in\prod_s\Delta(\mathcal A_s)$, and auxiliary
variables.  Its feasible region is nonempty and bounded, and it has a
single linear objective.  Its optimal face $\mathcal F$ satisfies
\begin{equation}
  \operatorname{proj}_{(\mathbf v,\pi)}\mathcal F
  =\{\mathbf v^\star\}\times\Pi^\star.
  \label{eq:lp-pi-projection}
\end{equation}
The LP has $O(ANT_\gamma)$ variables and constraints
and total encoding length polynomial in $N,L,T_\gamma$.
It can be constructed using polynomially many arithmetic operations and
comparisons in $N,T_\gamma$, with all intermediate numbers having bit
length polynomial in $N,L,T_\gamma$.
The LP is given explicitly in~\eqref{eq:lp-pi-objective}.
\end{theorem}

\begin{proof}[Proof Sketch]
\looseness=-1 The construction represents a finite sequence of policy evaluations
and the updates~\eqref{eq:controller-improvement} within one LP.
The policy values and action
choices are determined when the LP is solved, while every coefficient is
computed from the input.  Section~\ref{subsec:lp-pi-system}
establishes the required iteration bound and gives the LP construction.
Section~\ref{subsec:lp-pi-exactness} proves that the single weighted
objective recovers the robust optimal value and all optimal stationary
randomized policies, and establishes the bounds on size,
encoding length, and construction time.
\end{proof}

Theorem~\ref{thm:exact-lp} applies to the general uncertainty model
in~\eqref{eq:local-ambiguity}, allowing joint reward-transition
uncertainty specified by rational extended formulations.
The size bound uses the numbers of variables and constraints in these
descriptions.
To our knowledge, this is the first explicit exact LP representation for
this class of RMDPs with
polynomial dimension and a strongly polynomial
construction at fixed discount.
Its optimal solutions recover the robust optimal value and all optimal stationary
randomized policies without regularization or
approximation.  For this class of RMDPs, it answers the convex
formulation question raised by \citet{GrandClementPetrik2025}. In particular, the theorem covers the RMDPs in
Examples~\ref{ex:halfspace-rmdp}--\ref{ex:wasserstein-ambiguity} and the
turn-based stochastic games represented as RMDPs in
Example~\ref{ex:turn-based-games}.

The LP representation makes general-purpose LP
algorithms applicable to these RMDPs (and turn-based stochastic games). Advances in linear programming
can therefore yield new RMDP algorithms and improved complexity bounds.
At fixed discount, the
LP has $O(AN)$ variables and constraints and polynomial total encoding length,
so standard polynomial-time LP algorithms solve the RMDP exactly.
The LP can also be constructed using
polynomially many arithmetic operations and comparisons in $N$,
with all intermediate numbers having bit length polynomial in $N,L$
(Appendix~\ref{app:lp-encoding}).
Thus, for example, a strongly polynomial algorithm for this
family of LPs would yield a strongly polynomial
algorithm for these RMDPs.

\subsection{LP construction via robust policy iteration}
\label{subsec:controller-pi}
\label{subsec:lp-pi-system}

Algorithm~\ref{alg:exact-policy-iteration} gives robust policy
iteration \citep{Iyengar2005}, combining exact policy evaluation with
the update
\begin{equation}
  \pi^{k+1}(s)\in
  \operatorname*{arg\,max}_{a\in\mathcal A_s}
         \mathscr T_{sa}(\mathbf v^{\pi^k}),
  \qquad \forall s\in\mathcal S.
  \label{eq:controller-improvement}
\end{equation}
\begin{algorithm}[t]
\caption{Robust policy iteration}
\label{alg:exact-policy-iteration}
\begin{algorithmic}[1]
  \STATE Choose a deterministic initial policy $\pi^0$.
  \FOR{$k=0,1,\ldots$}
    \STATE Compute the policy value $\mathbf v^{\pi^k}$.
      \label{line:policy-evaluation}
    \STATE Compute $\mathscr T_{sa}(\mathbf v^{\pi^k})$ for every
      $s\in\mathcal S$ and $a\in\mathcal A_s$ by solving the minimization
      in~\eqref{eq:bellman-operator}.
      \label{line:policy-bellman-lps}
    \STATE \textbf{if} $\mathscr T(\mathbf v^{\pi^k})=\mathbf v^{\pi^k}$
      \textbf{then return} $\pi^k$ and $\mathbf v^{\pi^k}$.
      \label{line:policy-optimality-check}
    \STATE Choose $\pi^{k+1}$ according to~\eqref{eq:controller-improvement},
      using the values from Line~\ref{line:policy-bellman-lps}.
      \label{line:policy-improvement}
  \ENDFOR
\end{algorithmic}
\end{algorithm}

\begin{lemma}[Robust policy-iteration bound]
\label{lem:outer-pi}
Algorithm~\ref{alg:exact-policy-iteration} returns $\mathbf v^\star$
and an optimal deterministic policy after at most $(A-S)T_\gamma$ updates.
It performs at most $1+(A-S)T_\gamma$ policy evaluations and
$A[1+(A-S)T_\gamma]$ minimizations defining $\mathscr T_{sa}$ in~\eqref{eq:bellman-operator}.
\end{lemma}

\begin{proof}
We use the standard action-elimination argument of \citet{Ye2011},
applied to robust Bellman operators.
See also \citet{AsadiChatterjeeMontaseriShafiee2026} for a robust
policy-iteration bound of the same order.
We include the full proof for completeness and to make explicit the
iteration bound used in the LP construction.
We fix a deterministic policy $\pi$ and obtain $\pi'$ from $\pi$
by~\eqref{eq:controller-improvement}. Then
$\mathscr T_\pi\le\mathscr T$ and
$\mathscr T_{\pi'}(\mathbf v^\pi)=\mathscr T(\mathbf v^\pi)
\ge\mathbf v^\pi$.  Monotonicity and contraction give
\begin{equation}
  \mathbf v^\pi
  \le\mathscr T(\mathbf v^\pi)
  \le\mathbf v^{\pi'}\le\mathbf v^\star,
  \qquad
  \|\mathbf v^\star-\mathbf v^{\pi'}\|_\infty
  \le\gamma\|\mathbf v^\star-\mathbf v^\pi\|_\infty.
  \label{eq:outer-contraction}
\end{equation}
For each state-action pair, we define the optimality gap
$g_{sa}=[\mathbf v^\star]_s-\mathscr T_{sa}(\mathbf v^\star)\ge0$.
The policy fixed-point identity, monotonicity, and the $\gamma$-Lipschitz
property give
\[
  g_{s,\pi(s)}
  \le[\mathbf v^\star]_s-[\mathbf v^\pi]_s
  \le g_{s,\pi(s)}
       +\gamma\|\mathbf v^\star-\mathbf v^\pi\|_\infty.
\]
Taking maxima over states yields
\begin{equation}
  \max_s g_{s,\pi(s)}
  \le\|\mathbf v^\star-\mathbf v^\pi\|_\infty
  \le\frac{\max_s g_{s,\pi(s)}}{1-\gamma}.
  \label{eq:outer-gap-sandwich}
\end{equation}
At a suboptimal iteration $k$, we choose a state $s$ maximizing
$g_{s,\pi^k(s)}$ and set $a=\pi^k(s)$.  Then $g_{sa}>0$, and for
every $h\ge T_\gamma$, the definition~\eqref{eq:discount-block} gives
\[
  \|\mathbf v^\star-\mathbf v^{\pi^{k+h}}\|_\infty
  \le\gamma^h\|\mathbf v^\star-\mathbf v^{\pi^k}\|_\infty
  \le\frac{\gamma^{T_\gamma}}{1-\gamma}g_{sa}
  <\frac{g_{sa}}{2}.
\]
Thus $\pi^{k+h}(s)\ne a$, since using $a$ at state $s$ would imply
$\|\mathbf v^\star-\mathbf v^{\pi^{k+h}}\|_\infty\ge g_{sa}$.
Each block of $T_\gamma$ updates in~\eqref{eq:controller-improvement}
permanently eliminates a
distinct positive-gap action.  There are at most $A-S$ such actions,
since every state has a zero-gap action.  Once the policy uses only
zero-gap actions,~\eqref{eq:outer-gap-sandwich} gives
$\mathbf v^\pi=\mathbf v^\star$, proving the update bound and hence
the stated computation counts.
\end{proof}

We now construct a single LP that represents the robust policy-iteration
sequence of Algorithm~\ref{alg:exact-policy-iteration} through iteration
$AT_\gamma$, with the same fixed initial policy.
We continue the update~\eqref{eq:controller-improvement} after
optimality is reached. Since this update preserves optimality,
Lemma~\ref{lem:outer-pi} guarantees that the value at the final
iteration is $\mathbf v^\star$.

\paragraph{Step 1. Dual representation of $\mathscr T_{sa}$.}
We recall from~\eqref{eq:bellman-operator} that
$\mathscr T_{sa}(\mathbf v)$ is the minimum of
$r+\gamma\mathbf p^\top\mathbf v$ over
$\mathbf q=(r,\mathbf p,\mathbf u)\in\mathcal Q_{sa}$.
To use its dual representation in the LP, we introduce positive
integers $V,M$.
The bound $V$ applies to the values $\mathbf v$ and
$\mathscr T_{sa}(\mathbf v)$, and optimal dual multipliers can be chosen with every coordinate
bounded in absolute value by $M$. The integers $V$ and $M$ can be constructed directly from
the input and have polynomial bit length
(see Appendix~\ref{app:lp-dual-bounds}).
The vectors $\mathbf v^\pi$ and $\mathbf v^\star$ belong to
$[-V,V]^S$ for every policy $\pi$, and
\begin{equation}
  \mathscr T_{sa}(\mathbf v)\in[-V,V],
  \qquad \forall\mathbf v\in[-V,V]^S.
  \label{eq:lp-value-invariance}
\end{equation}
For $\mathbf v\in[-V,V]^S$, the primal problem $(\mathrm P)$ below
computes $\mathscr T_{sa}(\mathbf v)$. Its dual has the
same optimal value by strong duality. The choice of $M$ guarantees
an optimal dual solution within the displayed bounds, so imposing
these bounds in the dual problem $(\mathrm D)$ preserves that value.
\begin{equation}
 \begin{alignedat}[t]{3}
   &(\mathrm P)\quad &&{\min\nolimits_{\mathbf q}}
   &\quad&r+\gamma\mathbf p^\top\mathbf v\\
   &&&\text{\rm s.t.}
   &\quad&\left\{\begin{aligned}
     &\mathbf D_{sa}\mathbf q\le\mathbf b_{sa},\\
     &\mathbf E_{sa}\mathbf q=\mathbf f_{sa}.
   \end{aligned}\right.
 \end{alignedat}
 \qquad
 \begin{alignedat}[t]{3}
   &(\mathrm D)\quad
   &&{\max\nolimits_{\boldsymbol\mu_{sa},\boldsymbol\nu_{sa}}}
   &\quad&-\mathbf b_{sa}^{\top}\boldsymbol\mu_{sa}
           +\mathbf f_{sa}^{\top}\boldsymbol\nu_{sa}\\
   &&&\text{\rm s.t.}
   &\quad&\left\{\begin{aligned}
     &{-}\mathbf D_{sa}^{\top}\boldsymbol\mu_{sa}
          +\mathbf E_{sa}^{\top}\boldsymbol\nu_{sa}
       =(1,\gamma\mathbf v,\mathbf0),\\
     &\mathbf0\le\boldsymbol\mu_{sa}\le M\mathbf1,
       \qquad -M\mathbf1\le\boldsymbol\nu_{sa}\le M\mathbf1.
   \end{aligned}\right.
 \end{alignedat}
  \label{eq:lp-uniform-dual}
\end{equation}
\paragraph{Step 2. Representing the policy Bellman equation.}
For each iteration $k=0,\ldots,AT_\gamma$, we combine the dual
constraints in $(\mathrm D)$ of~\eqref{eq:lp-uniform-dual}, using a
shared variable $\mathbf v^k$ and dual variables
$\boldsymbol\mu^k_{sa}$ and $\boldsymbol\nu^k_{sa}$ for each $(s,a)$.
For a fixed deterministic policy
$\pi^k\in\prod_{s\in\mathcal S}\Delta(\mathcal A_s)$, we solve
$\mathbf v^k=\mathscr T_{\pi^k}(\mathbf v^k)$ and compute
$[\mathbf t^k]_{sa}=\mathscr T_{sa}(\mathbf v^k)$ through the following
lexicographic LP.
\begin{subequations}
\label{eq:lp-pi-block}
\newcommand{\policyblocklayout}[1]{%
  \ensuremath{\begin{aligned}[b]
    &\LPopt{lex\,max}{\mathbf v^k,\mathbf t^k,\boldsymbol\mu^k_{sa},\boldsymbol\nu^k_{sa}}\quad
      \left(\sum_s[\mathbf v^k]_s,\ \sum_{s,a}[\mathbf t^k]_{sa}\right)\\
    &#1
  \end{aligned}}%
}
\begin{empheq}[
  box=\policyblocklayout,
  left={\makebox[\widthof{$\operatorname{lex\,max}$}][l]{$\mathrm{s.t.}$}\quad\empheqlbrace}
]{align}
 &{-}\mathbf D_{sa}^{\top}\boldsymbol\mu^k_{sa}
   +\mathbf E_{sa}^{\top}\boldsymbol\nu^k_{sa}
   =(1,\gamma\mathbf v^k,\mathbf0),
   \label{eq:lp-pi-dual}\\
 &[\mathbf t^k]_{sa}\le-\mathbf b_{sa}^{\top}\boldsymbol\mu^k_{sa}
   +\mathbf f_{sa}^{\top}\boldsymbol\nu^k_{sa},
   \label{eq:lp-pi-action-bound}\\
 &[\mathbf v^k]_s\le[\mathbf t^k]_{sa}
   +2V(1-[\boldsymbol\pi_s^k]_a),
   \label{eq:lp-pi-activation}\\
 &{-}V\le[\mathbf v^k]_s,\ [\mathbf t^k]_{sa}\le V,
   \label{eq:lp-pi-value-bounds}\\
 &\mathbf0\le\boldsymbol\mu^k_{sa}\le M\mathbf1,
   \qquad -M\mathbf1\le\boldsymbol\nu^k_{sa}\le M\mathbf1.
   \label{eq:lp-pi-multiplier-bounds}
\end{empheq}
\end{subequations}
The constraints apply to every $s\in\mathcal S$ and $a\in\mathcal A_s$.
Here $\operatorname{lex\,max}$ means that we first maximize
$\sum_s[\mathbf v^k]_s$ and then, among the resulting maximizers,
maximize $\sum_{s,a}[\mathbf t^k]_{sa}$.
Constraints~\eqref{eq:lp-pi-dual} and~\eqref{eq:lp-pi-multiplier-bounds}
ensure that $(\boldsymbol\mu^k_{sa},\boldsymbol\nu^k_{sa})$ is feasible for
$(\mathrm D)$ in~\eqref{eq:lp-uniform-dual} at $\mathbf v^k$.
Together with~\eqref{eq:lp-pi-action-bound}, weak duality gives
$[\mathbf t^k]_{sa}
 \le-\mathbf b_{sa}^{\top}\boldsymbol\mu^k_{sa}
     +\mathbf f_{sa}^{\top}\boldsymbol\nu^k_{sa}
 \le\mathscr T_{sa}(\mathbf v^k)$.
For the selected action, $[\boldsymbol\pi_s^k]_a=1$, so~\eqref{eq:lp-pi-activation} gives
$[\mathbf v^k]_s\le[\mathbf t^k]_{sa}$. We obtain
$\mathbf v^k\le\mathscr T_{\pi^k}(\mathbf v^k)
 \le\mathscr T_{\pi^k}^{2}(\mathbf v^k)\le\cdots
 \longrightarrow\mathbf v^{\pi^k}$,
by monotonicity, $\gamma$-contraction, and the fixed-point
identity~\eqref{eq:policy-operator}.

To show that this upper bound is attainable, we set
$\mathbf v^k=\mathbf v^{\pi^k}\in[-V,V]^S$ and choose
an optimal solution $(\boldsymbol\mu^k_{sa},\boldsymbol\nu^k_{sa})$
of $(\mathrm D)$ in~\eqref{eq:lp-uniform-dual} for every $(s,a)$.
For these optimal solutions, strong duality in~\eqref{eq:lp-uniform-dual} gives
$-\mathbf b_{sa}^{\top}\boldsymbol\mu^k_{sa}
 +\mathbf f_{sa}^{\top}\boldsymbol\nu^k_{sa}
 =\mathscr T_{sa}(\mathbf v^{\pi^k})$, so setting
$[\mathbf t^k]_{sa}=\mathscr T_{sa}(\mathbf v^{\pi^k})$
makes~\eqref{eq:lp-pi-action-bound} hold at equality.
The bounds~\eqref{eq:lp-pi-value-bounds} follow from~\eqref{eq:lp-value-invariance}, and the policy Bellman equation~\eqref{eq:policy-operator} makes~\eqref{eq:lp-pi-activation} hold at
equality for selected actions.
For every unselected action, $[\boldsymbol\pi_s^k]_a=0$ and~\eqref{eq:lp-pi-value-bounds} gives
$[\mathbf v^k]_s-[\mathbf t^k]_{sa}\le2V$, so~\eqref{eq:lp-pi-activation} holds automatically.
Every feasible $\mathbf v^k$ satisfies
$\mathbf v^k\le\mathbf v^{\pi^k}$, and any strict inequality
in a coordinate gives
$\sum_s[\mathbf v^k]_s<\sum_s[\mathbf v^{\pi^k}]_s$.
Since $\mathbf v^k=\mathbf v^{\pi^k}$ is feasible,
maximizing $\sum_s[\mathbf v^k]_s$ forces this equality.
With $\mathbf v^k$ fixed, optimal dual multipliers in~\eqref{eq:lp-uniform-dual} simultaneously attain all bounds
$[\mathbf t^k]_{sa}\le\mathscr T_{sa}(\mathbf v^k)$.
The same argument for maximizing $\sum_{s,a}[\mathbf t^k]_{sa}$ forces
$[\mathbf t^k]_{sa}=\mathscr T_{sa}(\mathbf v^k)$ for every $(s,a)$.

\paragraph{Step 3. Identifying maximizing actions.}
Step~2 fixes $[\mathbf t^k]_{sa}
=\mathscr T_{sa}(\mathbf v^{\pi^k})$ for every $(s,a)$.
The policy update~\eqref{eq:controller-improvement} therefore selects
an action attaining $\max_{a\in\mathcal A_s}[\mathbf t^k]_{sa}$ at each
state. To identify all such actions, we introduce a positive rational
number $\varepsilon$ as a
uniform lower bound on positive gaps, so that every deterministic
policy $\pi$ satisfies
\begin{equation}
 \max_{a'\in\mathcal A_s}
       \mathscr T_{sa'}(\mathbf v^\pi)
       -\mathscr T_{sa}(\mathbf v^\pi)
       \ \in\ \{0\}\cup[\varepsilon,\infty),
 \qquad \forall s\in\mathcal S,\ a\in\mathcal A_s.
 \label{eq:lp-policy-gap}
\end{equation}
The constant $\varepsilon$ can be constructed directly from the input
and has polynomial bit length (see Appendix~\ref{app:lp-action-separation}).
With $\mathbf t^k$ fixed, we introduce $\mathbf w^k\in\mathbb R^S$
so that $[\mathbf w^k]_s$ records
$\max_{a\in\mathcal A_s}[\mathbf t^k]_{sa}$ at each state, and
$\mathbf z^k\in\mathbb R^A$ so that $[\mathbf z^k]_{sa}$ equals one
if $[\mathbf t^k]_{sa}=[\mathbf w^k]_s$ and zero otherwise.
The following lexicographic LP gives these values.
\begin{equation}
 \begin{alignedat}{2}
   &\LPopt{lex\,min}{\mathbf w^k,\mathbf z^k}
   &\quad&\left(\sum_s[\mathbf w^k]_s,\ \sum_{s,a}[\mathbf z^k]_{sa}\right)\\
   &\text{\rm s.t.}
   &\quad&\left\{\begin{aligned}
     &[\mathbf w^k]_s\ge[\mathbf t^k]_{sa},
       \qquad -V\le[\mathbf w^k]_s\le V,\\
     &[\mathbf z^k]_{sa}\ge
       1-\frac{[\mathbf w^k]_s-[\mathbf t^k]_{sa}}{\varepsilon},
       \qquad 0\le[\mathbf z^k]_{sa}\le1.
   \end{aligned}\right.
 \end{alignedat}
 \label{eq:lp-pi-flag}
\end{equation}
The constraints apply to every $s\in\mathcal S$ and $a\in\mathcal A_s$.
Every feasible $\mathbf w^k$ satisfies
$[\mathbf w^k]_s\ge\max_{a\in\mathcal A_s}[\mathbf t^k]_{sa}$.
These lower bounds are attained simultaneously with $\mathbf z^k=\mathbf1$,
so minimizing $\sum_s[\mathbf w^k]_s$ forces equality at every state.
For a maximizing action, the gap
$[\mathbf w^k]_s-[\mathbf t^k]_{sa}$ is zero, so~\eqref{eq:lp-pi-flag} forces $[\mathbf z^k]_{sa}=1$.
For every other action,~\eqref{eq:lp-policy-gap} makes the gap at least
$\varepsilon$ because $\pi^k$ is deterministic.
Hence $1-([\mathbf w^k]_s-[\mathbf t^k]_{sa})/\varepsilon\le0$,
so $[\mathbf z^k]_{sa}=0$ is feasible simultaneously for all such actions.
With $\mathbf w^k$ fixed, minimizing $\sum_{s,a}[\mathbf z^k]_{sa}$ gives
$[\mathbf z^k]_{sa}=1$ for maximizing actions and
$[\mathbf z^k]_{sa}=0$ for all other actions.

\paragraph{Step 4. Policy improvement and the final policy.}
For $k<AT_\gamma$, we use $\mathbf z^k$ from Step~3 to select a
deterministic policy $\pi^{k+1}$ satisfying the update~\eqref{eq:controller-improvement}, so that Step~2 applies at the next
iteration. To resolve ties, we let
$\omega(s,a)\in\{1,\ldots,|\mathcal A_s|\}$ denote the position of
action $a$ in the input ordering at state $s$.
With $\mathbf z^k$ fixed, we solve the LP
\begin{equation}
 \begin{alignedat}{2}
   &\LPopt{min}{\pi^{k+1}}
   &\quad&\sum_{s,a}\omega(s,a)[\boldsymbol\pi_s^{k+1}]_a\\
   &\text{\rm s.t.}
   &\quad&\sum_{a\in\mathcal A_s}[\boldsymbol\pi_s^{k+1}]_a=1,
     \qquad 0\le[\boldsymbol\pi_s^{k+1}]_a\le[\mathbf z^k]_{sa},
     \qquad \forall s\in\mathcal S,\ a\in\mathcal A_s.
 \end{alignedat}
 \label{eq:lp-pi-next}
\end{equation}
The constraints allow any probability distribution over the maximizing
actions at each state. Since the indices $\omega(s,a)$ are distinct
at each state, the unique optimum places probability one on
the maximizing action with the smallest $\omega(s,a)$.
At iteration $AT_\gamma$, we introduce the final policy $\pi$ with constraints
\begin{equation}
 \sum_{a\in\mathcal A_s}[\boldsymbol\pi_s]_a=1,
 \qquad 0\le[\boldsymbol\pi_s]_a\le[\mathbf z^{AT_\gamma}]_{sa},
 \qquad \forall s\in\mathcal S,\ a\in\mathcal A_s.
 \label{eq:lp-pi-terminal}
\end{equation}
The final policy $\pi$ appears in none of the objectives, so
every distribution over the maximizing actions at the final iteration
is permitted. Since $\mathbf v^{AT_\gamma}=\mathbf v^\star$ by
Lemma~\ref{lem:outer-pi}, these actions are exactly
$\mathcal A_s^\star$ from~\eqref{eq:optimal-actions}.
Thus~\eqref{eq:lp-pi-terminal} permits exactly the policies in
$\Pi^\star=\prod_s\Delta(\mathcal A_s^\star)$.

\paragraph{Step 5. The complete LP.}
We write the full variable vector as
\begin{equation*}
 \mathbf x=\Bigl(
   \{\mathbf v^k,\mathbf t^k,\mathbf w^k,\mathbf z^k\}_{k=0}^{AT_\gamma},
   \{\boldsymbol\mu^k_{sa},\boldsymbol\nu^k_{sa}\}_%
     {\substack{k=0,\ldots,AT_\gamma\\s\in\mathcal S,\ a\in\mathcal A_s}},
   \{\pi^k\}_{k=1}^{AT_\gamma},\pi
 \Bigr).
\end{equation*}
For each $k=0,\ldots,AT_\gamma-1$, we write the objectives in~\eqref{eq:lp-pi-block},~\eqref{eq:lp-pi-flag}, and~\eqref{eq:lp-pi-next}
as maximizations in the order in which they are optimized,
\begin{equation}
 \bigl(f_{5k+1}(\mathbf x),\ldots,f_{5k+5}(\mathbf x)\bigr)
 =\left(\sum_s[\mathbf v^k]_s,\ \sum_{s,a}[\mathbf t^k]_{sa},\
        -\sum_s[\mathbf w^k]_s,\ -\sum_{s,a}[\mathbf z^k]_{sa},\
        -\sum_{s,a}\omega(s,a)[\boldsymbol\pi_s^{k+1}]_a\right).
 \label{eq:lp-pi-priorities}
\end{equation}
For $k=AT_\gamma$, we use only the first four definitions.
We maximize $f_1,\ldots,f_{5AT_\gamma+4}$ in this order, each over the
maximizers of all preceding objectives. This completes Steps~2--4
for each iteration before proceeding to the next.
We compute a positive integer $W$ directly from the input
(see Appendix~\ref{app:lp-scalarization}) so that the weighted objective
below has exactly these lexicographic maximizers. With the initial policy
$\pi^0$ fixed as in Algorithm~\ref{alg:exact-policy-iteration},
the complete LP is
\begin{equation}
  \begin{alignedat}{2}
    &\LPopt{max}{\mathbf x}
    &\quad&\sum_{j=1}^{5AT_\gamma+4}W^{5AT_\gamma+4-j}f_j(\mathbf x)\\
    &\text{\rm s.t.}
    &\quad&\left\{\begin{aligned}
      &\text{constraints in~\eqref{eq:lp-pi-block},~\eqref{eq:lp-pi-flag}},
        &&\forall k=0,\ldots,AT_\gamma,\\
      &\text{constraints in~\eqref{eq:lp-pi-next}},
        &&\forall k=0,\ldots,AT_\gamma-1,\\
      &\text{constraints~\eqref{eq:lp-pi-terminal}}.
    \end{aligned}\right.
  \end{alignedat}
  \label{eq:lp-pi-objective}
\end{equation}
We denote the feasible region of~\eqref{eq:lp-pi-objective} by
$\mathcal P$ and its optimal face by $\mathcal F$.
We prove below that on $\mathcal F$, $\mathbf v^{AT_\gamma}=\mathbf v^\star$
and $\pi$ ranges over exactly $\Pi^\star$, giving~\eqref{eq:lp-pi-projection} with $\mathbf v=\mathbf v^{AT_\gamma}$.

\subsection{Proof of Theorem~\ref{thm:exact-lp}}
\label{subsec:lp-pi-exactness}

\begin{proof}[Proof of Theorem~\ref{thm:exact-lp}]
Steps~2--4 in Section~\ref{subsec:lp-pi-system},
applied along the complete robust policy-iteration
sequence with the tie rule in Step~4, give a feasible
point of~\eqref{eq:lp-pi-objective}.
The feasible point constructed above and the explicit bounds show that
$\mathcal P$ is a nonempty bounded rational polytope.
We apply Steps~2--4 inductively in the objective order.
At each stage, this point attains the bound on the current objective
together with all preceding optimal objective values.
Every maximizer at that stage therefore satisfies the corresponding
equality conditions.
Lemma~\ref{lem:outer-pi} and Step~3 then imply that every lexicographic
maximizer satisfies $\mathbf v^{AT_\gamma}=\mathbf v^\star$ and
$[\mathbf z^{AT_\gamma}]_{sa}=\mathbf1\{a\in\mathcal A_s^\star\}$.

For a linear objective over $\mathcal P$, the optimal solutions are
exactly the convex combinations of its maximizing vertices.
This also holds for lexicographic maximization by applying the same
fact successively. Hence it suffices to show that the weighted
objective in~\eqref{eq:lp-pi-objective} and lexicographic maximization
of~\eqref{eq:lp-pi-priorities} have the same maximizing vertices.
For two vertices $\mathbf x,\mathbf y$ whose objective values differ,
we let $j$ be the first index with $f_j(\mathbf x)\ne f_j(\mathbf y)$.
The sign of this difference determines their lexicographic order.
We choose $W$ so that the remaining objectives cannot reverse this
comparison for any pair of vertices. We use a positive integer
$D_{\mathrm{LP}}$, computable from the input, such that the coordinates
of each vertex share a denominator at most $D_{\mathrm{LP}}$
(see Appendix~\ref{app:lp-scalarization}). Since each $f_i$ has
integer coefficients, the nonzero difference
$f_j(\mathbf x)-f_j(\mathbf y)$ has denominator at most
$D_{\mathrm{LP}}^2$, giving
$|f_j(\mathbf x)-f_j(\mathbf y)|\ge D_{\mathrm{LP}}^{-2}$.
The objective definitions~\eqref{eq:lp-pi-priorities}, variable bounds,
and simplex constraints give
$|f_i(\mathbf x)|,|f_i(\mathbf y)|\le AV$
(see~\eqref{eq:lp-priority-range}), so
$|f_i(\mathbf x)-f_i(\mathbf y)|\le2AV$ for every $i$.
With these bounds, we take $W=2AVD_{\mathrm{LP}}^2+2$.
Dividing the weighted objective difference by the positive weight
$W^{5AT_\gamma+4-j}$ leaves
$f_j(\mathbf x)-f_j(\mathbf y)$ as the first nonzero term.
The remaining terms satisfy
\[
 \left|\sum_{i=j+1}^{5AT_\gamma+4}
   W^{j-i}\bigl(f_i(\mathbf x)-f_i(\mathbf y)\bigr)\right|
 \le\frac{2AV}{W-1}<D_{\mathrm{LP}}^{-2}
 \le|f_j(\mathbf x)-f_j(\mathbf y)|.
\]
The bound shows that the remaining terms cannot change the comparison
at the first differing objective. If two vertices have the same value
for every objective, their weighted objective values are also equal.
Thus the two formulations have the same maximizing vertices and hence
the same optimal solution set.
At every optimal solution of~\eqref{eq:lp-pi-objective},
$\mathbf v^{AT_\gamma}=\mathbf v^\star$, and~\eqref{eq:lp-pi-terminal}
permits exactly $\pi\in\Pi^\star$.
Since $\pi$ appears in no other constraint or objective,
every such policy occurs, proving
the projection identity~\eqref{eq:lp-pi-projection}.

Finally, counting the variables and constraints
in~\eqref{eq:lp-pi-objective} over its $AT_\gamma+1$ iterations gives
\begin{equation}
  \#\text{variables}\le 3(AT_\gamma+1)N,\qquad
  \#\text{inequalities}\le 10(AT_\gamma+1)N,\qquad
  \#\text{equalities}\le (AT_\gamma+1)N.
  \label{eq:lp-pi-counts}
\end{equation}
Counting each equality as two inequalities gives $O(ANT_\gamma)$
variables and constraints. Appendices~\ref{app:lp-scalarization}
and~\ref{app:lp-encoding} bound the total LP encoding length and
the bit lengths of all intermediate rationals by polynomials in $N,L,T_\gamma$.
Appendix~\ref{app:lp-encoding} also shows that construction uses
polynomially many arithmetic operations and comparisons in $N,T_\gamma$.
Together with the bit-length bounds, this gives construction time
polynomial in $N,L,T_\gamma$.
This proves Theorem~\ref{thm:exact-lp}.
\end{proof}

\section{A general complexity bound for robust policy iteration}
\label{sec:strong-polynomial}
\label{subsec:residual-criterion}

Theorem~\ref{thm:exact-lp} gives an LP representation of polynomial
dimension at fixed discount. We next ask when the RMDP can be solved
in strongly polynomial time. Without restrictions on the uncertainty
descriptions, a strongly polynomial algorithm for computing the robust
value, even in a one-state model, would also solve general LPs in
strongly polynomial time.

\begin{proposition}[Reduction from general linear programming]
\label{prop:general-lp-hardness}
Suppose a strongly polynomial algorithm computes the robust optimal
value for every RMDP in the model of Section~\ref{sec:model} with one
state, one action, and \(\gamma=\tfrac12\). Then general rational LPs
can be solved in strongly polynomial time.
\end{proposition}

\begin{proof}
For any nonempty bounded rational polyhedron
$\mathcal X=\{\mathbf x\in\mathbb R^d:
\mathbf A\mathbf x\ge\mathbf b\}$ defined by $\mathbf A\in\mathbb Q^{m\times d}$
and $\mathbf b\in\mathbb Q^m$,
and objective vector $\mathbf c\in\mathbb Q^d$, we construct an RMDP
with state space $\mathcal S=\{s_1\}$, action set
$\mathcal A_{s_1}=\{a_1\}$, and discount $\gamma=\tfrac12$.
Action $a_1$ returns to $s_1$ with probability one, while its reward
ranges over $\{\frac{\mathbf c^\top\mathbf x}{2}:
\mathbf x\in\mathcal X\}$.
Adjoining $r=\frac{\mathbf c^\top\mathbf x}{2}$ and $[\mathbf p]_{s_1}=1$
to the description of $\mathcal X$ gives a rational extended
uncertainty set with compact reward-transition projection. The RMDP
has robust value
\begin{equation}
  [\mathbf v^\star]_{s_1}
  =\min_{\mathbf x\in\mathbb R^d}
    \{\mathbf c^\top\mathbf x:\mathbf A\mathbf x\ge\mathbf b\}.
  \label{eq:general-lp-embedding}
\end{equation}
By~\eqref{eq:general-lp-embedding}, the RMDP algorithm computes
optimal values of linear functions over nonempty bounded rational
polyhedra in strongly polynomial time.
Appendix~\ref{app:general-lp-reduction} shows how these computations yield
a strongly polynomial algorithm for general rational LPs.
\end{proof}

By Proposition~\ref{prop:general-lp-hardness}, a strongly
polynomial algorithm for the general model would resolve the longstanding
open problem of strongly polynomial linear programming
\citep[Problem~9]{Smale1998}. We therefore analyze the cost of
Algorithm~\ref{alg:exact-policy-iteration} to identify structural conditions
on the uncertainty sets that ensure strong polynomiality.

Lemma~\ref{lem:outer-pi} bounds the number of policy
evaluations in Algorithm~\ref{alg:exact-policy-iteration}.
Each iteration computes the value $\mathbf v^\pi$ of its current
deterministic policy $\pi$ in Line~\ref{line:policy-evaluation}
and solves the minimization problems in~\eqref{eq:bellman-operator} at
$\mathbf v^\pi$ for all $A$ state-action pairs in
Line~\ref{line:policy-bellman-lps}.
Lines~\ref{line:policy-optimality-check} and~\ref{line:policy-improvement}
reuse these values and require $O(A)$ comparisons.
To compute $\mathbf v^\pi$, we consider a minimization MDP
whose action set at state $s$ is $\mathcal M_{s,\pi(s)}$ and whose
optimal value is $\mathbf v^\pi$, following
the policy-evaluation interpretation of \citet{HoPetrikWiesemann2021}.
A choice $\tau$ of one reward-transition pair at each state determines a reward vector
$\mathbf r^\tau$, a transition matrix $\mathbf P^\tau$, and
the value $\mathbf u^\tau$ satisfying
$(\mathbf I-\gamma\mathbf P^\tau)\mathbf u^\tau=\mathbf r^\tau$.
Algorithm~\ref{alg:nature-policy-iteration} applies Howard's
policy iteration \citep{Howard1960} to this MDP.
We choose the initial pairs $\tau^0$ from minimizers
in~\eqref{eq:bellman-operator}, with $a=\pi(s)$ and
$\mathbf v=\mathbf0$ at each state $s$.
Each iteration of Algorithm~\ref{alg:nature-policy-iteration}
then computes $\mathbf u^{\tau^t}$ by solving the linear system
in Line~\ref{line:pair-value} and updates the pairs
in Line~\ref{line:pair-update} by solving the $S$ minimization
problems in~\eqref{eq:nature-improvement}.

\begin{algorithm}[t]
\caption{Robust policy evaluation by policy iteration}
\label{alg:nature-policy-iteration}
\begin{algorithmic}[1]
  \REQUIRE A deterministic policy $\pi$ and initial pairs $\tau^0$,
    one from each $\mathcal M_{s,\pi(s)}$, with reward vector
    $\mathbf r^{\tau^0}$ and transition matrix $\mathbf P^{\tau^0}$.
  \FOR{$t=0,1,\ldots$}
    \STATE Solve
      $(\mathbf I-\gamma\mathbf P^{\tau^t})\mathbf u^{\tau^t}
        =\mathbf r^{\tau^t}$.
      \label{line:pair-value}
    \STATE Choose $\tau^{t+1}$ satisfying
      \label{line:pair-update}
      \begin{equation}
        \bigl([\mathbf r^{\tau^{t+1}}]_s,
               [\mathbf P^{\tau^{t+1}}]_{s,:}^{\top}\bigr)
        \in\operatorname*{arg\,min}_{(r,\mathbf p)\in\mathcal M_{s,\pi(s)}}
             \{r+\gamma\mathbf p^\top\mathbf u^{\tau^t}\},
        \quad \forall s\in\mathcal S.
        \label{eq:nature-improvement}
      \end{equation}
    \STATE \textbf{if}
      $\mathbf r^{\tau^{t+1}}+
       \gamma\mathbf P^{\tau^{t+1}}\mathbf u^{\tau^t}
       =\mathbf u^{\tau^t}$
      \textbf{then return} $\mathbf u^{\tau^t}$ and $\tau^{t+1}$.
      \label{line:pair-optimality-check}
  \ENDFOR
\end{algorithmic}
\end{algorithm}

\paragraph{Cost and encoding assumptions.}
For a given $\mathbf v\in\mathbb Q^S$, we solve
\begin{equation}
  \min_{(r,\mathbf p)\in\mathcal M_{sa}}
    \{r+\gamma\mathbf p^\top\mathbf v\},
  \qquad s\in\mathcal S,\ a\in\mathcal A_s,
  \label{eq:bellman-minimizers}
\end{equation}
and return a minimizing reward-transition pair and the value
$\mathscr T_{sa}(\mathbf v)$.
For these computations, we use the following upper bounds on
arithmetic operations and comparisons, valid for every rational
$\mathbf v$.
\begin{enumerate}[leftmargin=*]
  \item $\mathsf C_{\rm in}\ge S$ bounds the total cost for the $S$
  pairs $(s,\pi(s))$ of any deterministic policy $\pi$.
  These computations occur in Algorithm~\ref{alg:nature-policy-iteration},
  Line~\ref{line:pair-update} and initialization, as part of the policy
  evaluation in Algorithm~\ref{alg:exact-policy-iteration},
  Line~\ref{line:policy-evaluation}.
  \item $\mathsf C_{\rm out}\ge A$ bounds the total cost for all $A$
  state-action pairs, as required in
  Algorithm~\ref{alg:exact-policy-iteration},
  Line~\ref{line:policy-bellman-lps}.
\end{enumerate}
\looseness=-1 Both cost bounds include all computations needed to obtain the
minimizing pairs and objective values at a common $\mathbf v$.
We require the returned pairs to have encoding length polynomial
in $L$, independently of $\mathbf v$.
The objective values and intermediate rationals must have bit length polynomial in $L$ and the encoding length of $\mathbf v$.

\begin{theorem}[Complexity of robust policy iteration]
\label{thm:residual-criterion}
Let $\gamma\in[0,1)$ be rational. Assume the cost
and encoding conditions following~\eqref{eq:bellman-minimizers}, with computation
bounds $\mathsf C_{\rm in}$ and $\mathsf C_{\rm out}$.
For a run of Algorithm~\ref{alg:nature-policy-iteration} with a deterministic
policy $\pi$ and initial pairs selected at $\mathbf v=\mathbf0$, we define
$E_t=\|\mathbf u^{\tau^t}-\mathbf v^\pi\|_\infty$ at iteration $t$
and set $E_t=0$ after termination.
Suppose there is an integer $\mathsf G\ge0$ such that, for every
deterministic policy $\pi$, each such run admits a nonnegative sequence
$(g_t)_{t\ge0}$ satisfying
\begin{equation}
  \left\{\begin{aligned}
    &g_t\le E_t,\qquad
      E_{t+1}\le\frac{\gamma g_t}{1-\gamma},\quad \forall t\ge0,\\
    &\left|\{\lfloor\log_2 g_t\rfloor:t\ge0,\ g_t>0\}\right|\le \mathsf G.
  \end{aligned}\right.
  \label{eq:unified-scale-bound}
\end{equation}
Algorithm~\ref{alg:exact-policy-iteration}, with
Line~\ref{line:policy-evaluation} implemented by
Algorithm~\ref{alg:nature-policy-iteration}, returns $\mathbf v^\star$,
all sets $\mathcal A_s^\star$, an optimal deterministic stationary
policy, and a minimizing reward-transition pair at $\mathbf v^\star$
for every $(s,a)$, using
\begin{equation}
  O\left(
    [1+(A-S)T_\gamma]
    \left[(1+\mathsf G T_\gamma)(S^3+\mathsf C_{\rm in})
          +\mathsf C_{\rm out}\right]
  \right)
  \label{eq:residual-scale-total-cost}
\end{equation}
arithmetic operations and comparisons. Intermediate rationals have
bit length polynomial in $L$.
\end{theorem}

\begin{proof}[Proof of Theorem~\ref{thm:residual-criterion}]
We first bound the number of iterations in
Algorithm~\ref{alg:nature-policy-iteration} for a fixed deterministic
policy $\pi$. We then combine this bound with Lemma~\ref{lem:outer-pi}
to obtain the total operation count.

\vspace{5pt}
\noindent\textbf{Contraction and residual bounds.}
For a feasible choice $\tau$ of reward-transition pairs, we define
\begin{equation}
  \mathbf d^\tau
    =\mathbf r^\tau+\gamma\mathbf P^\tau\mathbf v^\pi
                         -\mathbf v^\pi.
  \label{eq:nature-residual}
\end{equation}
Since $\mathscr T_\pi(\mathbf v^\pi)=\mathbf v^\pi$ and the
pairs in $\tau$ are feasible, $\mathbf d^\tau\ge\mathbf0$.
The value equation for $\mathbf u^\tau$ and~\eqref{eq:nature-residual}
give $(\mathbf I-\gamma\mathbf P^\tau)(\mathbf u^\tau-\mathbf v^\pi)
=\mathbf d^\tau$.
Since $\mathbf P^\tau$ is stochastic and $\gamma<1$, expanding
$(\mathbf I-\gamma\mathbf P^\tau)^{-1}$ as a convergent geometric series gives
\begin{equation}
  \mathbf0\le\mathbf d^\tau
  \le\mathbf u^\tau-\mathbf v^\pi
   =\sum_{j=0}^{\infty}\gamma^j(\mathbf P^\tau)^j\mathbf d^\tau
  \le\frac{\|\mathbf d^\tau\|_\infty}{1-\gamma}\mathbf1.
  \label{eq:residual-sandwich}
\end{equation}
The lower bound follows because the $j=0$ term is $\mathbf d^\tau$
and all remaining terms are nonnegative. For the upper bound, each
component of $(\mathbf P^\tau)^j\mathbf d^\tau$ is a convex combination
of the entries of $\mathbf d^\tau$, hence at most $\|\mathbf d^\tau\|_\infty$.
Summing the coefficients $\gamma^j$ gives the factor $1/(1-\gamma)$.
Taking infinity norms yields
\begin{equation*}
  \|\mathbf d^\tau\|_\infty
  \le\|\mathbf u^\tau-\mathbf v^\pi\|_\infty
  \le\frac{\|\mathbf d^\tau\|_\infty}{1-\gamma}.
\end{equation*}
At iteration $t$ of Algorithm~\ref{alg:nature-policy-iteration},
the minimizing property in~\eqref{eq:nature-improvement} and the
value equation for $\mathbf u^{\tau^t}$ give
\begin{equation}
  \mathbf r^{\tau^{t+1}}+\gamma\mathbf P^{\tau^{t+1}}\mathbf u^{\tau^t}
    \le\mathbf r^{\tau^t}+\gamma\mathbf P^{\tau^t}\mathbf u^{\tau^t}
    =\mathbf u^{\tau^t}.
  \label{eq:nature-improvement-comparison}
\end{equation}
Rearranging~\eqref{eq:nature-improvement-comparison} gives
$\mathbf u^{\tau^t}
\ge(\mathbf I-\gamma\mathbf P^{\tau^{t+1}})^{-1}\mathbf r^{\tau^{t+1}}
=\mathbf u^{\tau^{t+1}}$.
Using~\eqref{eq:nature-improvement}, we then have
\begin{equation*}
  \mathbf u^{\tau^{t+1}}
  =\mathbf r^{\tau^{t+1}}+\gamma\mathbf P^{\tau^{t+1}}\mathbf u^{\tau^{t+1}}
  \le\mathbf r^{\tau^{t+1}}+\gamma\mathbf P^{\tau^{t+1}}\mathbf u^{\tau^t}
  =\mathscr T_\pi(\mathbf u^{\tau^t}).
\end{equation*}
Together with $\mathbf u^{\tau^{t+1}}\ge\mathbf v^\pi$
from~\eqref{eq:residual-sandwich} and
$\mathscr T_\pi(\mathbf v^\pi)=\mathbf v^\pi$,
the $\gamma$-contraction property of $\mathscr T_\pi$ gives
\begin{equation}
  \|\mathbf u^{\tau^{t+1}}-\mathbf v^\pi\|_\infty
  \le\|\mathscr T_\pi(\mathbf u^{\tau^t})
        -\mathscr T_\pi(\mathbf v^\pi)\|_\infty
  \le\gamma\|\mathbf u^{\tau^t}-\mathbf v^\pi\|_\infty.
  \label{eq:nature-contraction}
\end{equation}
Combining~\eqref{eq:residual-sandwich} and~\eqref{eq:nature-contraction} also gives
$\|\mathbf u^{\tau^{t+1}}-\mathbf v^\pi\|_\infty
\le\frac{\gamma\|\mathbf d^{\tau^t}\|_\infty}{1-\gamma}$.

\vspace{5pt}
\noindent\textbf{Iteration bound.}
We use the following lemma to turn the contraction bound into a
bound on the number of iterations.

\begin{lemma}[A bound for contracting sequences]
\label{lem:finite-scale-progress}
Let $0\le\gamma<1$ and $\mathsf G\in\mathbb Z_{\ge0}$. Suppose nonnegative
sequences $(E_t)_{t\ge0}$ and $(g_t)_{t\ge0}$ satisfy, for every $t\ge0$,
\begin{equation}
  0\le g_t\le E_t,\qquad
  E_{t+1}\le\gamma E_t,\qquad
  E_{t+1}\le\frac{\gamma g_t}{1-\gamma}.
  \label{eq:finite-scale-hypotheses}
\end{equation}
If the set $\{\lfloor\log_2g_t\rfloor:t\ge0,\ g_t>0\}$ has at most
$\mathsf G$ elements, then $E_t=0$ for some integer $0\le t\le \mathsf G T_\gamma+1$.
\end{lemma}

\begin{proof}[Proof of Lemma~\ref{lem:finite-scale-progress}]
The assumptions and~\eqref{eq:discount-block} give
\begin{equation*}
  g_{t+T_\gamma}\le E_{t+T_\gamma}
  \le\frac{\gamma^{T_\gamma}g_t}{1-\gamma}<\frac{g_t}{2}
\end{equation*}
whenever $g_t>0$. Thus
$\lfloor\log_2g_{t+T_\gamma}\rfloor<\lfloor\log_2g_t\rfloor$
whenever $g_t>0$ and $g_{t+T_\gamma}>0$. At least one of
$g_0,g_{T_\gamma},\ldots,g_{\mathsf G T_\gamma}$ must be zero,
since otherwise the corresponding values of $\lfloor\log_2g_t\rfloor$ would give $\mathsf G+1$
distinct integers, contradicting the assumed bound.
If $g_t=0$,
the last inequality in~\eqref{eq:finite-scale-hypotheses} gives
\(E_{t+1}=0\).
\end{proof}

For each policy evaluation,~\eqref{eq:nature-contraction} and the
assumptions on $g_t$ verify~\eqref{eq:finite-scale-hypotheses}.
Lemma~\ref{lem:finite-scale-progress} and~\eqref{eq:unified-scale-bound} give $E_t=0$ for some
$t\le \mathsf G T_\gamma+1$. By~\eqref{eq:nature-improvement}, the
stopping test is $\mathscr T_\pi(\mathbf u^{\tau^t})=\mathbf u^{\tau^t}$.
Algorithm~\ref{alg:nature-policy-iteration} therefore returns
$\mathbf v^\pi$ and the minimizing pairs selected
in Line~\ref{line:pair-update} after at most $\mathsf G T_\gamma+2$ iterations.

\vspace{5pt}
\noindent\textbf{Arithmetic operations and comparisons.}
Lemma~\ref{lem:outer-pi} bounds the number of policy evaluations
by $1+(A-S)T_\gamma$. In each iteration of
Algorithm~\ref{alg:nature-policy-iteration},
Line~\ref{line:pair-value} costs $O(S^3)$ operations to solve
the linear system, and Line~\ref{line:pair-update} costs at most
$\mathsf C_{\rm in}$ to compute the minimizers in~\eqref{eq:nature-improvement}
and their objective values. Initialization requires one additional set of
$S$ minimizations at $\mathbf v=\mathbf0$. Thus $\mathsf G T_\gamma+3$
bounds both the number of linear systems and the number of sets of
$S$ minimizations per policy evaluation, including initialization and the
stopping iteration. Each evaluation is followed by the computations in
Algorithm~\ref{alg:exact-policy-iteration},
Line~\ref{line:policy-bellman-lps}, at cost at most $\mathsf C_{\rm out}$.
The comparisons in Line~\ref{line:pair-optimality-check} of
Algorithm~\ref{alg:nature-policy-iteration} and
Lines~\ref{line:policy-optimality-check} and~\ref{line:policy-improvement}
of Algorithm~\ref{alg:exact-policy-iteration} reuse the computed LP
objective values and cost $O(S)$ and $O(A)$, respectively.
These costs are absorbed by $\mathsf C_{\rm in}\ge S$ and
$\mathsf C_{\rm out}\ge A$, respectively. Forming each reward vector
and transition matrix costs $O(S^2)$, which is absorbed by the
linear-system cost. The total cost is therefore
\begin{equation*}
  O\Bigl(
    \underbrace{1+(A-S)T_\gamma}_{\text{policy evaluations}}
    \Bigl[
      \underbrace{(\mathsf G T_\gamma+3)(S^3+\mathsf C_{\rm in})}_{
        \text{Algorithm~\ref{alg:exact-policy-iteration}, Line~\ref{line:policy-evaluation}}}
      +\underbrace{\mathsf C_{\rm out}}_{
        \text{Algorithm~\ref{alg:exact-policy-iteration}, Line~\ref{line:policy-bellman-lps}}}
    \Bigr]
  \Bigr),
\end{equation*}
which gives~\eqref{eq:residual-scale-total-cost}.

\vspace{5pt}
\noindent\textbf{Encoding length.}
By the output encoding bound for the pairs
in~\eqref{eq:bellman-minimizers}, every linear system in
Algorithm~\ref{alg:nature-policy-iteration}, Line~\ref{line:pair-value}
has coefficients of bit length polynomial in $L$.
After clearing denominators, fraction-free Gaussian elimination
computes $\mathbf u^\tau$ with all solution coordinates and intermediate
numbers having bit length polynomial in $L$.
Each minimization in~\eqref{eq:bellman-minimizers} uses either
$\mathbf v=\mathbf0$ or one of these computed value vectors.
The encoding assumption for these minimizations therefore ensures
that all intermediate rationals throughout the algorithm have
bit length polynomial in $L$.

At termination, the LP solutions from
Line~\ref{line:policy-bellman-lps} of
Algorithm~\ref{alg:exact-policy-iteration} identify
$\mathcal A_s^\star=\{a:\mathscr T_{sa}(\mathbf v^\star)
=[\mathbf v^\star]_s\}$ and supply a minimizing reward-transition
pair for every $(s,a)$. Any selection
$\pi(s)\in\mathcal A_s^\star$ is optimal by~\eqref{eq:optimal-actions}.
\end{proof}

For fixed $\gamma$, polynomial bounds on $\mathsf C_{\rm in}$,
$\mathsf C_{\rm out}$, and $\mathsf G$ in the input size
$N$ give a strongly polynomial algorithm.
Section~\ref{sec:applications} derives such bounds from structural
properties of the uncertainty sets.

\section{Structural conditions and applications}
\label{sec:applications}

In this section, we apply the
general complexity bound in Theorem~\ref{thm:residual-criterion} to concrete
RMDP models, including $\ell_1$, weighted $\ell_1$, interval, and Wasserstein
uncertainty (Examples~\ref{ex:norm-ambiguity}--\ref{ex:wasserstein-ambiguity}).
We also extend the analysis to turn-based stochastic games.
Table~\ref{tab:result-summary} summarizes
our general and model-specific complexity bounds.

\begin{table}[t]
  \begingroup
  \definecolor{tableInk}{HTML}{344B5B}
  \definecolor{tableTint}{HTML}{EEF2F5}
  \definecolor{tableRule}{HTML}{B7C3CC}
  \color{black}
  \hypersetup{allcolors=tableInk}
  \centering
  \small
  \setlength{\tabcolsep}{4pt}
  \renewcommand{\arraystretch}{1.25}
  \begin{tabular}{
    >{\centering\arraybackslash}m{\dimexpr0.16\linewidth-2\tabcolsep\relax}
    >{\centering\arraybackslash}m{\dimexpr0.29\linewidth-2\tabcolsep\relax}
    >{\centering\arraybackslash}m{\dimexpr0.55\linewidth-2\tabcolsep\relax}}
    \arrayrulecolor{tableRule}
    \toprule
    \rowcolor{tableTint}
    \textcolor{tableInk}{\textbf{Result}}
      & \textcolor{tableInk}{\textbf{Model}}
      & \textcolor{tableInk}{\textbf{Complexity bound}}\\
    \midrule
    \addlinespace[3pt]
    Theorem~\ref{thm:residual-criterion}
      & Polyhedral RMDPs
      & \mbox{$O\bigl((A-S+1)[(1+\mathsf G)(S^3+\mathsf C_{\rm in})
                         +\mathsf C_{\rm out}]\bigr)$}\\
    \addlinespace[3pt]
    \midrule
    \addlinespace[3pt]
    Corollary~\ref{cor:norm-planning}
      & $\ell_1$ RMDPs\par\kern3pt
          {\color{tableRule}\hrule height 0.3pt}\kern3pt
          Interval RMDPs\newline
          {\footnotesize (including weighted $\ell_\infty$ RMDPs)}
      & \mbox{$O\bigl((A-S+1)[S^4+AS]\bigr)$}\\
    \addlinespace[3pt]
    \midrule
    \addlinespace[3pt]
    Corollary~\ref{cor:budget-models}
      & Weighted $\ell_1$ RMDPs\par\kern3pt
          {\color{tableRule}\hrule height 0.3pt}\kern3pt
          Wasserstein RMDPs
      & \mbox{$O\bigl((A-S+1)[S^7+AS\log(2+S)]\bigr)$}\\
    \addlinespace[3pt]
    \midrule
    \addlinespace[3pt]
    Theorem~\ref{thm:game-progress-criterion}
      & Turn-based stochastic games
      & \mbox{$O\bigl((1+\mathsf G_{\max})[(1+\mathsf G_{\min})
              (S^3+\mathsf C_{\min})+\mathsf C_{\max}]\bigr)$}\\
    \addlinespace[3pt]
    \bottomrule
  \end{tabular}
  \caption{Summary of our arithmetic complexity bounds at fixed discount.}
  \label{tab:result-summary}
  \endgroup
\end{table}

\subsection{\texorpdfstring{$\ell_1$}{L1} and interval RMDPs}
\label{subsec:order-compatible}
\label{subsec:norm-corollaries}

We introduce a unified condition on transition sets to bound
$\mathsf G$ in Theorem~\ref{thm:residual-criterion}.
Both $\ell_1$ balls and interval sets satisfy this condition.

\begin{condition}[Order compatibility]
\label{cond:order-compatible}
For every $(s,a)$, $\mathcal M_{sa}=\{r_{sa}\}\times\mathcal U_{sa}$,
where $r_{sa}$ is fixed and $\mathcal U_{sa}$ is a nonempty compact
subset of $\Delta(\mathcal S)$.
For every $(s,a)$ and every permutation $\sigma$ of $\mathcal S$,
there exists $\mathbf p_{sa}^{\sigma}\in\mathcal U_{sa}$ such that
\begin{equation}
  (\mathbf p_{sa}^{\sigma})^\top\mathbf w
   =\min_{\mathbf p\in\mathcal U_{sa}}\mathbf p^\top\mathbf w,
  \quad\forall\mathbf w\in\mathbb R^S
  \text{ with }[\mathbf w]_{\sigma(1)}\le\cdots\le[\mathbf w]_{\sigma(S)}.
  \label{eq:order-compatible}
\end{equation}
\end{condition}

The transition sets are called \emph{order compatible} because one
minimizing distribution works for every vector with a prescribed
coordinate order, including ties.

\begin{theorem}[Robust policy iteration under order compatibility]
\label{thm:order-pi}
Let $\gamma\in[0,1)$ be rational. Under
Condition~\ref{cond:order-compatible}, for every deterministic policy
$\pi$ and every feasible initial choice $\tau^0$,
each run of
Algorithm~\ref{alg:nature-policy-iteration},
allowing any choice among minimizers in
Line~\ref{line:pair-update},
admits a nonnegative sequence $(g_t)_{t\ge0}$
satisfying~\eqref{eq:unified-scale-bound} with
\begin{equation*}
  \mathsf G=\max\{0,2S-3\}.
\end{equation*}
Assume also the cost and encoding conditions
following~\eqref{eq:bellman-minimizers} hold with bounds
$\mathsf C_{\rm in}$ and $\mathsf C_{\rm out}$.
For fixed $\gamma$, Theorem~\ref{thm:residual-criterion} implies that
Algorithm~\ref{alg:exact-policy-iteration}, with
Line~\ref{line:policy-evaluation} implemented by
Algorithm~\ref{alg:nature-policy-iteration}, computes $\mathbf v^\star$
and an optimal deterministic stationary policy using
\begin{equation*}
  O\left((A-S+1)
    [S^4+S\mathsf C_{\rm in}+\mathsf C_{\rm out}]\right)
\end{equation*}
arithmetic operations and comparisons.
Every intermediate rational has bit length polynomial in $L$.
\end{theorem}

\begin{proof}[Proof of Theorem~\ref{thm:order-pi}]
We use the ordering condition to bound the error after
each update and verify the bound on $\mathsf G$. The complexity bound
then follows from Theorem~\ref{thm:residual-criterion}.
If $\gamma=0$, fixed rewards imply
$\mathbf u^{\tau^0}=\mathbf v^\pi$, so the test in
Line~\ref{line:pair-optimality-check} of
Algorithm~\ref{alg:nature-policy-iteration} holds at $t=0$.
In this case, $g_t=0$ satisfies~\eqref{eq:unified-scale-bound}.
We henceforth assume $\gamma>0$ and fix $\pi$.

\vspace{5pt}
\noindent\textbf{Bounding the error after an update.}
We measure disagreements between the coordinate orders of
$\mathbf u^{\tau^t}$ and $\mathbf v^\pi$ by
\begin{equation}
  \Delta_t
   =\max_{i,j\in\mathcal S,\;[\mathbf u^{\tau^t}]_i\le[\mathbf u^{\tau^t}]_j}
      \bigl([\mathbf v^\pi]_i-[\mathbf v^\pi]_j\bigr).
  \label{eq:order-defect}
\end{equation}
We set $\Delta_t=0$ after termination, as for $E_t$.
The maximum includes $i=j$, so $\Delta_t\ge0$.
Since $\mathbf v^\pi\le\mathbf u^{\tau^t}$
by~\eqref{eq:residual-sandwich},
$[\mathbf u^{\tau^t}]_i\le[\mathbf u^{\tau^t}]_j$ implies
$[\mathbf v^\pi]_i-[\mathbf v^\pi]_j
 \le[\mathbf u^{\tau^t}]_j-[\mathbf v^\pi]_j\le E_t$.
Thus $0\le\Delta_t\le E_t$.
We approximate $\mathbf v^\pi$ from below by a vector
$\mathbf w^t\in\mathbb R^S$ that respects the order and ties of
$\mathbf u^{\tau^t}$, defining
\begin{equation*}
  [\mathbf w^t]_i
    =\min_{j\in\mathcal S,\;[\mathbf u^{\tau^t}]_i\le[\mathbf u^{\tau^t}]_j}
      [\mathbf v^\pi]_j,
  \qquad \forall i\in\mathcal S.
\end{equation*}
If $[\mathbf u^{\tau^t}]_i\le[\mathbf u^{\tau^t}]_j$, the minimum defining
$[\mathbf w^t]_j$ is taken over a subset of the indices used for
$[\mathbf w^t]_i$, so $[\mathbf w^t]_i\le[\mathbf w^t]_j$.
Equality of the coordinates of $\mathbf u^{\tau^t}$ gives identical sets
and hence equal coordinates of $\mathbf w^t$.
Since each set contains its own index, this definition
and~\eqref{eq:order-defect} give
$\mathbf0\le\mathbf v^\pi-\mathbf w^t\le\Delta_t\mathbf1$.

Condition~\ref{cond:order-compatible} supplies a common minimizer at
$\mathbf u^{\tau^t}$ and $\mathbf w^t$. Since
Algorithm~\ref{alg:nature-policy-iteration} may select any minimizer at
$\mathbf u^{\tau^t}$, we show that every such choice also minimizes at $\mathbf w^t$.
At a fixed state $s$, we choose a permutation $\sigma$ sorting
$\mathbf u^{\tau^t}$ and a distribution
$\mathbf p_s^\sigma\in\mathcal U_{s,\pi(s)}$
satisfying~\eqref{eq:order-compatible} for this permutation.
For any minimizer $\mathbf p$ of $\mathbf p^\top\mathbf u^{\tau^t}$
over $\mathcal U_{s,\pi(s)}$, we have
$(\mathbf p-\mathbf p_s^\sigma)^\top\mathbf u^{\tau^t}=0$.
Both $\mathbf u^{\tau^t}\pm\eta\mathbf w^t$ retain the order $\sigma$
for sufficiently small $\eta>0$, since $\mathbf w^t$ is constant
wherever $\mathbf u^{\tau^t}$ is tied. By~\eqref{eq:order-compatible},
$0\le(\mathbf p-\mathbf p_s^\sigma)^\top
 (\mathbf u^{\tau^t}\pm\eta\mathbf w^t)
 =\pm\eta(\mathbf p-\mathbf p_s^\sigma)^\top\mathbf w^t$.
Thus $(\mathbf p-\mathbf p_s^\sigma)^\top\mathbf w^t=0$.
Since $\mathbf p_s^\sigma$ minimizes at $\mathbf w^t$
by~\eqref{eq:order-compatible}, so does $\mathbf p$.
The distributions selected in Line~\ref{line:pair-update} of
Algorithm~\ref{alg:nature-policy-iteration} therefore minimize at
$\mathbf w^t$. Fixed rewards give
$\mathbf r^{\tau^{t+1}}+\gamma\mathbf P^{\tau^{t+1}}\mathbf w^t
=\mathscr T_\pi(\mathbf w^t)$, so~\eqref{eq:nature-residual} yields
\begin{equation*}
  \mathbf0\le\mathbf d^{\tau^{t+1}}
    =\mathscr T_\pi(\mathbf w^t)-\mathbf v^\pi
      +\gamma\mathbf P^{\tau^{t+1}}(\mathbf v^\pi-\mathbf w^t)
    \le\gamma\Delta_t\mathbf1.
\end{equation*}
The last inequality uses $\mathbf v^\pi-\mathbf w^t\le\Delta_t\mathbf1$,
stochasticity of $\mathbf P^{\tau^{t+1}}$, and
$\mathscr T_\pi(\mathbf w^t)\le\mathscr T_\pi(\mathbf v^\pi)
=\mathbf v^\pi$.
Bounds~\eqref{eq:nature-contraction}
and~\eqref{eq:residual-sandwich} give $E_{t+1}\le\gamma E_t$ and
$E_{t+1}\le\frac{\|\mathbf d^{\tau^{t+1}}\|_\infty}{1-\gamma}
\le\frac{\gamma\Delta_t}{1-\gamma}$.
Together with $0\le\Delta_t\le E_t$, these
verify the comparison inequalities
in~\eqref{eq:unified-scale-bound} with $g_t=\Delta_t$.

\vspace{5pt}
\noindent\textbf{Bounding $\mathsf G$.}
Every positive $\Delta_t$ is a positive difference between coordinates
of the fixed vector $\mathbf v^\pi$.
The following lemma bounds the number of distinct integers
$\lfloor\log_2\Delta_t\rfloor$ for $\Delta_t>0$.

\begin{lemma}[Pairwise differences]
\label{lem:pairwise-scales}
Let $S\ge1$ be an integer. For every $\mathbf x\in\mathbb R^S$,
\begin{equation*}
  \left|\left\{
    \lfloor\log_2([\mathbf x]_i-[\mathbf x]_j)\rfloor:
    1\le i,j\le S,\ [\mathbf x]_i>[\mathbf x]_j
  \right\}\right|
  \le\max\{0,2S-3\}.
\end{equation*}
\end{lemma}

\begin{proof}
If all coordinates are equal, the set is empty.
Otherwise, after relabeling coordinates, we assume
$[\mathbf x]_1\le\cdots\le[\mathbf x]_S$.
We sort the positive differences $[\mathbf x]_{i+1}-[\mathbf x]_i$,
$i=1,\ldots,S-1$, as $0<a_1\le\cdots\le a_k$, where $1\le k\le S-1$.
For $p<q$, we have
$[\mathbf x]_q-[\mathbf x]_p
=\sum_{i=p}^{q-1}([\mathbf x]_{i+1}-[\mathbf x]_i)$.
Thus every positive coordinate difference is a sum of a nonempty
subset of the $a_i$.
We write $\Sigma_j=\sum_{i=1}^j a_i$.
For each positive coordinate difference, we let $j$ be the largest
index among its summands $a_i$. This difference is at least $a_j$
and at most $\Sigma_j$, so it belongs to $[a_j,\Sigma_j]$.
The interval $[a_1,\Sigma_1]=\{a_1\}$ contributes one integer.
For $j\ge2$, if $a_j>\Sigma_{j-1}$, then
$\Sigma_j=a_j+\Sigma_{j-1}<2a_j$.
Thus every $y\in[a_j,\Sigma_j]$ satisfies
$\lfloor\log_2 a_j\rfloor\le\lfloor\log_2 y\rfloor
\le\lfloor\log_2 a_j\rfloor+1$, so this interval contributes at most
two previously uncounted values of $\lfloor\log_2 y\rfloor$.
If $a_j\le\Sigma_{j-1}$, then
$[a_j,\Sigma_{j-1}]\subseteq[a_{j-1},\Sigma_{j-1}]$ and
$\Sigma_j\le2\Sigma_{j-1}$, so $[a_j,\Sigma_j]$ contributes at most
one previously uncounted value of $\lfloor\log_2 y\rfloor$.
The total number is therefore at most $1+2(k-1)\le2S-3$.
\end{proof}

Applying Lemma~\ref{lem:pairwise-scales} to $\mathbf v^\pi$
completes the proof of~\eqref{eq:unified-scale-bound}
with $g_t=\Delta_t$ from~\eqref{eq:order-defect} and
\mbox{$\mathsf G=\max\{0,2S-3\}$}.
For fixed $\gamma$, Theorem~\ref{thm:residual-criterion}
gives the stated complexity and encoding bounds because $T_\gamma$
is constant and $1+\mathsf G=O(S)$.
\end{proof}

\paragraph{Verification for $\ell_1$ balls and interval sets.}
We apply Theorem~\ref{thm:order-pi} to the $\ell_1$ and
interval RMDPs in Examples~\ref{ex:norm-ambiguity} and~\ref{ex:interval-ambiguity}
to obtain the following strongly polynomial complexity bounds.
Interval uncertainty includes weighted $\ell_\infty$ balls as special cases.

\begin{corollary}[Complexity bounds for $\ell_1$ and interval RMDPs]
\label{cor:norm-planning}
Fix a rational discount $\gamma\in[0,1)$.
Suppose rewards are fixed and each transition uncertainty set is an
$\ell_1$ ball from Example~\ref{ex:norm-ambiguity} or an interval set
from Example~\ref{ex:interval-ambiguity}.
For Algorithm~\ref{alg:exact-policy-iteration} with
Line~\ref{line:policy-evaluation} implemented by
Algorithm~\ref{alg:nature-policy-iteration}, the parameters in
Theorem~\ref{thm:residual-criterion} can be chosen as
\begin{equation}
  \mathsf C_{\rm in}=O(S^2),\qquad
  \mathsf C_{\rm out}=O(AS),\qquad
  \mathsf G=\max\{0,2S-3\}.
  \label{eq:norm-local-costs}
\end{equation}
Hence, Algorithm~\ref{alg:exact-policy-iteration} returns $\mathbf v^\star$, all sets $\mathcal A_s^\star$, and an optimal
deterministic stationary policy using
\begin{equation}
  O\left((A-S+1)[S^4+AS]\right)
  \label{eq:norm-full-bound}
\end{equation}
arithmetic operations and comparisons. Every intermediate rational
has bit length polynomial in $L$.
\end{corollary}

\begin{proof}
We verify Condition~\ref{cond:order-compatible}, the bounds
in~\eqref{eq:norm-local-costs}, and the encoding requirements of
Theorem~\ref{thm:residual-criterion}.

\vspace{5pt}
\noindent\textbf{Minimizing transition distributions.}
For a fixed $(s,a)$, we omit these subscripts and list the states
as $i_1,\ldots,i_{S}$ with
$[\mathbf v]_{i_1}\le\cdots\le[\mathbf v]_{i_{S}}$.
We resolve ties by a fixed state ordering.
For $1\le k\le S$, we define $\widehat{\mathbf p}$ by
\begin{equation*}
  \sum_{j=1}^k[\widehat{\mathbf p}]_{i_j}
  =\begin{cases}
    \min\left\{1,\sum_{j=1}^k[\overline{\mathbf p}]_{i_j}+\sigma\right\},
       &\ell_1\text{ ball},\\
    \min\left\{\sum_{j=1}^k[\mathbf h]_{i_j},
                    1-\sum_{j=k+1}^{S}[\mathbf l]_{i_j}\right\},
       &\text{interval set}.
  \end{cases}
\end{equation*}
The cumulative sums are nonnegative, nondecreasing, and equal to one at $k=S$,
so $\widehat{\mathbf p}\in\Delta(\mathcal S)$.
For an $\ell_1$ ball, subtracting consecutive cumulative sums gives
$[\widehat{\mathbf p}]_{i_j}\le[\overline{\mathbf p}]_{i_j}$ for $j\ge2$,
since capping both sums at one cannot increase their difference. Hence
$\|\widehat{\mathbf p}-\overline{\mathbf p}\|_1
=2([\widehat{\mathbf p}]_{i_1}-[\overline{\mathbf p}]_{i_1})\le2\sigma$.
For an interval set, differences of consecutive cumulative sums lie
between the corresponding lower and upper bounds, giving
$\mathbf l\le\widehat{\mathbf p}\le\mathbf h$.
These sorting constructions are inspired by the methods of
\citet{HoPetrikWiesemann2021} for $\ell_1$ balls and
\citet{GivanLeachDean2000} for interval sets.
Every feasible $\mathbf p$ satisfies
$\sum_{j=1}^k[\mathbf p]_{i_j}\le
\sum_{j=1}^k[\widehat{\mathbf p}]_{i_j}$ for $1\le k<S$. For the ball,
$\sum_{j=1}^k([\mathbf p]_{i_j}-[\overline{\mathbf p}]_{i_j})
\le\frac{\|\mathbf p-\overline{\mathbf p}\|_1}{2}\le\sigma$.
For intervals, the upper bounds on the first $k$ coordinates and
the lower bounds on the remaining coordinates give the two terms
in the minimum.
\looseness=-1 Summation by parts therefore gives
\begin{equation*}
  \mathbf p^\top\mathbf v
  =[\mathbf v]_{i_{S}}-\sum_{k=1}^{S-1}([\mathbf v]_{i_{k+1}}-[\mathbf v]_{i_k})
                          \sum_{j=1}^k[\mathbf p]_{i_j}
  \ge\widehat{\mathbf p}^\top\mathbf v.
\end{equation*}
The construction depends only on the local input and the state
ordering, and the inequality holds for every vector with that
ordering, including ties. This proves
Condition~\ref{cond:order-compatible}.

\vspace{5pt}
\noindent\textbf{Operation counts and encoding lengths.}
At a common value vector $\mathbf v$, sorting all $S$ states takes
$O(S\log(2+S))$ operations and comparisons. For each $(s,a)$,
we compute the cumulative sums, their successive differences,
and $r_{sa}+\gamma\widehat{\mathbf p}_{sa}^\top\mathbf v$ in $O(S)$
operations. Since $A\ge S$, we obtain
$O(S\log(2+S)+S^2)=O(S^2)$ for a policy's $S$ pairs and
$O(S\log(2+S)+AS)=O(AS)$ for all $A$ pairs, proving the cost bounds
in~\eqref{eq:norm-local-costs}. Theorem~\ref{thm:order-pi} gives
$\mathsf G=\max\{0,2S-3\}$.

Every returned distribution is obtained by sums, differences, and
minima of the original probabilities, radii, and endpoints.
Its coordinates therefore have a common input-dependent denominator
of polynomial bit length, independently of $\mathbf v$.
The rational numbers used to construct these distributions
and evaluate their objectives have bit length polynomial in $L$ and the
encoding length of $\mathbf v$. These are the encoding requirements
of Theorem~\ref{thm:residual-criterion}.
Substituting~\eqref{eq:norm-local-costs} into
Theorem~\ref{thm:residual-criterion} gives
\begin{equation}
  O\left(
    [1+(A-S)T_\gamma]
    \left[(1+\max\{0,2S-3\}T_\gamma)S^3+AS\right]
  \right).
  \label{eq:norm-general-bound}
\end{equation}
For fixed $\gamma$, this gives~\eqref{eq:norm-full-bound}.
\end{proof}

For fixed $\gamma$,
the iteration bounds of
\citet[Theorems~1--2]{AsadiChatterjeeMontaseriShafiee2026} yield
\mbox{$O(A[S^6\log^2(2+S)+AS])$} arithmetic operations and comparisons
for $\ell_1$ and $\ell_\infty$ RMDPs.
Corollary~\ref{cor:norm-planning} improves this bound to
\mbox{$O((A-S+1)[S^4+AS])$} and applies more generally to interval
uncertainty, which includes weighted $\ell_\infty$ balls as special
cases.

\subsection{Weighted \texorpdfstring{$\ell_1$}{L1} and Wasserstein RMDPs}
\label{subsec:transport-corollaries}

To analyze weighted $\ell_1$ and Wasserstein uncertainty together,
we first show that the weighted $\ell_1$ ball~\eqref{eq:weighted-l1-set}
admits the same transportation form as the Wasserstein
set~\eqref{eq:finite-wasserstein-lift}.

\begin{lemma}[Transportation representation of weighted $\ell_1$ uncertainty]
\label{lem:weighted-transport}
For the weighted $\ell_1$ ball $\mathcal U_{sa}$ in~\eqref{eq:weighted-l1-set},
let $[\mathbf C_{sa}]_{ii}=0$ and
$[\mathbf C_{sa}]_{ij}=[\mathbf w_{sa}]_i+[\mathbf w_{sa}]_j$ for $i\ne j$.
Then $\mathcal M_{sa}=\{r_{sa}\}\times\mathcal U_{sa}$ has the extended representation
\begin{equation*}
  \mathcal Q_{sa}
  =\left\{(r_{sa},\mathbf p,\boldsymbol\Pi):
    \mathbf p\in\Delta(\mathcal S),\quad \boldsymbol\Pi\ge0,\quad
    \boldsymbol\Pi\mathbf1=\overline{\mathbf p}_{sa},\quad
    \boldsymbol\Pi^\top\mathbf1=\mathbf p,\quad
    \langle\mathbf C_{sa},\boldsymbol\Pi\rangle\le\rho_{sa}
  \right\}.
\end{equation*}
\end{lemma}
\begin{proof}
Fix $\mathbf p\in\Delta(\mathcal S)$.
For $\boldsymbol\Pi\ge0$ with
$\boldsymbol\Pi\mathbf1=\overline{\mathbf p}_{sa}$ and
$\boldsymbol\Pi^\top\mathbf1=\mathbf p$,
the cost is at least the weighted $\ell_1$ distance,
\begin{equation*}
  \langle\mathbf C_{sa},\boldsymbol\Pi\rangle
  =\sum_i[\mathbf w_{sa}]_i
    \bigl([\overline{\mathbf p}_{sa}]_i+[\mathbf p]_i
          -2[\boldsymbol\Pi]_{ii}\bigr)
  \ge\sum_i[\mathbf w_{sa}]_i
    |[\mathbf p]_i-[\overline{\mathbf p}_{sa}]_i|.
\end{equation*}
For $\mathbf p\ne\overline{\mathbf p}_{sa}$, we attain equality with
\begin{equation*}
  [\boldsymbol\Pi]_{ij}=
  \begin{cases}
    \min\{[\overline{\mathbf p}_{sa}]_i,[\mathbf p]_i\}, & i=j,\\[4pt]
    \dfrac{
      2\max\{[\overline{\mathbf p}_{sa}]_i-[\mathbf p]_i,0\}
       \max\{[\mathbf p]_j-[\overline{\mathbf p}_{sa}]_j,0\}
    }{\|\mathbf p-\overline{\mathbf p}_{sa}\|_1}, & i\ne j.
  \end{cases}
\end{equation*}
Since $\mathbf p$ and $\overline{\mathbf p}_{sa}$ both sum to one,
$\sum_i\max\{[\overline{\mathbf p}_{sa}]_i-[\mathbf p]_i,0\}
 =\sum_i\max\{[\mathbf p]_i-[\overline{\mathbf p}_{sa}]_i,0\}
 =\frac{\|\mathbf p-\overline{\mathbf p}_{sa}\|_1}{2}$.
The product in the off-diagonal formula is zero when $i=j$,
so summing the entries gives
\begin{equation*}
  \left\{
  \begin{aligned}
    \sum_j[\boldsymbol\Pi]_{ij}
      &=\min\{[\overline{\mathbf p}_{sa}]_i,[\mathbf p]_i\}
        +\max\{[\overline{\mathbf p}_{sa}]_i-[\mathbf p]_i,0\}
       =[\overline{\mathbf p}_{sa}]_i,\\
    \sum_i[\boldsymbol\Pi]_{ij}
      &=\min\{[\overline{\mathbf p}_{sa}]_j,[\mathbf p]_j\}
        +\max\{[\mathbf p]_j-[\overline{\mathbf p}_{sa}]_j,0\}
       =[\mathbf p]_j.
  \end{aligned}
  \right.
\end{equation*}
For $\mathbf p=\overline{\mathbf p}_{sa}$, we set
$[\boldsymbol\Pi]_{ii}=[\mathbf p]_i$ and
$[\boldsymbol\Pi]_{ij}=0$ for $i\ne j$.
Hence $\mathbf p$ lies in the weighted $\ell_1$ ball~\eqref{eq:weighted-l1-set}
exactly when there exists such a matrix $\boldsymbol\Pi$ with
$\langle\mathbf C_{sa},\boldsymbol\Pi\rangle\le\rho_{sa}$.
\end{proof}

\begin{corollary}[Complexity bounds for weighted $\ell_1$ and Wasserstein RMDPs]
\label{cor:budget-models}
Fix a rational discount $\gamma\in[0,1)$.
Suppose rewards are fixed and either all local transition sets are
weighted $\ell_1$ sets as in~\eqref{eq:weighted-l1-set}, or all are
Wasserstein sets represented by~\eqref{eq:finite-wasserstein-lift}.
For Algorithm~\ref{alg:exact-policy-iteration} with
Line~\ref{line:policy-evaluation} implemented by
Algorithm~\ref{alg:nature-policy-iteration}, the parameters in
Theorem~\ref{thm:residual-criterion} can be chosen as
\begin{equation}
  \mathsf C_{\rm in}=O(S^2\log(2+S)),\qquad
  \mathsf C_{\rm out}=O(AS\log(2+S)),\qquad
  \mathsf G=O(S^4).
  \label{eq:budget-local-parameters}
\end{equation} 
Hence, Algorithm~\ref{alg:exact-policy-iteration} returns $\mathbf v^\star$, all sets $\mathcal A_s^\star$, and an
optimal deterministic stationary policy using
\begin{equation}
  O\left((A-S+1)[S^7+AS\log(2+S)]\right)
  \label{eq:budget-full-bound}
\end{equation} 
arithmetic operations and comparisons.
Every intermediate rational has bit length polynomial in $L$.
\end{corollary}

\begin{proof}[Proof of Corollary~\ref{cor:budget-models}]
We first bound $\mathsf C_{\rm in}$ and $\mathsf C_{\rm out}$.
For a given $\mathbf v\in\mathbb Q^S$, Lemma~\ref{lem:weighted-transport}
and the Wasserstein representation~\eqref{eq:finite-wasserstein-lift}
express the minimization of $\mathbf p^\top\mathbf v$ as
\begin{equation}
  \min_{\boldsymbol\Pi\ge0}
  \left\{\mathbf v^\top\boldsymbol\Pi^\top\mathbf1:
    \boldsymbol\Pi\mathbf1=\overline{\mathbf p},\quad
    \langle\mathbf C,\boldsymbol\Pi\rangle\le\rho\right\},
  \label{eq:transport-primal}
\end{equation}
where $\overline{\mathbf p}$ is the nominal distribution, $\rho$
is the budget, and $\mathbf C$ is the corresponding cost matrix.
The transition distribution is $\mathbf p=\boldsymbol\Pi^\top\mathbf1$.
We introduce a multiplier $\lambda\ge0$ for the budget constraint
$\langle\mathbf C,\boldsymbol\Pi\rangle\le\rho$. The resulting
Lagrangian dual is
\begin{equation}
  \max_{\lambda\ge0}
  \min_{\boldsymbol\Pi\ge0,\ \boldsymbol\Pi\mathbf1=\overline{\mathbf p}}
  \left\{\mathbf v^\top\boldsymbol\Pi^\top\mathbf1
    +\lambda\bigl(\langle\mathbf C,\boldsymbol\Pi\rangle-\rho\bigr)\right\}
  =\max_{\lambda\ge0}
  \Biggl\{
    \sum_{i\in\mathcal S}[\overline{\mathbf p}]_i
    \underbrace{\min_{j\in\mathcal S}
      \bigl\{[\mathbf v]_j+\lambda[\mathbf C]_{ij}\bigr\}}
      _{=:\,\psi_i(\lambda)}
    -\lambda\rho
  \Biggr\}.
  \label{eq:transport-dual}
\end{equation}
The equality follows by minimizing separately over each row, whose
nonnegative entries sum to $[\overline{\mathbf p}]_i$.
Each $\psi_i$ is piecewise-linear on $[0,\infty)$.
We construct a representation of these functions together with a
choice of minimizing indices. This computation depends only on
$\mathbf C$ and $\mathbf v$, so the representation can be reused
for different nominal distributions and budgets.
Given $\overline{\mathbf p}$ and $\rho$, we compute an optimal multiplier
and construct a nonnegative matrix $\boldsymbol\Pi$ with
$\boldsymbol\Pi\mathbf1=\overline{\mathbf p}$.
Its positive entries occur only at indices attaining the corresponding
row minima, and it satisfies
$\langle\mathbf C,\boldsymbol\Pi\rangle\le\rho$ and
$\lambda(\langle\mathbf C,\boldsymbol\Pi\rangle-\rho)=0$.
The objective value of $\boldsymbol\Pi$ in~\eqref{eq:transport-primal}
equals the dual optimum in~\eqref{eq:transport-dual},
so $\mathbf p=\boldsymbol\Pi^\top\mathbf1$ minimizes
$\mathbf p^\top\mathbf v$ over the transition set.
The following lemma bounds the cost of computing this minimum
and a minimizing distribution.

\begin{lemma}[Computing minimizing transitions]
\label{lem:budget-minimizers}
Let $\mathbf v\in\mathbb Q^S$. Assume that the nominal distributions,
budgets, and weights or costs below are rational and that the transition
sets are nonempty. The minimum of $\mathbf p^\top\mathbf v$ over each
transition set and a minimizing distribution can be computed with
the following cost bounds.
All costs count arithmetic operations and comparisons.
\begin{enumerate}[label=(\roman*),leftmargin=*]
  \item For weighted $\ell_1$ sets in~\eqref{eq:weighted-l1-set},
  constructing a representation of the functions $\psi_i$
  in~\eqref{eq:transport-dual} together with a choice of
  minimizing indices costs $O(S\log(2+S))$, using only
  $\mathbf w$ and $\mathbf v$. For each nominal distribution and
  budget, computing the minimum and a minimizing distribution
  then costs an additional $O(S)$.
  \item For Wasserstein sets represented
  by~\eqref{eq:finite-wasserstein-lift}, constructing a representation
  of the functions $\psi_i$ in~\eqref{eq:transport-dual} together with
  a choice of minimizing indices costs $O(S^2\log(2+S))$, using only
  $\mathbf C$ and $\mathbf v$. For each nominal distribution and
  budget, computing the minimum and a minimizing distribution
  then costs an additional $O(S\log(2+S))$.
\end{enumerate}
Each returned distribution has encoding length polynomial in the
encoding length of the nominal distribution, budget, and weights
or costs, independently of $\mathbf v$. All computed rationals
have bit length polynomial in the full input, including~$\mathbf v$.
\end{lemma}
\begin{proof}
See Appendix~\ref{app:weighted-oracle} for part~(i) and
Appendix~\ref{app:wasserstein-oracle} for part~(ii).
\end{proof}

Lemma~\ref{lem:budget-minimizers}(i) gives a cost of
$O(S\log(2+S))$ per state-action pair for weighted $\ell_1$ sets.
For Wasserstein sets, Lemma~\ref{lem:budget-minimizers}(ii) gives a cost of
$O(S^2\log(2+S))$ to represent the functions $\psi_i$
in~\eqref{eq:transport-dual} once, followed by
$O(S\log(2+S))$ per pair.
Summing over a policy's $S$ pairs or all $A$ pairs, and using $A\ge S$, gives
\begin{equation*}
  \mathsf C_{\rm in}=O(S^2\log(2+S)),\qquad
  \mathsf C_{\rm out}=O(AS\log(2+S)).
\end{equation*}
Forming the pair $(r_{sa},\mathbf p)$ and its objective value
$r_{sa}+\gamma\mathbf p^\top\mathbf v$ in~\eqref{eq:bellman-minimizers}
requires only $O(1)$ additional operations per pair because
$\mathbf p^\top\mathbf v$ is already computed.
Lemma~\ref{lem:budget-minimizers} also gives the required encoding bounds.

\vspace{5pt}
\noindent\textbf{Policy-evaluation error.}
It remains to verify~\eqref{eq:unified-scale-bound} with
$\mathsf G=O(S^4)$.
For $\mathbf u\ge\mathbf v$, the following lemma bounds the
objective gap at $\mathbf v$ of a distribution minimizing
$\mathbf p^\top\mathbf u$ over the transition set. The bound belongs to
a finite set determined by $\mathbf v$ and the weights or costs.

\begin{lemma}[Error bounds for minimizing transitions]
\label{lem:budget-error}
Fix $\mathbf v\in\mathbb R^S$ and either a weight vector
$\mathbf w\ge\mathbf0$ for~\eqref{eq:weighted-l1-set} or a
nonnegative rational cost matrix $\mathbf C$
for~\eqref{eq:finite-wasserstein-lift}.
There exists a finite set $\mathcal G\subseteq\mathbb R_{\ge0}$
such that, for every nonempty transition set $\mathcal U$ of the
respective form, every $\mathbf u\ge\mathbf v$, and every
$\mathbf p\in\operatorname*{arg\,min}_{\mathbf z\in\mathcal U}
\mathbf u^\top\mathbf z$, some $g\in\mathcal G$ satisfies
\begin{equation}
  0\le\mathbf v^\top\mathbf p
       -\min_{\mathbf z\in\mathcal U}\mathbf v^\top\mathbf z
  \le g\le\|\mathbf u-\mathbf v\|_\infty.
  \label{eq:budget-objective-gap}
\end{equation}
The set $\mathcal G$ contains at most
(i)~$8S^3$ positive numbers for weighted $\ell_1$ sets and
(ii)~$2S^4$ for Wasserstein sets.
\end{lemma}
\begin{proof}
See Appendix~\ref{app:weighted-progress} for part~(i) and
Appendix~\ref{app:wasserstein-progress} for part~(ii).
\end{proof}

For $\gamma>0$, we fix $\pi$.
In Algorithm~\ref{alg:nature-policy-iteration}, Line~\ref{line:pair-update},
the distribution in $\tau^{t+1}$ at state $s$ minimizes
$\mathbf p^\top\mathbf u^{\tau^t}$ over $\mathcal U_{s,\pi(s)}$.
Since $\mathbf u^{\tau^t}\ge\mathbf v^\pi$
by~\eqref{eq:residual-sandwich}, we apply
Lemma~\ref{lem:budget-error} at each state with
$\mathbf u=\mathbf u^{\tau^t}$ and $\mathbf v=\mathbf v^\pi$.
At each state $s$, the lemma supplies a finite set $\mathcal G_s$
and numbers $g_{s,t}\in\mathcal G_s$ satisfying
$0\le g_{s,t}\le E_t=\|\mathbf u^{\tau^t}-\mathbf v^\pi\|_\infty$.
Because rewards are fixed, the Bellman equation for $\mathbf v^\pi$
and the residual definition~\eqref{eq:nature-residual} give
\begin{equation*}
  0\le[\mathbf d^{\tau^{t+1}}]_s
  =\gamma\Bigl(
    [\mathbf P^{\tau^{t+1}}]_{s,:}\mathbf v^\pi
    -\min_{\mathbf p\in\mathcal U_{s,\pi(s)}}\mathbf p^\top\mathbf v^\pi
  \Bigr)
  \le\gamma g_{s,t},\qquad \forall s\in\mathcal S.
\end{equation*}
Taking $g_t=\max_s g_{s,t}$ and applying~\eqref{eq:residual-sandwich}
gives the comparison inequalities in~\eqref{eq:unified-scale-bound},
\begin{equation*}
  0\le g_t\le E_t,\qquad
  E_{t+1}\le\frac{\|\mathbf d^{\tau^{t+1}}\|_\infty}{1-\gamma}
  \le\frac{\gamma g_t}{1-\gamma}.
\end{equation*}
For weighted $\ell_1$ sets, Lemma~\ref{lem:budget-error}(i) gives
at most $8S^3$ positive numbers in each $\mathcal G_s$, hence at
most $8S^4$ in $\bigcup_s\mathcal G_s$.
For Wasserstein sets, $\mathbf v^\pi$ and $\mathbf C$ are the
same at every state, so Lemma~\ref{lem:budget-error}(ii) allows
$\mathcal G_s=\mathcal G$ for all $s$, with at most $2S^4$
positive numbers in $\mathcal G$.
In both cases, $\bigcup_s\mathcal G_s$ is independent of $t$
and contains $g_t=\max_s g_{s,t}$.
Hence there are $O(S^4)$ distinct integers
$\lfloor\log_2g_t\rfloor$ for $g_t>0$.
Setting $g_t=0$ after termination
verifies~\eqref{eq:unified-scale-bound} with $\mathsf G=O(S^4)$.
For $\gamma=0$, fixed rewards give $E_t=0$, so we take $g_t=0$.

\vspace{5pt}
\noindent\textbf{Total cost.}
Substituting~\eqref{eq:budget-local-parameters} into
Theorem~\ref{thm:residual-criterion} and using
$S^2\log(2+S)=O(S^3)$, we obtain the following bound on arithmetic
operations and comparisons,
\begin{equation}
  O\left(
    [1+(A-S)T_\gamma]
    \{(1+S^4T_\gamma)S^3+AS\log(2+S)\}
  \right).
  \label{eq:budget-general-bound}
\end{equation}
For fixed $\gamma$, this is~\eqref{eq:budget-full-bound}.
\end{proof}

\subsection{Turn-based stochastic games}
\label{subsec:structured-games}

We extend the analysis of Theorem~\ref{thm:residual-criterion}
to turn-based stochastic games whose reward-transition choices
have polyhedral descriptions.
The state partition and discounted objective are as in
Example~\ref{ex:turn-based-games}. At state $s$, the acting player
selects $(r,\mathbf p,\mathbf u)$ from
\begin{equation}
  \mathcal Q_s
  =\{\mathbf q=(r,\mathbf p,\mathbf u):
       \mathbf D_s\mathbf q\le\mathbf b_s,\quad
       \mathbf E_s\mathbf q=\mathbf f_s\}.
  \label{eq:implicit-game-model}
\end{equation}
As in~\eqref{eq:local-ambiguity}, the constraint matrices, right-hand sides,
and discount are rational, and the
reward-transition projection $\mathcal M_s$ is a nonempty compact
subset of $\mathbb R\times\Delta(\mathcal S)$. The Bellman operator is
\begin{equation}
  \mathscr T_s(\mathbf v)=
  \left\{\begin{aligned}
    &\max_{(r,\mathbf p,\mathbf u)\in\mathcal Q_s}
       \{r+\gamma\mathbf p^\top\mathbf v\},
       &&s\in\mathcal S_{\max},\\
    &\min_{(r,\mathbf p,\mathbf u)\in\mathcal Q_s}
       \{r+\gamma\mathbf p^\top\mathbf v\},
       &&s\in\mathcal S_{\min}.
  \end{aligned}\right.
  \label{eq:implicit-game-bellman}
\end{equation}
The game value $\mathbf v^\star$ is the unique fixed point
of $\mathscr T$, and
optimizers at $\mathbf v^\star$ give stationary optimal strategies
for both players.
For each state $s$, we write $n_s$ and $m_s$
for the number of variables and the total number
of inequality and equality constraints %
in~\eqref{eq:implicit-game-model}, respectively. We set
$N=S+\sum_s(n_s+m_s)$ and write $L$ for the total encoding length
of the rational input.

A stationary maximizing strategy $\pi$ selects a pair
$\pi(s)\in\mathcal M_s$ at each $s\in\mathcal S_{\max}$.
Fixing these pairs in~\eqref{eq:implicit-game-bellman}
defines the operator $\mathscr T_\pi$, whose unique fixed point
is $\mathbf v^\pi$. To compute $\mathbf v^\pi$, we apply
Algorithm~\ref{alg:nature-policy-iteration}, replacing its choice set
$\mathcal M_{s,\pi(s)}$ by $\{\pi(s)\}$ for $s\in\mathcal S_{\max}$
and by $\mathcal M_s$ for $s\in\mathcal S_{\min}$.
We write $\tau(s)$ for the reward-transition pair selected at state $s$
by $\tau$, so $\tau^t(s)=\pi(s)$ for $s\in\mathcal S_{\max}$
during this evaluation. Algorithm~\ref{alg:game-strategy-iteration}
combines these evaluations with maximizing-strategy updates.
We use $k$ for its iterations and $t$ for the iterations of
Algorithm~\ref{alg:nature-policy-iteration} within an evaluation.

\begin{algorithm}[t]
\caption{Strategy iteration for turn-based stochastic games}
\label{alg:game-strategy-iteration}
\begin{algorithmic}[1]
  \STATE Choose $\displaystyle
    \pi^0(s)\in\operatorname*{arg\,max}_{(r,\mathbf p)\in\mathcal M_s}r$,
    $\forall s\in\mathcal S_{\max}$.
    \label{line:game-initialization}
  \FOR{$k=0,1,\ldots$}
    \STATE Choose $\tau^0$ satisfying
      \label{line:game-evaluation-initialization}
      \[
        \left\{\begin{aligned}
          \tau^0(s)&=\pi^k(s), &&s\in\mathcal S_{\max},\\
          \tau^0(s)&\in\operatorname*{arg\,min}_{(r,\mathbf p)\in\mathcal M_s}r,
            &&s\in\mathcal S_{\min}.
        \end{aligned}\right.
      \]
    \STATE $(\mathbf v^{\pi^k},\tau)\gets{}$
      Algorithm~\ref{alg:nature-policy-iteration}$(\pi^k,\tau^0)$.
      \label{line:game-evaluation}
    \STATE Choose $\displaystyle
      \pi^{k+1}(s)\in\operatorname*{arg\,max}_{(r,\mathbf p)\in\mathcal M_s}
      \{r+\gamma\mathbf p^\top\mathbf v^{\pi^k}\}$,
      and record the maximum value $\mathscr T_s(\mathbf v^{\pi^k})$
      for each $s\in\mathcal S_{\max}$.
      \label{line:game-maximization}
      \label{line:game-strategy-update}
    \STATE \textbf{if} $\mathscr T_s(\mathbf v^{\pi^k})=[\mathbf v^{\pi^k}]_s$
      for all $s\in\mathcal S_{\max}$ \textbf{then return}
      $\mathbf v^{\pi^k}$, $\pi^k$, and $\{\tau(s)\}_{s\in\mathcal S_{\min}}$.
      \label{line:game-optimality-check}
  \ENDFOR
\end{algorithmic}
\end{algorithm}

At minimizing states,
$\mathscr T_s(\mathbf v^{\pi^k})=[\mathbf v^{\pi^k}]_s$
already holds by the definition of $\mathbf v^{\pi^k}$.
Thus Algorithm~\ref{alg:game-strategy-iteration},
Line~\ref{line:game-optimality-check} tests
$\mathscr T(\mathbf v^{\pi^k})=\mathbf v^{\pi^k}$,
and the choice in Line~\ref{line:game-strategy-update} satisfies
$\mathscr T_{\pi^{k+1}}(\mathbf v^{\pi^k})
=\mathscr T(\mathbf v^{\pi^k})$.

\paragraph{Cost and encoding assumptions.}
We use bounds
$\mathsf C_{\min}\ge\max\{1,|\mathcal S_{\min}|\}$ and
$\mathsf C_{\max}\ge\max\{1,|\mathcal S_{\max}|\}$ on the total number of
arithmetic operations and comparisons required to compute optimal
pairs and objective values for all minimizations and all maximizations
in~\eqref{eq:implicit-game-bellman}, respectively.
These bounds hold for every rational $\mathbf v$ and include all
required computations. The returned pairs have encoding length
polynomial in $L$, independently of $\mathbf v$.
All computed rationals have bit length
polynomial in $L$ and the encoding length of $\mathbf v$.

\begin{theorem}[Complexity of strategy iteration for turn-based stochastic games]
\label{thm:game-progress-criterion}
Let $\gamma\in[0,1)$ be rational. Assume the cost
and encoding requirements for~\eqref{eq:implicit-game-bellman}, with bounds
$\mathsf C_{\min}$ and $\mathsf C_{\max}$.
Suppose integers
$\mathsf G_{\min},\mathsf G_{\max}\ge0$ satisfy the following conditions
for Algorithm~\ref{alg:game-strategy-iteration}.
\begin{enumerate}[label=(\roman*),leftmargin=*]
\item For each strategy $\pi$ evaluated in
Algorithm~\ref{alg:game-strategy-iteration},
Line~\ref{line:game-evaluation}, let
$e_t^\pi=\|\mathbf u^{\tau^t}-\mathbf v^\pi\|_\infty$ be the error at
iteration $t$ of Algorithm~\ref{alg:nature-policy-iteration}, with
$e_t^\pi=0$ after that evaluation terminates.
Each run admits a nonnegative sequence $(g_t^\pi)_{t\ge0}$ satisfying
\begin{equation}
  \left\{\begin{aligned}
    &g_t^\pi\le e_t^\pi,\qquad
      e_{t+1}^\pi\le\frac{\gamma g_t^\pi}{1-\gamma},\quad t\ge0,\\
    &\left|\{\lfloor\log_2g_t^\pi\rfloor:
                       t\ge0,\ g_t^\pi>0\}\right|
       \le\mathsf G_{\min}.
  \end{aligned}\right.
  \label{eq:game-inner-scale-bound}
\end{equation}
\item For the maximizing strategies $\pi^k$ in
Algorithm~\ref{alg:game-strategy-iteration}, let
$E_k=\|\mathbf v^\star-\mathbf v^{\pi^k}\|_\infty$, with $E_k=0$
after termination.
There is a nonnegative sequence $(h_k)_{k\ge0}$ satisfying
\begin{equation}
  \left\{\begin{aligned}
    &h_k\le E_k,\qquad
      E_{k+1}\le\frac{\gamma h_k}{1-\gamma},\quad k\ge0,\\
    &\left|\{\lfloor\log_2h_k\rfloor:
                       k\ge0,\ h_k>0\}\right|
       \le\mathsf G_{\max}.
  \end{aligned}\right.
  \label{eq:game-outer-scale-bound}
\end{equation}
\end{enumerate}
Then Algorithm~\ref{alg:game-strategy-iteration} returns $\mathbf v^\star$ and rational stationary
optimal strategies for both players using
\begin{equation}
  O\left(
    (1+\mathsf G_{\max}T_\gamma)
    \left[
      (1+\mathsf G_{\min}T_\gamma)(S^3+\mathsf C_{\min})
      +\mathsf C_{\max}
    \right]
  \right)
  \label{eq:game-progress-total-cost}
\end{equation}
arithmetic operations and comparisons. Intermediate rationals have
bit length polynomial in $L$.
\end{theorem}

\begin{proof}
We first bound the number of iterations in each call to
Algorithm~\ref{alg:nature-policy-iteration} and the number of strategy
evaluations in Algorithm~\ref{alg:game-strategy-iteration}.
We then derive the total cost and encoding bounds.

\vspace{5pt}
\noindent\textbf{Contraction and termination.}
For fixed $\pi$, Algorithm~\ref{alg:nature-policy-iteration}
evaluates a minimization MDP, so~\eqref{eq:nature-contraction} gives
$e_{t+1}^\pi\le\gamma e_t^\pi$.
Since $\mathscr T_\pi(\mathbf v^\star)
\le\mathscr T(\mathbf v^\star)=\mathbf v^\star$, monotonicity and
contraction of $\mathscr T_\pi$ imply $\mathbf v^\pi\le\mathbf v^\star$.
Moreover, Algorithm~\ref{alg:game-strategy-iteration},
Line~\ref{line:game-strategy-update} gives
$\mathbf v^{\pi^k}=\mathscr T_{\pi^k}(\mathbf v^{\pi^k})
\le\mathscr T(\mathbf v^{\pi^k})
=\mathscr T_{\pi^{k+1}}(\mathbf v^{\pi^k})$.
Monotonicity and contraction of $\mathscr T_{\pi^{k+1}}$
therefore yield
\begin{equation*}
  \mathbf v^{\pi^k}\le\mathscr T(\mathbf v^{\pi^k})
       \le\mathbf v^{\pi^{k+1}}\le\mathbf v^\star.
\end{equation*}
Since $\mathscr T(\mathbf v^\star)=\mathbf v^\star$, contraction gives
$E_{k+1}\le\|\mathscr T(\mathbf v^\star)-
\mathscr T(\mathbf v^{\pi^k})\|_\infty\le\gamma E_k$.
The stopping tests in Algorithm~\ref{alg:nature-policy-iteration},
Line~\ref{line:pair-optimality-check} and
Algorithm~\ref{alg:game-strategy-iteration},
Line~\ref{line:game-optimality-check} hold exactly when $e_t^\pi=0$
and $E_k=0$, respectively. Applying Lemma~\ref{lem:finite-scale-progress}
to~\eqref{eq:game-inner-scale-bound} and~\eqref{eq:game-outer-scale-bound}
gives at most $\mathsf G_{\min}T_\gamma+2$ iterations of
Algorithm~\ref{alg:nature-policy-iteration} per evaluation and at most
$\mathsf G_{\max}T_\gamma+2$ evaluations of maximizing strategies,
including the stopping iterations.
At termination,
$\mathscr T_{\pi^k}(\mathbf v^{\pi^k})=\mathbf v^{\pi^k}
=\mathscr T(\mathbf v^{\pi^k})$, so
$\mathbf v^{\pi^k}=\mathbf v^\star$ and the pairs $\pi^k(s)$
attain the maxima in~\eqref{eq:implicit-game-bellman}.
Algorithm~\ref{alg:nature-policy-iteration} returns minimizing pairs
at this value, so both returned stationary strategies are optimal.

\vspace{5pt}
\noindent\textbf{Arithmetic operations and comparisons.}
Each iteration of Algorithm~\ref{alg:nature-policy-iteration}
uses $O(S^3)$ operations in Line~\ref{line:pair-value} and at most
$\mathsf C_{\min}$ operations and comparisons for the minimizations
at $s\in\mathcal S_{\min}$ in Line~\ref{line:pair-update}.
Forming the linear system and performing the stopping test in
Line~\ref{line:pair-optimality-check} cost $O(S^2)$, which is
absorbed by $O(S^3)$.
Including initialization in Algorithm~\ref{alg:game-strategy-iteration},
Line~\ref{line:game-evaluation-initialization}, each strategy evaluation
therefore costs $O((1+\mathsf G_{\min}T_\gamma)(S^3+\mathsf C_{\min}))$.
Lines~\ref{line:game-initialization} and~\ref{line:game-maximization}
of Algorithm~\ref{alg:game-strategy-iteration} each cost at most
$\mathsf C_{\max}$. The stopping test in
Algorithm~\ref{alg:game-strategy-iteration},
Line~\ref{line:game-optimality-check} reuses the objective values from
Line~\ref{line:game-maximization} and costs
$O(|\mathcal S_{\max}|)$ comparisons, also absorbed by $\mathsf C_{\max}$.
Together with the bound on the number of evaluated maximizing
strategies, these costs give~\eqref{eq:game-progress-total-cost}.

\vspace{5pt}
\noindent\textbf{Encoding length.}
The output encoding bounds for both players ensure that every
linear system in Algorithm~\ref{alg:nature-policy-iteration},
Line~\ref{line:pair-value} has coefficients of bit length
polynomial in $L$.
The encoding argument in the proof of
Theorem~\ref{thm:residual-criterion} therefore applies and gives
polynomial bit length in $L$ for all intermediate rationals.
\end{proof}

For games with fixed
rewards and the $\ell_1$, interval, weighted $\ell_1$, or Wasserstein
transition sets specified in Sections~\ref{subsec:norm-corollaries}
and~\ref{subsec:transport-corollaries}, %
the computation and error bounds established there verify the
hypotheses of Theorem~\ref{thm:game-progress-criterion}.
Algorithm~\ref{alg:game-strategy-iteration} is therefore strongly
polynomial at fixed discount for these games.
We omit the details to avoid
repetition.

\section{Concluding remarks}
\label{sec:boundaries}

We study LP representations and strongly polynomial
algorithms for rational polyhedral state-action rectangular RMDPs
with joint reward-transition uncertainty.
Our LP recovers the robust optimal value and all
optimal stationary randomized policies, with polynomial dimension
and a strongly polynomial construction at fixed discount. We also
develop a general complexity analysis of robust policy iteration
that yields improved or new strongly polynomial bounds for concrete
uncertainty sets at fixed discount.
The analysis extends to turn-based stochastic games.

Open directions include obtaining tighter complexity bounds and
identifying broader uncertainty classes that admit
strongly polynomial algorithms. One approach is to exploit the
structure of our LP representation.
We next examine the difficulties in obtaining polynomial-time LP
constructions when the discount is part of the input and in extending
rational affine recovery to state-rectangular uncertainty.

\subsection{Dependence on the discount factor}
\label{subsec:general-boundaries}

The bounds in Theorem~\ref{thm:exact-lp} depend on $T_\gamma$,
which can be exponential in the discount's bit length $L_\gamma$.
The following corollary relates polynomial-time LP construction
and recovery to solving discounted-sum games with the discount included
in the input. These games are the deterministic-transition case of
Example~\ref{ex:turn-based-games}.

\begin{corollary}[LP representations uniform in the discount]
\label{prop:parity-boundary}
Suppose every rational polyhedral state-action rectangular
RMDP with the discount included in the input admits a
polynomial-size exact rational LP representation with polynomial-time
construction and recovery of either the robust optimal value vector or
an optimal stationary policy.
Then discounted-sum games and parity games can be solved
in polynomial time.
\end{corollary}

\begin{proof}
\citet{SuilenPerez2026} reduce discounted-sum games to
rational polyhedral state-action rectangular RMDPs in polynomial bit time.
Under the hypothesis, we can construct and solve the LP and recover
the robust value or an optimal policy in polynomial bit time
\citep{Schrijver1986}. If only a policy is recovered,
we select a deterministic optimal policy from its support by~\eqref{eq:optimal-actions}.
Fixing this policy leaves a minimization MDP, whose value
can be computed in polynomial time. We can therefore solve discounted-sum
games in polynomial time. The reduction of \citet{Jurdzinski1998}
\looseness=-1 then gives a polynomial-time algorithm for solving parity games.
\end{proof}

Parity games admit quasipolynomial-time algorithms
\citep{CaludeEtAl2017}, but whether they can be solved in polynomial
time remains a longstanding open problem. An LP representation satisfying
the hypotheses of Corollary~\ref{prop:parity-boundary} would resolve this question.

\subsection{State-rectangular uncertainty}
\label{subsec:s-rectangular}

State rectangularity allows transition uncertainty to couple
different actions at the same state, while remaining independent across
states. We denote the nonempty polytope of joint transition distributions
at state $s$ by
$\mathcal U_s\subseteq\prod_{a\in\mathcal A_s}\Delta(\mathcal S)$.
For fixed rewards, the Bellman operator is
\begin{equation*}
  \mathscr T_s(\mathbf v)
   =\max_{\boldsymbol\pi_s\in\Delta(\mathcal A_s)}
      \min_{\{\mathbf p_{sa}\}_{a\in\mathcal A_s}\in\mathcal U_s}
        \sum_{a\in\mathcal A_s}[\boldsymbol\pi_s]_a
          (r_{sa}+\gamma\mathbf p_{sa}^\top\mathbf v).
\end{equation*}
Its unique fixed point is the robust optimal value
$\mathbf v^\star$. A stationary policy is optimal from every state
if and only if its action distribution at each state $s$
attains the maximum defining $\mathscr T_s(\mathbf v^\star)$
\citep{WiesemannKuhnRustem2013}.
The following example shows that the optimal value vector
and the unique optimal stationary policy can have irrational coordinates
even when all input coefficients are rational.

\begin{proposition}[Irrational values and optimal policies]
\label{prop:s-rectangular-irrationality}
There exists a three-state state-rectangular RMDP with fixed rational
rewards, rational polyhedral transition uncertainty, and
$\gamma=1/2$ such that $\mathbf v^\star$ has an irrational coordinate
and the unique optimal stationary randomized policy has an irrational
action probability.
\end{proposition}

\begin{proof}
We use states $s_1,s_2,s_3$ and discount $1/2$.
States $s_2,s_3$ each have one action and are absorbing, with rewards
zero and one. State $s_1$ has two actions $a,b$ with reward zero,
and joint transition uncertainty
\begin{equation*}
  \mathcal U_{s_1}=\operatorname{conv}
       \{(\mathbf e_{s_2},\mathbf e_{s_3}),
         (\mathbf e_{s_3},\mathbf e_{s_1})\}.
\end{equation*}
The first and second entries of each pair specify the
transitions for $a$ and $b$, respectively.
The absorbing values are $[\mathbf v^\star]_{s_2}=0$ and
$[\mathbf v^\star]_{s_3}=2$, and rewards in $[0,1]$ imply
$[\mathbf v^\star]_{s_1}\in[0,2]$. For a candidate value
$x\in[0,2]$ at $s_1$ and probability $z\in[0,1]$ of choosing action $a$,
the objective in the Bellman operator equals $1-z$ and
$z+(1-z)x/2$ at the two extreme transition pairs.
By linearity, the minimum over $\mathcal U_{s_1}$ is the
smaller of these two expressions. They are decreasing and nondecreasing
in $z$, respectively, so their minimum is maximized at their intersection
$z=\frac{2-x}{4-x}\in[0,1]$. Hence
\begin{equation*}
  \mathscr T_{s_1}(x,0,2)
   =\max_{0\le z\le1}\min\{1-z,\ z+(1-z)x/2\}
   =\frac{2}{4-x},\qquad 0\le x\le2.
\end{equation*}
The fixed-point equation $x=\frac{2}{4-x}$
has the unique solution $x=2-\sqrt2$ in $[0,2]$, giving
\begin{equation*}
  [\mathbf v^\star]_{s_1}=2-\sqrt2,\qquad
  [\boldsymbol\pi_{s_1}^\star]_a=\sqrt2-1,\qquad
  [\boldsymbol\pi_{s_1}^\star]_b=2-\sqrt2.
\end{equation*}
At $x=2-\sqrt2<2$, the two expressions are
strictly decreasing and strictly increasing in $z$, respectively, so the
maximizing probability is unique. Since the other
states each have one action, the optimal stationary policy is unique.
\end{proof}

Proposition~\ref{prop:s-rectangular-irrationality} rules out an exact rational LP
representation of the robust optimal value or optimal stationary policies
for general state-rectangular RMDPs. Indeed, every rational LP with a finite
attained optimum has an optimal solution whose coordinates are all rational.
These coordinates cannot give the optimal value vector or the unique optimal
policy in Proposition~\ref{prop:s-rectangular-irrationality}, both of which
have irrational coordinates.

\subsection*{AI use statement}
We used generative AI to assist with some derivations,
literature searches, proof checking, language editing, and \LaTeX{} formatting.
The authors carefully reviewed all AI-assisted material and independently
verified the mathematical arguments. The authors take full responsibility
for this paper.

\bibliographystyle{ims}
\bibliography{references}

\appendix

\section{Constants and encoding bounds for Theorem~\ref{thm:exact-lp}}
\label{app:lp-arithmetic}

We define the value bound $V$, the multiplier bound $M$, the gap bound
$\varepsilon$ in~\eqref{eq:lp-policy-gap}, and the objective weight
$W$ in~\eqref{eq:lp-pi-objective} directly from the input.
We then prove the encoding and construction bounds in
Theorem~\ref{thm:exact-lp}.

\subsection{Input denominators and coefficient bounds}
\label{app:lp-primitive-bounds}

We write every supplied rational coefficient as \(a_i/b_i\), with
\(a_i\in\mathbb Z\) and \(b_i\in\mathbb Z_{>0}\).  This list includes
all entries of \(\mathbf D_{sa},\mathbf b_{sa},\mathbf E_{sa},
\mathbf f_{sa}\) and the discount \(\gamma\).
In this appendix, we use
$D_{\mathrm{in}}$ for a common input denominator and $C$ for a bound
on the coefficients after clearing denominators, chosen as
\begin{equation}
  D_{\mathrm{in}}=\prod_i b_i,\qquad
  C=D_{\mathrm{in}}\left(2+\sum_i|a_i|\right).
  \label{eq:lp-clearing-constants}
\end{equation}
These are positive integers with \(C\ge2\) and
\(D_{\mathrm{in}}\le C\).  Each cleared coefficient is
an integer satisfying
\(\left|D_{\mathrm{in}}a_i/b_i\right|
 \le D_{\mathrm{in}}|a_i|\le C\).
This bound also covers \(0,\pm D_{\mathrm{in}}\).
Moreover, \(\log D_{\mathrm{in}}=O(L)\) and
\(\log C=O(L+\log N)\).

The following support bound applies even when the right-hand side is
real. We use it to bound dual multipliers uniformly over the entire
value box.

\begin{lemma}[A support bound with real right-hand sides]
\label{lem:lp-support-bound}
Let $\mathbf G\in\mathbb Q^{m\times n}$ and $\mathbf h\in\mathbb R^m$.
Suppose $\mathbf G\mathbf y=\mathbf h$, $\mathbf y\ge\mathbf0$ is
feasible and $D_{\mathrm{in}}\mathbf G$ has integer entries of
magnitude at most $C$.
There is a feasible $\mathbf y$ whose support
$\mathcal I=\{i:[\mathbf y]_i>0\}$ indexes linearly independent columns of
$\mathbf G$. If $\mathcal I\ne\varnothing$, then
\begin{equation*}
  \|\mathbf y\|_\infty\le |\mathcal I|!C^{|\mathcal I|-1}D_{\mathrm{in}}
                      \|\mathbf h\|_\infty.
\end{equation*}
\end{lemma}

\begin{proof}
We choose a feasible $\mathbf y$ with the fewest positive coordinates
and write $\mathcal I$ for its support.
If \([\mathbf G]_{:,\mathcal I}\) has dependent columns, there is a nonzero
$\mathbf d$ supported on $\mathcal I$ with $\mathbf G\mathbf d=\mathbf0$.
We choose its sign so that some $[\mathbf d]_i>0$ and take
$\theta=\min_{i\in\mathcal I:\,[\mathbf d]_i>0}[\mathbf y]_i/[\mathbf d]_i$.
The vector $\mathbf y-\theta\mathbf d$ satisfies
\begin{equation*}
  \mathbf y-\theta\mathbf d\ge\mathbf0,
  \qquad
  \mathbf G(\mathbf y-\theta\mathbf d)=\mathbf h.
\end{equation*}
At least one positive coordinate becomes zero, contradicting our
choice of $\mathbf y$.
Thus $[\mathbf G]_{:,\mathcal I}$ has linearly independent columns.

For nonempty $\mathcal I$, we select $|\mathcal I|$ rows indexed by
\(\mathcal J\) such that
\(\mathbf B=D_{\mathrm{in}}[\mathbf G]_{\mathcal J,\mathcal I}\) is nonsingular.
Then \(\mathbf B[\mathbf y]_{\mathcal I}=D_{\mathrm{in}}[\mathbf h]_{\mathcal J}\) and
\(|\det\mathbf B|\ge1\). Cramer's rule expresses each positive
coordinate as a determinant ratio. Each cofactor of $\mathbf B$ is a
determinant of order $|\mathcal I|-1$. Its expansion contains
$(|\mathcal I|-1)!$ signed products, each of magnitude at most
$C^{|\mathcal I|-1}$. Expanding the numerator along
the column $D_{\mathrm{in}}[\mathbf h]_{\mathcal J}$ gives
\[
  \|\mathbf y\|_\infty
  \le |\mathcal I|\,D_{\mathrm{in}}\|\mathbf h\|_\infty
       (|\mathcal I|-1)!C^{|\mathcal I|-1}
  =|\mathcal I|!C^{|\mathcal I|-1}D_{\mathrm{in}}\|\mathbf h\|_\infty.
\]
For empty support, \(\mathbf y=\mathbf0\).
\end{proof}

\subsection{Bounding values and optimal dual multipliers}
\label{app:lp-dual-bounds}

We construct the bounds used in~\eqref{eq:lp-value-invariance}
and the dual representation~\eqref{eq:lp-uniform-dual}.
Each $\mathbf q=(r,\mathbf p,\mathbf u)$ has at most $N$ coordinates.
We define the reward, value, and multiplier bounds by
\begin{equation}
  R:=(N+1)!C^{N+1},\qquad
  V:=1+RD_{\mathrm{in}},\qquad M:=RV.
  \label{eq:lp-arithmetic-constants}
\end{equation}
These constants are positive integers. Since $D_{\mathrm{in}}\gamma$
is an integer and $\gamma<1$, we have $D_{\mathrm{in}}(1-\gamma)\ge1$.
Thus $D_{\mathrm{in}}$ also bounds $1/(1-\gamma)$.

\begin{lemma}[Bounds on rewards, values, and dual multipliers]
\label{lem:lp-uniform-boxes}
Let $R,V,M$ be given by~\eqref{eq:lp-arithmetic-constants}.
For every $(s,a)$ and every $(r,\mathbf p,\mathbf u)\in\mathcal Q_{sa}$,
$|r|\le R$. The value of every stationary randomized policy and
$\mathbf v^\star$ belong to $[-V,V]^S$.
For every $\mathbf v\in[-V,V]^S$ and $(s,a)$,
$\mathscr T_{sa}(\mathbf v)\in[-V,V]$, and the dual representation~\eqref{eq:lp-uniform-dual} holds.
\end{lemma}

\begin{proof}
We first bound the rewards and values, then construct bounded optimal dual multipliers.

\vspace{5pt}
\noindent\textbf{Rewards and values.}
For each \(\sigma\in\{-1,1\}\),
\(\max_{(r,\mathbf p,\mathbf u)\in\mathcal Q_{sa}}\{\sigma r\}\)
is finite and attained because the reward-transition image is compact.
Using the description of $\mathcal Q_{sa}$ in~\eqref{eq:local-ambiguity},
LP duality therefore gives feasible multipliers satisfying
\(\mathbf D_{sa}^{\top}\boldsymbol\mu+
\mathbf E_{sa}^{\top}\boldsymbol\nu=(\sigma,\mathbf0,\mathbf0)\),
with \(\boldsymbol\mu\ge\mathbf0\) and $\boldsymbol\nu$ unrestricted in sign.
After splitting \(\boldsymbol\nu=\boldsymbol\nu^+-
\boldsymbol\nu^-\), the nonnegative system has at most \(N\) rows
and right-hand side of norm one.
Applying Lemma~\ref{lem:lp-support-bound} to this system, we can choose
$(\boldsymbol\mu,\boldsymbol\nu^+,\boldsymbol\nu^-)$ with at most
\(N\) nonzero coordinates, each bounded by \(N!C^{N-1}D_{\mathrm{in}}\).
The entries of $\mathbf b_{sa}$ and $\mathbf f_{sa}$ have magnitude at most
$C/D_{\mathrm{in}}$. Together with the support and coordinate bounds
above, weak duality gives
\[
\begin{aligned}
  \sigma r
  &\le\mathbf b_{sa}^{\top}\boldsymbol\mu+
       \mathbf f_{sa}^{\top}(\boldsymbol\nu^+-\boldsymbol\nu^-)\\
  &\le \frac{NC}{D_{\mathrm{in}}}
       N!C^{N-1}D_{\mathrm{in}}
   =N\,N!C^N\le R.
\end{aligned}
\]
Both signs therefore give \(|r|\le R\).
Together with stochasticity, this bounds both
$\|\mathscr T_{\pi}(\mathbf v)\|_\infty$ and
$\|\mathscr T(\mathbf v)\|_\infty$ by $R+\gamma\|\mathbf v\|_\infty$.
The fixed-point equations for $\mathbf v^{\pi}$ and
$\mathbf v^\star$ therefore give
\[
  \max\{\|\mathbf v^{\pi}\|_\infty,
          \|\mathbf v^\star\|_\infty\}
  \le\frac{R}{1-\gamma}
  \le RD_{\mathrm{in}}<V.
\]
Since \((1-\gamma)V\ge R\),
every \(\mathbf v\in[-V,V]^S\) and feasible reward-transition pair
satisfy \(|r+\gamma\mathbf p^\top\mathbf v|\le R+\gamma V\le V\).
This proves~\eqref{eq:lp-value-invariance}.

\vspace{5pt}
\noindent\textbf{Optimal dual multipliers.}
For a real \(\mathbf v\in[-V,V]^S\), strong duality for
$(\mathrm P)$ and $(\mathrm D)$ in~\eqref{eq:lp-uniform-dual}, before
imposing multiplier bounds, gives optimal multipliers.
To preserve optimality
when applying Lemma~\ref{lem:lp-support-bound}, we fix the dual
objective at $\mathscr T_{sa}(\mathbf v)$.
After splitting $\boldsymbol\nu=\boldsymbol\nu^+-\boldsymbol\nu^-$
into nonnegative parts, we obtain the feasible system
\begin{equation}
\left\{\begin{aligned}
 &{-}\mathbf D_{sa}^{\top}\boldsymbol\mu
   +\mathbf E_{sa}^{\top}\boldsymbol\nu^+
   -\mathbf E_{sa}^{\top}\boldsymbol\nu^-
   =(1,\gamma\mathbf v,\mathbf0),\\
 &{-}\mathbf b_{sa}^{\top}\boldsymbol\mu
   +\mathbf f_{sa}^{\top}\boldsymbol\nu^+
   -\mathbf f_{sa}^{\top}\boldsymbol\nu^-
   =\mathscr T_{sa}(\mathbf v).
\end{aligned}\right.
 \label{eq:lp-optimal-certificate-system}
\end{equation}
Here $\boldsymbol\mu,\boldsymbol\nu^+,\boldsymbol\nu^-\ge\mathbf0$.
The system~\eqref{eq:lp-optimal-certificate-system} has at most
\(N+1\) equations and a right-hand side of infinity norm at most \(V\).
Multiplying its coefficient matrix by $D_{\mathrm{in}}$ gives integer
entries of magnitude at most $C$.
Lemma~\ref{lem:lp-support-bound} therefore gives a nonnegative solution
$(\boldsymbol\mu,\boldsymbol\nu^+,\boldsymbol\nu^-)$ satisfying
\[
  \max\{\|\boldsymbol\mu\|_\infty,
         \|\boldsymbol\nu^+\|_\infty,\|\boldsymbol\nu^-\|_\infty\}
  \le(N+1)!C^ND_{\mathrm{in}}V
  \le M.
\]
Thus \(\boldsymbol\mu\le M\mathbf1\) and
\(-M\mathbf1\le\boldsymbol\nu^+-\boldsymbol\nu^-\le M\mathbf1\).
This proves~\eqref{eq:lp-uniform-dual} for every real
$\mathbf v\in[-V,V]^S$.
\end{proof}

\subsection{A uniform lower bound on positive action gaps}
\label{app:lp-action-separation}

For the action constraints in~\eqref{eq:lp-pi-flag}, we establish
the uniform gap~\eqref{eq:lp-policy-gap} with the positive rational number
\begin{equation}
  \varepsilon:=\frac{1}{(2N)!C^{2N}}.
  \label{eq:lp-gap-constant}
\end{equation}

\begin{lemma}[Separation at every deterministic-policy value]
\label{lem:lp-policy-separation}
For every deterministic stationary policy \(\pi\), all numbers
$[\mathbf v^\pi]_s$ and $\mathscr T_{sa}(\mathbf v^\pi)$, with
$s\in\mathcal S$ and $a\in\mathcal A_s$, have a common positive
integer denominator at most $\varepsilon^{-1}$.
Consequently,~\eqref{eq:lp-policy-gap} holds with gap $\varepsilon$.
\end{lemma}

\begin{proof}
For a fixed deterministic policy $\pi$, we use the lexicographic
LP~\eqref{eq:lp-pi-block} to obtain $\mathbf v^\pi$ and all
$\mathscr T_{sa}(\mathbf v^\pi)$ at a common vertex.
We split each unrestricted multiplier as
$\boldsymbol\nu_{sa}=\boldsymbol\nu^+_{sa}-\boldsymbol\nu^-_{sa}$
and use the bounds from Lemma~\ref{lem:lp-uniform-boxes} to obtain
the following lexicographic LP.
\begin{equation}
  \begin{alignedat}{2}
    &\LPopt{lex\,max}{\mathbf v,\mathbf t,\{\boldsymbol\mu_{sa},\boldsymbol\nu^+_{sa},\boldsymbol\nu^-_{sa}\}_{s,a}}
    &\quad&\left(\sum_s[\mathbf v]_s,\ \sum_{s,a}[\mathbf t]_{sa}\right)\\
    &\text{\rm s.t.}
    &\quad&\left\{\begin{aligned}
      &{-}\mathbf D_{sa}^{\top}\boldsymbol\mu_{sa}
        +\mathbf E_{sa}^{\top}\boldsymbol\nu^+_{sa}
        -\mathbf E_{sa}^{\top}\boldsymbol\nu^-_{sa}
        =(1,\gamma\mathbf v,\mathbf0),\\
      &[\mathbf t]_{sa}\le-\mathbf b_{sa}^{\top}\boldsymbol\mu_{sa}
           +\mathbf f_{sa}^{\top}\boldsymbol\nu^+_{sa}
           -\mathbf f_{sa}^{\top}\boldsymbol\nu^-_{sa},\\
      &[\mathbf v]_s\le [\mathbf t]_{s,\pi(s)},\\
      &-V\le [\mathbf v]_s\le V,\qquad
        -V\le[\mathbf t]_{sa}\le V,\\
      &\mathbf0\le\boldsymbol\mu_{sa},\boldsymbol\nu^+_{sa},
                     \boldsymbol\nu^-_{sa}\le M\mathbf1.
    \end{aligned}\right.
  \end{alignedat}
  \label{eq:lp-separation-polytope}
\end{equation}
The constraints hold for every $s\in\mathcal S$ and $a\in\mathcal A_s$,
with each state constraint included once. The
\(S+A\) coordinates of $\mathbf v$ and $\mathbf t$, one multiplier per original
inequality, and two per equality give at most \(2N\) coordinates.
Weak duality and the policy inequalities give
\begin{equation*}
  [\mathbf t]_{sa}\le\mathscr T_{sa}(\mathbf v),\qquad
  [\mathbf v]_s\le[\mathbf t]_{s,\pi(s)}
    \le[\mathscr T_\pi(\mathbf v)]_s.
\end{equation*}
Monotonicity and contraction imply that every feasible $\mathbf v$
satisfies $\mathbf v\le\mathbf v^\pi$.
The proof of Lemma~\ref{lem:lp-uniform-boxes} gives nonnegative
multipliers satisfying~\eqref{eq:lp-optimal-certificate-system} at $\mathbf v^\pi$,
with every coordinate at most $M$.
Since $\mathbf v^\pi=\mathscr T_\pi(\mathbf v^\pi)$, these multipliers
make $\mathbf v=\mathbf v^\pi$ and
$[\mathbf t]_{sa}=\mathscr T_{sa}(\mathbf v^\pi)$ feasible
in~\eqref{eq:lp-separation-polytope}.
Maximizing $\sum_s[\mathbf v]_s$ therefore forces $\mathbf v=\mathbf v^\pi$.
With $\mathbf v$ fixed at this value, every $[\mathbf t]_{sa}$ is bounded
above by $\mathscr T_{sa}(\mathbf v^\pi)$, and the same feasible assignment
attains all these bounds. Maximizing $\sum_{s,a}[\mathbf t]_{sa}$
therefore forces equality for every $(s,a)$.

The feasible region of~\eqref{eq:lp-separation-polytope} is nonempty
and bounded, so we choose a vertex attaining the lexicographic optimum.
Its coordinates are uniquely determined by linearly independent
constraints of~\eqref{eq:lp-separation-polytope} holding with equality.
Multiplying these equations by $D_{\mathrm{in}}$ gives a nonsingular
integer coefficient matrix $\mathbf B$ of order at most $2N$ and an
integer right-hand side. The entries of $\mathbf B$ have magnitude at
most $C$, since $V$ and $M$ occur only on the right-hand side.
Cramer's rule expresses every coordinate as an integer divided by
$\det\mathbf B$. This determinant is a nonzero integer, and its
expansion contains at most $(2N)!$ products, each of magnitude at
most $C^{2N}$. Thus
\begin{equation*}
  1\le|\det\mathbf B|\le(2N)!C^{2N}=\varepsilon^{-1}.
\end{equation*}
At the chosen vertex, $\mathbf v=\mathbf v^\pi$ and
$[\mathbf t]_{sa}=\mathscr T_{sa}(\mathbf v^\pi)$, so these numbers
have the common denominator $|\det\mathbf B|$.
Since $\mathscr T_s(\mathbf v^\pi)$ equals one of the action values,
every positive gap
$\mathscr T_s(\mathbf v^\pi)-\mathscr T_{sa}(\mathbf v^\pi)$
is at least $1/|\det\mathbf B|\ge\varepsilon$.
This proves~\eqref{eq:lp-policy-gap}.
\end{proof}

\subsection{Choosing the weight for the single LP objective}
\label{app:lp-scalarization}

To choose $W$ in~\eqref{eq:lp-pi-objective}, we first bound the
denominators of the vertex coordinates of $\mathcal P$. As in the
proof of Lemma~\ref{lem:lp-policy-separation}, this requires bounds
on the number of variables and on the constraint coefficients after
clearing denominators. By~\eqref{eq:lp-pi-counts}, there are at most
$3(AT_\gamma+1)N$ variables.

Apart from the input coefficients, the constraints
in~\eqref{eq:lp-pi-block} and~\eqref{eq:lp-pi-flag}--\eqref{eq:lp-pi-terminal}
use only coefficients of magnitude at most $\max\{1,2V,\varepsilon^{-1}\}$.
Multiplying each constraint by $D_{\mathrm{in}}$ gives integer
coefficients and integer right-hand sides. The matrix entries inherited
from the input then have magnitude at most $C$, while the remaining
entries have magnitude at most
$D_{\mathrm{in}}\max\{1,2V,\varepsilon^{-1}\}$.
Since $D_{\mathrm{in}}\le C$ and $2V\le M$, every matrix entry has
magnitude at most $C(M+\varepsilon^{-1}+1)$.
The determinant argument from Lemma~\ref{lem:lp-policy-separation}
therefore shows that the coordinates of each vertex have a common
denominator no greater than
\begin{equation}
  D_{\mathrm{LP}}=[3(AT_\gamma+1)N]!
    \bigl[C(M+\varepsilon^{-1}+1)\bigr]^{3(AT_\gamma+1)N}.
  \label{eq:lp-final-scalarization-constants}
\end{equation}
Each $f_j$ in~\eqref{eq:lp-pi-priorities} has integer coefficients.
Its value at each vertex therefore has denominator at most
$D_{\mathrm{LP}}$, so any nonzero difference
$f_j(\mathbf x)-f_j(\mathbf y)$ between two vertices has magnitude
at least $D_{\mathrm{LP}}^{-2}$.

For any $\mathbf x\in\mathcal P$ and $k=0,\ldots,AT_\gamma$, the variable bounds
in~\eqref{eq:lp-pi-value-bounds} and~\eqref{eq:lp-pi-flag} give
\(|f_{5k+1}(\mathbf x)|\le SV\), \(|f_{5k+2}(\mathbf x)|\le AV\),
\(|f_{5k+3}(\mathbf x)|\le SV\), and \(|f_{5k+4}(\mathbf x)|\le A\).
For $k<AT_\gamma$, the definition of $f_{5k+5}$
in~\eqref{eq:lp-pi-priorities} and the simplex constraints
in~\eqref{eq:lp-pi-next} give
\begin{equation*}
  |f_{5k+5}(\mathbf x)|
  =\sum_s\sum_a\omega(s,a)[\boldsymbol\pi_s^{k+1}]_a
  \le\sum_s|\mathcal A_s|=A.
\end{equation*}
Since \(S\le A\) and \(V\ge1\),
\begin{equation}
 |f_j(\mathbf x)|\le AV,
 \qquad \forall\mathbf x\in\mathcal P,\ j\in[5AT_\gamma+4].
 \label{eq:lp-priority-range}
\end{equation}
Thus $|f_j(\mathbf x)-f_j(\mathbf y)|\le2AV$ for any two vertices.
We choose the positive integer
\begin{equation}
  W=2AVD_{\mathrm{LP}}^2+2.
  \label{eq:lp-scalarization-weight}
\end{equation}
This choice ensures $2AV/(W-1)<D_{\mathrm{LP}}^{-2}$.
Together with the vertex denominator
bound~\eqref{eq:lp-final-scalarization-constants} and the objective
bound~\eqref{eq:lp-priority-range}, the argument
in Section~\ref{subsec:lp-pi-exactness} shows that the weighted
objective in~\eqref{eq:lp-pi-objective} has exactly the lexicographic maximizers.

\subsection{Encoding length and construction time of the LP}
\label{app:lp-encoding}

We prove the encoding and construction bounds in
Theorem~\ref{thm:exact-lp} for the LP in~\eqref{eq:lp-pi-objective}.
The bounds in Appendix~\ref{app:lp-primitive-bounds} give
$\log D_{\mathrm{in}}=O(L)$ and $\log C=O(L+\log N)$.
The definitions of $R,V,M$ in~\eqref{eq:lp-arithmetic-constants}
and of $\varepsilon$ in~\eqref{eq:lp-gap-constant},
together with $m!\le m^m$, give
\begin{equation}
  \log R+\log V+\log M+\log\varepsilon^{-1}
   =O\bigl(N(L+\log N)\bigr).
  \label{eq:lp-primitive-bit-bounds}
\end{equation}
Thus $\log[C(M+\varepsilon^{-1}+1)]=O(N(L+\log N))$.
Applying this estimate to~\eqref{eq:lp-final-scalarization-constants}
and~\eqref{eq:lp-scalarization-weight} gives
\begin{equation}
 \log D_{\mathrm{LP}}+\log W
   =O\bigl(AN^2T_\gamma[L+\log(AN\,T_\gamma)]\bigr).
 \label{eq:lp-weight-bit-bound}
\end{equation}
Each variable appears
in at most one objective $f_j$ in~\eqref{eq:lp-pi-priorities}, with
coefficient of magnitude at most $A$. The largest power of $W$
in~\eqref{eq:lp-pi-objective} is $W^{5AT_\gamma+3}$, so each nonzero
objective coefficient has magnitude at most $AW^{5AT_\gamma+3}$
and bit length
\begin{equation}
 O\bigl(\log A+AT_\gamma\log W\bigr)
   =O\bigl(A^2N^2T_\gamma^2[L+\log(AN\,T_\gamma)]\bigr).
 \label{eq:lp-objective-bit-bound}
\end{equation}
Every coefficient and right-hand-side entry in the constraints
in~\eqref{eq:lp-pi-block} and~\eqref{eq:lp-pi-flag}--\eqref{eq:lp-pi-terminal}
has a denominator dividing $D_{\mathrm{in}}$. After multiplication by
$D_{\mathrm{in}}$, each has magnitude at most
$C(M+\varepsilon^{-1}+1)$.
Its bit length is therefore $O(N(L+\log N))$.
By~\eqref{eq:lp-pi-counts}, a dense description of the LP has
\(O((AN\,T_\gamma)^2)\) entries. Together
with~\eqref{eq:lp-objective-bit-bound} and $A\le N$, these bounds
give total encoding length polynomial in $N,L,T_\gamma$.

We next bound the arithmetic cost of constructing~\eqref{eq:lp-pi-objective}
and the bit lengths of intermediate numbers.
We first compute $T_\gamma$ by successive multiplications by $\gamma$
and stop at the first $k\ge1$ satisfying
$\gamma^k<(1-\gamma)/2$, as in~\eqref{eq:discount-block}.
Maintaining the numerator and denominator of each power and comparing
by cross-multiplication uses $O(T_\gamma)$ integer operations on
$O(T_\gamma L_\gamma)$-bit integers.
Forming $D_{\mathrm{in}}$ and $C$ from the input numerators and
denominators in~\eqref{eq:lp-clearing-constants} uses $O(N^2)$
arithmetic operations and comparisons.
Each factorial or power in~\eqref{eq:lp-arithmetic-constants},~\eqref{eq:lp-gap-constant},
and~\eqref{eq:lp-final-scalarization-constants}
involves at most $3(AT_\gamma+1)N$ factors.
Evaluating these formulas and~\eqref{eq:lp-scalarization-weight}
therefore computes the remaining constants using polynomially many
operations in $N,T_\gamma$.
We obtain the objective weights $1,W,\ldots,W^{5AT_\gamma+3}$
by repeatedly multiplying by $W$.
Inserting these constants and the input coefficients
into~\eqref{eq:lp-pi-objective} completes the LP, whose number of entries
is polynomial in $N,T_\gamma$ by~\eqref{eq:lp-pi-counts}.
The bounds in~\eqref{eq:lp-primitive-bit-bounds}--\eqref{eq:lp-objective-bit-bound}
also apply to intermediate sums and products in computing the integer
constants and objective weights, since all summands are nonnegative
and all factors are positive integers.
Together with the preceding bounds for $T_\gamma,D_{\mathrm{in}},C$,
this gives polynomial bit length in $N,L,T_\gamma$ for all intermediate
numbers in the LP construction, and hence construction time polynomial
in $N,L,T_\gamma$.

\section{Completion of the proof of Proposition~\ref{prop:general-lp-hardness}}
\label{app:general-lp-reduction}

\begin{proof}
By~\eqref{eq:general-lp-embedding}, the RMDP algorithm computes
the minimum of a rational linear function over any nonempty bounded
rational polyhedron. We show how to use these values to test
feasibility, detect unboundedness, and recover an optimal solution
of a general rational LP. Given
$\mathbf A\in\mathbb Q^{m\times d}$, $\mathbf b\in\mathbb Q^m$, and
$\mathbf c\in\mathbb Q^d$, we consider a general rational LP and its dual
\begin{equation}
  \begin{alignedat}[t]{3}
    &(\mathrm P)\quad &&{\min\nolimits_{\mathbf x\in\mathbb R^d}}
    &\quad&\mathbf c^\top\mathbf x\\
    &&&\text{\rm s.t.}
    &\quad&\mathbf A\mathbf x\ge\mathbf b,
  \end{alignedat}
  \qquad
  \begin{alignedat}[t]{3}
    &(\mathrm D)\quad &&{\max\nolimits_{\mathbf y\in\mathbb R^m}}
    &\quad&\mathbf b^\top\mathbf y\\
    &&&\text{\rm s.t.}
    &\quad&\mathbf A^\top\mathbf y=\mathbf c,\quad \mathbf y\ge\mathbf0.
  \end{alignedat}
  \label{eq:general-lp-primal-dual}
\end{equation}

\vspace{5pt}
\noindent\textbf{Testing feasibility.}
We test feasibility of $(\mathrm P)$
in~\eqref{eq:general-lp-primal-dual} using the LP
\begin{equation}
  \begin{alignedat}{2}
    &\LPopt{min}{\mathbf x,t}
    &\quad&-t\\
    &\text{\rm s.t.}
    &\quad&\left\{\begin{aligned}
      &\mathbf A\mathbf x\ge t\mathbf b,\\
      &-\mathbf1\le\mathbf x\le\mathbf1,\quad 0\le t\le1.
    \end{aligned}\right.
  \end{alignedat}
  \label{eq:bounded-lp-feasibility}
\end{equation}
The feasible set of~\eqref{eq:bounded-lp-feasibility} contains the origin
and is compact. Using reward $r=-\frac{t}{2}$, the RMDP algorithm
therefore computes the minimum in~\eqref{eq:bounded-lp-feasibility}.
Any solution $\mathbf x$ of $\mathbf A\mathbf x\ge\mathbf b$ gives a
feasible point $(t\mathbf x,t)$ of~\eqref{eq:bounded-lp-feasibility}
with $t=(1+\|\mathbf x\|_\infty)^{-1}>0$.
Conversely, a feasible point $(\mathbf x,t)$ with $t>0$ gives the
solution $\mathbf x/t$ of the original system.
Thus $(\mathrm P)$ is feasible exactly when the minimum
in~\eqref{eq:bounded-lp-feasibility} is negative.
If this minimum is zero, we report infeasibility and stop.

\vspace{5pt}
\noindent\textbf{Optimality and solution recovery.}
When $(\mathrm P)$ is feasible, we determine whether it has a finite
optimum and, if so, recover an optimal solution.
We test feasibility of the system
\begin{equation}
  \mathbf A\mathbf x\ge\mathbf b,\qquad
  \mathbf A^\top\mathbf y=\mathbf c,\qquad
  \mathbf y\ge\mathbf0,\qquad
  \mathbf c^\top\mathbf x=\mathbf b^\top\mathbf y.
  \label{eq:general-lp-optimality-system}
\end{equation}
Any solution $(\mathbf x,\mathbf y)$ is feasible for $(\mathrm P)$
and $(\mathrm D)$ with equal objective values, so $\mathbf x$ is
optimal for $(\mathrm P)$ by weak duality.
Strong duality guarantees such a pair whenever $(\mathrm P)$
has a finite optimum.
To perform this test using~\eqref{eq:general-lp-embedding},
we scale both $\mathbf x$ and $\mathbf y$, as
in~\eqref{eq:bounded-lp-feasibility}, and consider
\begin{equation}
  \begin{alignedat}{2}
    &\LPopt{min}{\mathbf x,\mathbf y,t}
    &\quad&-t\\
    &\text{\rm s.t.}
    &\quad&\left\{\begin{aligned}
      &\mathbf A\mathbf x\ge t\mathbf b,\quad
        \mathbf A^\top\mathbf y=t\mathbf c,\\
      &\mathbf c^\top\mathbf x=\mathbf b^\top\mathbf y,\\
      &-\mathbf1\le\mathbf x\le\mathbf1,\quad
        \mathbf0\le\mathbf y\le\mathbf1,\quad 0\le t\le1.
    \end{aligned}\right.
  \end{alignedat}
  \label{eq:bounded-lp-optimality-test}
\end{equation}
The origin is feasible and the feasible set is compact,
so we compute the optimal value $[\mathbf v^\star]_{s_1}$
of~\eqref{eq:bounded-lp-optimality-test}
using~\eqref{eq:general-lp-embedding}.
Any solution $(\mathbf x,\mathbf y)$
of~\eqref{eq:general-lp-optimality-system} gives a feasible point
$(t\mathbf x,t\mathbf y,t)$ of~\eqref{eq:bounded-lp-optimality-test}
for sufficiently small $t>0$.
Conversely, a feasible point $(\mathbf x,\mathbf y,t)$ with $t>0$
gives the solution $(\mathbf x/t,\mathbf y/t)$
of~\eqref{eq:general-lp-optimality-system}.
Thus~\eqref{eq:general-lp-optimality-system} is feasible exactly when
$[\mathbf v^\star]_{s_1}<0$.
If $[\mathbf v^\star]_{s_1}=0$, the already feasible problem
$(\mathrm P)$ has no finite optimum, so we report that it is unbounded below.
If $[\mathbf v^\star]_{s_1}<0$, we add the equality
$t=-[\mathbf v^\star]_{s_1}$ to the constraints
of~\eqref{eq:bounded-lp-optimality-test}.
Applying the construction in~\eqref{eq:general-lp-embedding}
to this feasible set with reward $r=[\mathbf x]_1/{2}$
returns the minimum of the first coordinate $[\mathbf x]_1$.
We add an equality fixing $[\mathbf x]_1$ at this value and
repeat the construction for each remaining coordinate of $\mathbf x$
and $\mathbf y$, subject to all previously added equalities.
Each minimum is attained on a nonempty compact set, so each new
equality preserves nonemptiness.
Once all coordinates are determined, we return $\mathbf x/t$,
which is optimal for $(\mathrm P)$
by~\eqref{eq:general-lp-optimality-system}.

\vspace{5pt}
\noindent\textbf{Complexity.}
The procedure makes at most $m+d+2$ calls to the RMDP algorithm,
one for each of~\eqref{eq:bounded-lp-feasibility}
and~\eqref{eq:bounded-lp-optimality-test} and at most $m+d$
to recover the coordinates of $\mathbf x$ and $\mathbf y$.
Each RMDP input has dimension and constraint count polynomial in $m,d$.
The equalities added during recovery successively select nonempty
faces of the polytope in~\eqref{eq:bounded-lp-optimality-test}.
Each such face contains a vertex of that polytope, so every fixed
value is a coordinate of one of its vertices.
Cramer's rule bounds the bit lengths of these coordinates
polynomially in the original LP input length, so all RMDP inputs
have polynomial encoding length.
Since the RMDP algorithm is strongly polynomial, these calls use
polynomially many arithmetic operations and comparisons in $m,d$,
with intermediate rationals of polynomial bit length.
Forming the inputs and dividing the recovered vectors by $t$
satisfy the same bounds, so the entire procedure is strongly polynomial.
\end{proof}

\section{Proofs for weighted \texorpdfstring{$\ell_1$}{L1} and Wasserstein uncertainty}
\label{app:budget-tools}

We prove
Lemmas~\ref{lem:budget-minimizers} and~\ref{lem:budget-error},
which provide the computation costs and error bounds used in
Corollary~\ref{cor:budget-models}.

\subsection{Proof of Lemma~\ref{lem:budget-minimizers}(i)}
\label{app:weighted-oracle}

\begin{proof}
Given rational $\overline{\mathbf p}\in\Delta(\mathcal S)$,
$\mathbf w\ge\mathbf0$, $\rho\ge0$, and $\mathbf v\in\mathbb Q^S$,
we set $[\mathbf C]_{ii}=0$ and
$[\mathbf C]_{ij}=[\mathbf w]_i+[\mathbf w]_j$ for $i\ne j$,
computing these costs as needed.
To solve the transportation LP~\eqref{eq:transport-primal} with these costs,
we track the minimizing indices in~\eqref{eq:transport-dual} as
the multiplier $\lambda$ increases.
\looseness=-1 We find the first multiplier for which the corresponding transport cost
is at most $\rho$ and construct an optimal matrix from the minimizing
choices at that multiplier.

\vspace{5pt}
\noindent\textbf{Representing the functions $\psi_i$.}
In the definition of $\psi_i$ in~\eqref{eq:transport-dual},
the candidate for $j=i$ is $[\mathbf v]_i$, whereas the candidate
for $j\ne i$ is $[\mathbf v]_j+\lambda([\mathbf w]_i+[\mathbf w]_j)$.
Since $[\mathbf v]_i+2\lambda[\mathbf w]_i\ge[\mathbf v]_i$,
we obtain
\begin{equation}
  \begin{aligned}
    \psi_i(\lambda)
    &=\min\Biggl\{[\mathbf v]_i,\;
      \min_{j\in\mathcal S}
      \bigl\{[\mathbf v]_j+\lambda([\mathbf w]_i+[\mathbf w]_j)\bigr\}
      \Biggr\}=\min\Biggl\{[\mathbf v]_i,\;
      \lambda[\mathbf w]_i+
      \underbrace{\min_{j\in\mathcal S}
        \bigl\{[\mathbf v]_j+\lambda[\mathbf w]_j\bigr\}}
        _{=:\,\Psi(\lambda)}\Biggr\}.
  \end{aligned}
  \label{eq:weighted-row-envelope}
\end{equation}
We represent all $\psi_i$ in~\eqref{eq:weighted-row-envelope} in two steps.
\begin{itemize}[leftmargin=*]
\item \textbf{Step 1. Constructing $\Psi$.}
We compute the breakpoints of $\Psi$ and a minimizing index on each interval.
We process the lines in decreasing order of slope, keeping one line
with the smallest intercept for each slope.
A new line has smaller slope and can therefore replace only the final
portion of the minimum constructed so far.
Starting from the last piece, we remove a piece whenever the new line
lies at or below that piece at the piece's left endpoint.
The new piece starts at the intersection with the last remaining line,
or at zero if none remains.
Each line is added once and removed at most once, giving
$O(S\log(2+S))$ operations and comparisons including sorting.
The resulting $\Psi$ has nonincreasing slopes and at most $S-1$
positive breakpoints, since each line contributes at most one piece.
At zero and each positive breakpoint, we select the minimizing index
from the interval immediately to its right.

\item \textbf{Step 2. Computing the thresholds $\lambda_i$.}
For each $i$, we define $\lambda_i$ as the smallest $\lambda\ge0$
at which the constant $[\mathbf v]_i$ attains the minimum
in~\eqref{eq:weighted-row-envelope}, taking $\lambda_i=\infty$ if it never does.
Since $\mathbf w\ge\mathbf0$, the other term
$\lambda[\mathbf w]_i+\Psi(\lambda)$ is continuous and nondecreasing.
Thus a minimizing index for row $i$ is the index recorded for $\Psi$
when $\lambda<\lambda_i$, and $i$ when $\lambda\ge\lambda_i$.
We set $\lambda_i=0$ if $[\mathbf v]_i=\Psi(0)$.
Otherwise, binary search finds the first positive breakpoint of $\Psi$
satisfying
$\lambda[\mathbf w]_i+\Psi(\lambda)\ge[\mathbf v]_i$.
Each test evaluates $\Psi$ using its recorded index in $O(1)$ operations,
so the search costs $O(\log(2+S))$.
On the piece ending at this breakpoint, or on the final piece if no
breakpoint satisfies the inequality, we denote the recorded minimizing
index by $j(i)$.
Solving $[\mathbf v]_{j(i)}+\lambda([\mathbf w]_i+[\mathbf w]_{j(i)})=[\mathbf v]_i$
gives
$\lambda_i=\frac{[\mathbf v]_i-[\mathbf v]_{j(i)}}{[\mathbf w]_i+[\mathbf w]_{j(i)}}$
when the denominator is positive, and $\lambda_i=\infty$ otherwise.
Computing all thresholds and sorting their finite positive values together
with the breakpoints of $\Psi$ therefore costs $O(S\log(2+S))$
operations and comparisons.
We group equal values, obtaining a list of at most $2S-1$ distinct points.
\end{itemize}

\vspace{5pt}
\noindent\textbf{Finding the multiplier.}
We use the minimizing indices to construct a matrix and compare
its cost with the budget in~\eqref{eq:transport-primal}.
For a given $\lambda\ge0$, let $j$ be the selected minimizing index
of $\Psi(\lambda)$ and $\mathcal I=\{i:\lambda<\lambda_i\}$.
We define $\boldsymbol\Pi(\lambda)$ by setting
$[\boldsymbol\Pi(\lambda)]_{ij}=[\overline{\mathbf p}]_i$ for
$i\in\mathcal I$ and $[\boldsymbol\Pi(\lambda)]_{ii}=[\overline{\mathbf p}]_i$
for $i\notin\mathcal I$, with all remaining entries zero.
Then $\boldsymbol\Pi(\lambda)\ge0$ and
$\boldsymbol\Pi(\lambda)\mathbf1=\overline{\mathbf p}$, and each positive
entry occurs at an index attaining the corresponding $\psi_i(\lambda)$.
For $i\in\mathcal I$, the inequality
$\lambda[\mathbf w]_i+\Psi(\lambda)<[\mathbf v]_i$ ensures $j\ne i$.
Thus the matrix has cost
\begin{equation*}
  \langle\mathbf C,\boldsymbol\Pi(\lambda)\rangle
  =\sum_{i\in\mathcal I}[\overline{\mathbf p}]_i[\mathbf w]_i
   +[\mathbf w]_j\sum_{i\in\mathcal I}[\overline{\mathbf p}]_i.
\end{equation*}
The matrix $\boldsymbol\Pi(\lambda)$ is constant between consecutive
stored thresholds and breakpoints.
As $\lambda$ increases, $[\mathbf w]_j$ is nonincreasing and
$\mathcal I$ can only shrink, so the cost is nonincreasing.
On each interval where $\boldsymbol\Pi(\lambda)$ is constant,
the dual objective in~\eqref{eq:transport-dual} has slope
$\langle\mathbf C,\boldsymbol\Pi(\lambda)\rangle-\rho$.

If the cost at $\lambda=0$ is at most $\rho$, we set $\lambda^\star=0$.
Otherwise, we scan the merged list to find the first
$\lambda^\star>0$ at which the cost is at most $\rho$.
On the final interval of the merged list, every remaining
$i\in\mathcal I$ has $[\mathbf w]_i+[\mathbf w]_j=0$.
Otherwise, $[\mathbf v]_j+\lambda([\mathbf w]_i+[\mathbf w]_j)$
would eventually exceed $[\mathbf v]_i$, giving a further finite
threshold $\lambda_i$. Thus the final cost is zero and the search terminates.
During the scan, removing $i$ from $\mathcal I$ subtracts
$[\overline{\mathbf p}]_i[\mathbf w]_i$ and
$[\overline{\mathbf p}]_i$ from the two sums, while a change of $j$
only updates $[\mathbf w]_j$.
Each row is removed at most once and $j$ changes at most $S-1$ times,
so computing the two sums and scanning the list require
$O(S)$ operations and comparisons.

\vspace{5pt}
\noindent\textbf{Recovering an optimal distribution.}
If $\lambda^\star=0$, we set $\boldsymbol\Pi^\star=\boldsymbol\Pi(0)$.
Otherwise, we define
\begin{equation*}
  \boldsymbol\Pi^-=\lim_{\lambda\uparrow\lambda^\star}\boldsymbol\Pi(\lambda),
  \qquad
  \boldsymbol\Pi^+=\boldsymbol\Pi(\lambda^\star).
\end{equation*}
The left limit is the constant matrix on the interval immediately
to the left of $\lambda^\star$.
The matrices satisfy
$\langle\mathbf C,\boldsymbol\Pi^-\rangle>\rho
  \ge\langle\mathbf C,\boldsymbol\Pi^+\rangle$.
For each row $i$, the linear functions selected on the two adjacent
intervals equal $\psi_i$ on their respective intervals and, by continuity,
at their common endpoint $\lambda^\star$.
We choose the mixing coefficient so that the combined cost equals $\rho$,
\begin{equation}
  \boldsymbol\Pi^\star=\theta\boldsymbol\Pi^-+(1-\theta)\boldsymbol\Pi^+,
  \qquad
  \theta=\frac{\rho-\langle\mathbf C,\boldsymbol\Pi^+\rangle}
                   {\langle\mathbf C,\boldsymbol\Pi^--\boldsymbol\Pi^+\rangle}.
  \label{eq:transport-optimal-interpolation}
\end{equation}
Then $0\le\theta<1$ and
$\langle\mathbf C,\boldsymbol\Pi^\star\rangle=\rho$.
In both cases, $\boldsymbol\Pi^\star$ is feasible for~\eqref{eq:transport-primal}
and satisfies $\lambda^\star(\langle\mathbf C,\boldsymbol\Pi^\star\rangle-\rho)=0$.
Every positive entry $[\boldsymbol\Pi^\star]_{ij}$ satisfies
$[\mathbf v]_j+\lambda^\star[\mathbf C]_{ij}=\psi_i(\lambda^\star)$.
\looseness=-1 Multiplying by $[\boldsymbol\Pi^\star]_{ij}$ and summing over $i,j$ gives
\begin{equation*}
  \begin{aligned}
    \mathbf v^\top(\boldsymbol\Pi^\star)^\top\mathbf1
    &=\sum_i[\overline{\mathbf p}]_i\psi_i(\lambda^\star)
      -\lambda^\star\langle\mathbf C,\boldsymbol\Pi^\star\rangle=\sum_i[\overline{\mathbf p}]_i\psi_i(\lambda^\star)
      -\lambda^\star\rho.
  \end{aligned}
\end{equation*}
The last expression is the dual objective in~\eqref{eq:transport-dual},
so weak duality proves that $\boldsymbol\Pi^\star$ solves~\eqref{eq:transport-primal}.
By Lemma~\ref{lem:weighted-transport}, its column sums
$(\boldsymbol\Pi^\star)^\top\mathbf1$ give a minimizing transition distribution.
Forming the matrices and computing the costs, column sums, and objective
value require $O(S)$ operations, since $\boldsymbol\Pi^-$ and
$\boldsymbol\Pi^+$ each have at most $S$ nonzero entries and
$\boldsymbol\Pi^\star$ has at most $2S$.

In row $i$, the entries of $\boldsymbol\Pi^\pm$ and $\boldsymbol\Pi(0)$
are zero or $[\overline{\mathbf p}]_i$.
By~\eqref{eq:transport-optimal-interpolation}, the returned distribution
therefore has encoding length polynomial in that of
$(\overline{\mathbf p},\mathbf w,\rho)$, independently of $\mathbf v$.
The breakpoints and finite thresholds are ratios of sums or differences
of input entries, so all computed rationals have polynomial bit length
in the full input, including $\mathbf v$.
\end{proof}

\subsection{Proof of Lemma~\ref{lem:budget-minimizers}(ii)}
\label{app:wasserstein-oracle}

\begin{proof}
We solve~\eqref{eq:transport-primal} using the multiplier search and
interpolation in Appendix~\ref{app:weighted-oracle}, constructing
the functions $\psi_i$ separately for the given cost matrix $\mathbf C$.
We denote the nonempty feasible region by
\begin{equation}
  \mathcal K(\overline{\mathbf p},\rho)
  =\{\boldsymbol\Pi\ge0:\boldsymbol\Pi\mathbf1=\overline{\mathbf p},
                    \ \langle\mathbf C,\boldsymbol\Pi\rangle\le\rho\}.
  \label{eq:wasserstein-plan-family}
\end{equation}

\vspace{5pt}
\noindent\textbf{Representing the functions $\psi_i$.}
For each row $i$, we apply Step 1 in Appendix~\ref{app:weighted-oracle}
to the lines $\lambda\mapsto[\mathbf v]_j+\lambda[\mathbf C]_{ij}$,
$j\in\mathcal S$.
This computes the breakpoints and minimizing indices of $\psi_i$
in~\eqref{eq:transport-dual} in $O(S\log(2+S))$ operations and
comparisons per row.
Constructing these representations and sorting and grouping all positive
breakpoints therefore costs $O(S^2\log(2+S))$ in total.
We write the distinct positive breakpoints as
$0<\lambda_1<\cdots<\lambda_K$, where $K\le S(S-1)$, and set $\lambda_0=0$.
For $k=0,\ldots,K$, we denote by $j_i^k$ the minimizing index on the
interval immediately to the right of $\lambda_k$.
The selected lines have nonincreasing slopes, and the final line
has the smallest slope in the row, so
\begin{equation*}
  [\mathbf C]_{i,j_i^0}\ge [\mathbf C]_{i,j_i^1}\ge\cdots\ge [\mathbf C]_{i,j_i^K}
     =\min_l [\mathbf C]_{il}.
\end{equation*}
For $k\ge1$, both $j_i^{k-1}$ and $j_i^k$ attain $\psi_i(\lambda_k)$
by continuity.

We store these indices so that
$\sum_i[\overline{\mathbf p}]_i[\mathbf C]_{i,j_i^k}$ can be computed
in $O(S)$ operations for any $k$ and nominal distribution
$\overline{\mathbf p}$.
A balanced binary tree stores one index per leaf, with one leaf per row.
Updating an index copies only its leaf and the path to the root,
sharing unchanged subtrees, and costs $O(\log(2+S))$ operations.
We store the initial root and the root obtained after processing each breakpoint.
There are at most $S(S-1)$ index changes, so constructing and storing these trees requires
$O(S^2\log(2+S))$ operations and space.
Each stored tree has $O(S)$ nodes, so traversing it recovers all
$S$ indices and computes the sum in $O(S)$ operations.
The construction depends only on $\mathbf C$ and $\mathbf v$.

\vspace{5pt}
\noindent\textbf{Computing a minimizing distribution.}
For a given $\overline{\mathbf p}$ and $\rho$, if
$\sum_i[\overline{\mathbf p}]_i[\mathbf C]_{i,j_i^0}\le\rho$,
we set $[\boldsymbol\Pi^\star]_{i,j_i^0}=[\overline{\mathbf p}]_i$
for every $i$ and all remaining entries to zero.
This matrix is feasible and attains the dual objective
in~\eqref{eq:transport-dual} at $\lambda=0$, so it is optimal.
Otherwise, the sums $\sum_i[\overline{\mathbf p}]_i[\mathbf C]_{i,j_i^k}$
are nonincreasing in $k$, and nonemptiness of
$\mathcal K(\overline{\mathbf p},\rho)$ gives
$\sum_i[\overline{\mathbf p}]_i [\mathbf C]_{i,j_i^K}=\sum_i[\overline{\mathbf p}]_i\min_l [\mathbf C]_{il}\le\rho$.
Thus there is a first $k$ satisfying
\begin{equation}
  \sum_i[\overline{\mathbf p}]_i [\mathbf C]_{i,j_i^{k-1}}>\rho\ge\sum_i[\overline{\mathbf p}]_i [\mathbf C]_{i,j_i^k}.
  \label{eq:wasserstein-budget-crossing}
\end{equation}
Binary search tests $O(\log(2+S))$ breakpoints because $K\le S(S-1)$.
Each test computes the sum using the stored tree in $O(S)$ operations,
so the search costs $O(S\log(2+S))$ operations and comparisons.
We define $\boldsymbol\Pi^-$ and $\boldsymbol\Pi^+$ by
\begin{equation*}
  [\boldsymbol\Pi^-]_{i,j_i^{k-1}}
  =[\boldsymbol\Pi^+]_{i,j_i^k}
  =[\overline{\mathbf p}]_i,\qquad i\in\mathcal S,
\end{equation*}
with all remaining entries zero, and define $\boldsymbol\Pi^\star$
by~\eqref{eq:transport-optimal-interpolation}.
The cost inequalities in~\eqref{eq:wasserstein-budget-crossing} ensure
that $\boldsymbol\Pi^\star\ge0$, its row sums are $\overline{\mathbf p}$,
and its cost is $\rho$.
Every positive entry $[\boldsymbol\Pi^\star]_{ij}$ satisfies
$[\mathbf v]_j+\lambda_k[\mathbf C]_{ij}=\psi_i(\lambda_k)$.
The objective equality proved in Appendix~\ref{app:weighted-oracle}
therefore shows that $\boldsymbol\Pi^\star$ attains the dual objective
in~\eqref{eq:transport-dual} at $\lambda_k$ and is optimal.
In either case, $(\boldsymbol\Pi^\star)^\top\mathbf1$ is a minimizing
distribution. Forming the nonzero entries of $\boldsymbol\Pi^\star$
and computing its column sums and objective value take $O(S)$ operations,
since it has at most $2S$ nonzero entries.
Together with the search, this gives $O(S\log(2+S))$ operations and
comparisons per problem.

The matrix constructed at $\lambda=0$ and the matrices $\boldsymbol\Pi^\pm$
have entries zero or $[\overline{\mathbf p}]_i$ in row $i$.
By~\eqref{eq:transport-optimal-interpolation}, the returned distribution
therefore has encoding length polynomial in that of
$(\mathbf C,\overline{\mathbf p},\rho)$, independently of $\mathbf v$.
Each breakpoint is a ratio of differences of entries of $\mathbf v$
and $\mathbf C$, so all computed rationals have polynomial bit length
in the full input, including $\mathbf v$.
\end{proof}

\subsection{Proof of Lemma~\ref{lem:budget-error}(i)}
\label{app:weighted-progress}

\begin{proof}
Fix $\mathbf v$, $\mathbf w\ge\mathbf0$, and a weighted
$\ell_1$ set $\mathcal U$ with center $\overline{\mathbf p}$ and radius $\rho$.
Let $\mathbf u\ge\mathbf v$ and let $\mathbf p$
minimize $\mathbf u^\top\mathbf p$ over $\mathcal U$.
We construct $g$ satisfying~\eqref{eq:budget-objective-gap} and
bound the number of positive values it can take for fixed $\mathbf v,\mathbf w$.
For $\phi(\mathbf z)=\sum_i[\mathbf w]_i|[\mathbf z]_i-[\overline{\mathbf p}]_i|$,
we write $w_i^-$ and $w_i^+$ for its one-sided rates of change
at $\mathbf p$ when coordinate $i$ decreases or increases, respectively,
\begin{equation*}
  w_i^-=
  \begin{cases}-[\mathbf w]_i,&[\mathbf p]_i>[\overline{\mathbf p}]_i,\\ [\mathbf w]_i,&[\mathbf p]_i\le [\overline{\mathbf p}]_i,\end{cases}
  \qquad
  w_i^+=
  \begin{cases}-[\mathbf w]_i,&[\mathbf p]_i<[\overline{\mathbf p}]_i,\\ [\mathbf w]_i,&[\mathbf p]_i\ge [\overline{\mathbf p}]_i.\end{cases}
\end{equation*}
We write a direction from $\mathbf p$ as $\mathbf y-\mathbf x$,
with $\mathbf x,\mathbf y\ge\mathbf0$, and let $g$ be the optimal value of
\begin{equation}
  \begin{alignedat}{2}
    &\LPopt{max}{\mathbf x,\mathbf y}
    &\quad&\mathbf v^\top(\mathbf x-\mathbf y)\\
    &\text{\rm s.t.}
    &\quad&\left\{\begin{aligned}
      &\mathbf x,\mathbf y\ge\mathbf0,\quad [\mathbf x]_i=0\text{ if }[\mathbf p]_i=0,\\
      &\mathbf1^\top\mathbf x=\mathbf1^\top\mathbf y\le1,\\
      &\sum_i(w_i^-[\mathbf x]_i+w_i^+[\mathbf y]_i)\le0
           \text{ if }\phi(\mathbf p)=\rho.
    \end{aligned}\right.
  \end{alignedat}
  \label{eq:weighted-direction-certificate}
\end{equation}
The feasible set is compact and contains the origin, so the maximum
is attained and $g\ge0$.

\vspace{5pt}
\noindent\textbf{Verifying the bounds for $g$ in~\eqref{eq:budget-objective-gap}.}
Every feasible $(\mathbf x,\mathbf y)$ gives a direction
$\mathbf y-\mathbf x$ that preserves the simplex for sufficiently
small positive steps. The directional derivative of $\phi$ satisfies
\begin{equation}
  \begin{aligned}
    \phi'(\mathbf p;\mathbf y-\mathbf x)
    &=\sum_i\left\{\begin{aligned}
      &w_i^-[\mathbf x]_i+w_i^+[\mathbf y]_i,
        &&[\mathbf p]_i\ne[\overline{\mathbf p}]_i,\\
      &[\mathbf w]_i|[\mathbf y]_i-[\mathbf x]_i|,
        &&[\mathbf p]_i=[\overline{\mathbf p}]_i,
    \end{aligned}\right.\\
    &\le\sum_i(w_i^-[\mathbf x]_i+w_i^+[\mathbf y]_i).
  \end{aligned}
  \label{eq:weighted-directional-derivative}
\end{equation}
The inequality uses
$|[\mathbf y]_i-[\mathbf x]_i|\le[\mathbf x]_i+[\mathbf y]_i$
when $[\mathbf p]_i=[\overline{\mathbf p}]_i$, where
$w_i^-=w_i^+=[\mathbf w]_i$.
If $\phi(\mathbf p)=\rho$,~\eqref{eq:weighted-direction-certificate}--\eqref{eq:weighted-directional-derivative}
give $\phi'(\mathbf p;\mathbf y-\mathbf x)\le0$.
Along this direction, $\phi$ is affine for sufficiently small positive
steps and therefore cannot increase.
If $\phi(\mathbf p)<\rho$, continuity ensures that sufficiently small
steps still satisfy $\phi\le\rho$.
Optimality of $\mathbf p$ for $\mathbf u$ gives
$\mathbf u^\top(\mathbf y-\mathbf x)\ge0$.
Since $\mathbf u\ge\mathbf v$, it follows that
\begin{equation*}
  \mathbf v^\top(\mathbf x-\mathbf y)
  \le(\mathbf u-\mathbf v)^\top(\mathbf y-\mathbf x)
  \le\|\mathbf u-\mathbf v\|_\infty\mathbf1^\top\mathbf y
  \le\|\mathbf u-\mathbf v\|_\infty.
\end{equation*}
Maximizing over feasible $(\mathbf x,\mathbf y)$ gives
$g\le\|\mathbf u-\mathbf v\|_\infty$.
To bound the gap by $g$, we choose
$\mathbf p^\star\in\operatorname*{arg\,min}_{\mathbf z\in\mathcal U}
\mathbf v^\top\mathbf z$ and set
$[\mathbf x]_i=\max\{[\mathbf p]_i-[\mathbf p^\star]_i,0\}$, $[\mathbf y]_i=\max\{[\mathbf p^\star]_i-[\mathbf p]_i,0\}$.
Then $\mathbf1^\top\mathbf x=\mathbf1^\top\mathbf y\le1$ and
$[\mathbf x]_i=0$ whenever $[\mathbf p]_i=0$.
These choices also satisfy $[\mathbf x]_i[\mathbf y]_i=0$ for every $i$,
so equality holds in~\eqref{eq:weighted-directional-derivative}.
\looseness=-1 Convexity of $\phi$ then gives
\begin{equation*}
  \sum_i(w_i^-[\mathbf x]_i+w_i^+[\mathbf y]_i)
  =\phi'(\mathbf p;\mathbf p^\star-\mathbf p)
  \le\phi(\mathbf p^\star)-\phi(\mathbf p).
\end{equation*}
If $\phi(\mathbf p)=\rho$, the right-hand side is nonpositive
because $\mathbf p^\star\in\mathcal U$.
Thus $(\mathbf x,\mathbf y)$ is feasible
in~\eqref{eq:weighted-direction-certificate} and has objective
$\mathbf v^\top(\mathbf p-\mathbf p^\star)$,
proving~\eqref{eq:budget-objective-gap}.

\vspace{5pt}
\noindent\textbf{Counting the possible values of $g$.}
We fix $\mathbf v,\mathbf w$ and bound the number of possible positive optimal
values of~\eqref{eq:weighted-direction-certificate} as
$\overline{\mathbf p},\rho,\mathbf p$ vary.
These choices affect~\eqref{eq:weighted-direction-certificate} only
through the variables fixed at zero, whether the last constraint
is imposed, and the signs of its coefficients.
Suppose $g>0$. Every optimal pair satisfies
$\mathbf1^\top\mathbf x=\mathbf1^\top\mathbf y=1$.
If the common sum were less than one, scaling both vectors to make
it one would preserve feasibility and strictly increase the objective.
We choose an optimal pair $(\mathbf x,\mathbf y)$ with the fewest
positive coordinates and consider perturbations that keep its zero
coordinates fixed and satisfy
\begin{equation*}
  \left\{\begin{aligned}
    &\mathbf1^\top\Delta\mathbf x=0,\quad
      \mathbf1^\top\Delta\mathbf y=0,\\
    &\sum_i\bigl(w_i^-[\Delta\mathbf x]_i
                 +w_i^+[\Delta\mathbf y]_i\bigr)=0,
      \quad\text{if the last constraint in~\eqref{eq:weighted-direction-certificate} is imposed with equality.}
  \end{aligned}\right.
\end{equation*}
If a nonzero perturbation existed, sufficiently small steps in both
directions would be feasible. Optimality would therefore give
$\mathbf v^\top(\Delta\mathbf x-\Delta\mathbf y)=0$.
We could choose the sign of the perturbation so that the left-hand
side of the last constraint in~\eqref{eq:weighted-direction-certificate}
does not increase, then move in that direction until a positive coordinate
first becomes zero, preserving feasibility and the objective value.
This contradicts the choice of $(\mathbf x,\mathbf y)$.
Thus the two or three equations above admit no nonzero perturbation
supported on the positive coordinates of $(\mathbf x,\mathbf y)$,
so this pair has at most three positive coordinates.

With two positive coordinates, the mass equations give
$[\mathbf x]_i=[\mathbf y]_j=1$, yielding at most $S^2$ objective
values $[\mathbf v]_i-[\mathbf v]_j$.
With three, one vector has two positive coordinates and the other
has one, giving $2\binom{S}{2}S$ choices of positions.
The last constraint in~\eqref{eq:weighted-direction-certificate}
holds with equality. Each of its three coefficients is the corresponding
weight or its negative, giving at most $2^3$ sign choices.
Together with the two mass equations, this equality uniquely determines
the positive coordinates, as proved above. The coefficients depend
only on $\mathbf w$, and the right-hand sides are $1,1,0$.
Thus each combination of positions and signs corresponds to at most
one selected optimal pair and hence one objective value determined
by $\mathbf v,\mathbf w$.
Every positive $g$ is therefore among at most
\begin{equation*}
  S^2+2\binom{S}{2}S\,2^3\le8S^3
\end{equation*}
candidate values. Their positive values together with zero form a set
$\mathcal G$, determined only by $\mathbf v,\mathbf w$, that contains
every $g$ obtained from~\eqref{eq:weighted-direction-certificate}.
\end{proof}

\subsection{Proof of Lemma~\ref{lem:budget-error}(ii)}
\label{app:wasserstein-progress}

\begin{proof}
Fix $\mathbf v$, a nonnegative rational matrix $\mathbf C$,
and a nonempty transition set $\mathcal U$ represented by
$\mathcal K(\overline{\mathbf p},\rho)$ in~\eqref{eq:wasserstein-plan-family}.
For $\mathbf u\ge\mathbf v$, we fix
$\mathbf p\in\operatorname*{arg\,min}_{\mathbf z\in\mathcal U}
\mathbf u^\top\mathbf z$ and
choose $\widehat{\boldsymbol\Pi}\in\mathcal K(\overline{\mathbf p},\rho)$
with $\widehat{\boldsymbol\Pi}^\top\mathbf1=\mathbf p$.
Since the objective depends only on column sums,
$\widehat{\boldsymbol\Pi}$ minimizes
$\mathbf u^\top\boldsymbol\Pi^\top\mathbf1$ over
$\mathcal K(\overline{\mathbf p},\rho)$.
We use nonnegative matrices $\mathbf X,\mathbf Y\in\mathbb R^{S\times S}$
to write a direction from $\widehat{\boldsymbol\Pi}$ as
$\mathbf Y-\mathbf X$, and let $g$ be the optimal value of
\begin{equation}
  \begin{alignedat}{2}
    &\LPopt{max}{\mathbf X,\mathbf Y}
    &\quad&\mathbf v^\top(\mathbf X-\mathbf Y)^\top\mathbf1\\
    &\text{\rm s.t.}
    &\quad&\left\{\begin{aligned}
      &\mathbf X,\mathbf Y\ge\mathbf0,\quad
        [\mathbf X]_{ij}=0\text{ if }[\widehat{\boldsymbol\Pi}]_{ij}=0,\\
      &\mathbf X\mathbf1=\mathbf Y\mathbf1,\quad
        \mathbf1^\top\mathbf X\mathbf1\le1,\\
      &\langle\mathbf C,\mathbf Y-\mathbf X\rangle\le0
        \quad\text{if }\langle\mathbf C,\widehat{\boldsymbol\Pi}\rangle=\rho.
    \end{aligned}\right.
  \end{alignedat}
  \label{eq:wasserstein-transfer-certificate}
\end{equation}
The feasible set is compact and contains $(\mathbf0,\mathbf0)$,
so the maximum is attained and $g\ge0$.

\vspace{5pt}
\noindent\textbf{Verifying the bounds for $g$ in~\eqref{eq:budget-objective-gap}.}
For any feasible $(\mathbf X,\mathbf Y)$
in~\eqref{eq:wasserstein-transfer-certificate},
$(\mathbf Y-\mathbf X)\mathbf1=\mathbf0$, and
$[\mathbf Y-\mathbf X]_{ij}\ge0$ whenever
$[\widehat{\boldsymbol\Pi}]_{ij}=0$.
Thus sufficiently small positive steps from $\widehat{\boldsymbol\Pi}$
in this direction preserve nonnegativity and row sums.
The cost constraint is preserved by the last inequality
in~\eqref{eq:wasserstein-transfer-certificate} when
$\langle\mathbf C,\widehat{\boldsymbol\Pi}\rangle=\rho$,
and by continuity when $\langle\mathbf C,\widehat{\boldsymbol\Pi}\rangle<\rho$.
Optimality of $\widehat{\boldsymbol\Pi}$ gives
$\mathbf u^\top(\mathbf Y-\mathbf X)^\top\mathbf1\ge0$, hence
\begin{equation*}
  \mathbf v^\top(\mathbf X-\mathbf Y)^\top\mathbf1
  \le(\mathbf u-\mathbf v)^\top(\mathbf Y-\mathbf X)^\top\mathbf1
  \le\|\mathbf u-\mathbf v\|_\infty\mathbf1^\top\mathbf Y\mathbf1
  \le\|\mathbf u-\mathbf v\|_\infty.
\end{equation*}
The last two inequalities use $\mathbf u\ge\mathbf v$,
$\mathbf X,\mathbf Y\ge\mathbf0$, and
$\mathbf1^\top\mathbf Y\mathbf1=\mathbf1^\top\mathbf X\mathbf1\le1$.
Maximizing over $(\mathbf X,\mathbf Y)$ gives
$g\le\|\mathbf u-\mathbf v\|_\infty$.
To bound the gap by $g$, we choose a matrix
$\boldsymbol\Pi^\star$ minimizing
$\mathbf v^\top\boldsymbol\Pi^\top\mathbf1$ over
$\mathcal K(\overline{\mathbf p},\rho)$ and set
\begin{equation*}
  [\mathbf X]_{ij}=\max\{[\widehat{\boldsymbol\Pi}]_{ij}-[\boldsymbol\Pi^\star]_{ij},0\},
  \qquad
  [\mathbf Y]_{ij}=\max\{[\boldsymbol\Pi^\star]_{ij}-[\widehat{\boldsymbol\Pi}]_{ij},0\}.
\end{equation*}
Then $\mathbf Y-\mathbf X=\boldsymbol\Pi^\star-\widehat{\boldsymbol\Pi}$,
so $\mathbf X\mathbf1=\mathbf Y\mathbf1$ because
$\widehat{\boldsymbol\Pi}\mathbf1=\boldsymbol\Pi^\star\mathbf1$.
Also $\mathbf0\le\mathbf X\le\widehat{\boldsymbol\Pi}$,
which gives the zero-coordinate restrictions and
$\mathbf1^\top\mathbf X\mathbf1\le1$.
When $\langle\mathbf C,\widehat{\boldsymbol\Pi}\rangle=\rho$,
we also have
$\langle\mathbf C,\mathbf Y-\mathbf X\rangle
=\langle\mathbf C,\boldsymbol\Pi^\star\rangle-\rho\le0$.
Thus $(\mathbf X,\mathbf Y)$ is feasible
in~\eqref{eq:wasserstein-transfer-certificate} and has objective
$\mathbf v^\top\mathbf p-\min_{\mathbf z\in\mathcal U}\mathbf v^\top\mathbf z$,
proving the remaining inequality in~\eqref{eq:budget-objective-gap}.

\vspace{5pt}
\noindent\textbf{Counting the possible values of $g$.}
We write $\mathcal I=\{(i,j):[\widehat{\boldsymbol\Pi}]_{ij}>0\}$.
When $\langle\mathbf C,\widehat{\boldsymbol\Pi}\rangle=\rho$,
we introduce a multiplier $\lambda\ge0$ for the last constraint
in~\eqref{eq:wasserstein-transfer-certificate}.
For fixed $\lambda$, the Lagrangian objective is
\begin{equation*}
  \sum_{i,j}([\mathbf v]_j+\lambda[\mathbf C]_{ij})[\mathbf X]_{ij}
  -\sum_{i,k}([\mathbf v]_k+\lambda[\mathbf C]_{ik})[\mathbf Y]_{ik}.
\end{equation*}
We maximize this expression under the other constraints
in~\eqref{eq:wasserstein-transfer-certificate}.
For fixed $\mathbf X$, minimizing the second sum over $\mathbf Y$
gives $\sum_i\psi_i(\lambda)\sum_j[\mathbf X]_{ij}$,
by $\mathbf Y\mathbf1=\mathbf X\mathbf1$ and the definition
of $\psi_i$ in~\eqref{eq:transport-dual}.
Thus maximizing over $\mathbf Y$ leaves the objective
$\sum_{(i,j)\in\mathcal I}f_{(i,j)}(\lambda)[\mathbf X]_{ij}$, where
for $z=(i,j)\in\mathcal S\times\mathcal S$ we define
\begin{equation*}
  f_z(\lambda)
  =[\mathbf v]_j+\lambda[\mathbf C]_{ij}-\psi_i(\lambda)
  =\max_{k\in\mathcal S}
    \bigl\{[\mathbf v]_j-[\mathbf v]_k
       -\lambda([\mathbf C]_{ik}-[\mathbf C]_{ij})\bigr\},
  \qquad\lambda\ge0.
\end{equation*}
Each $f_z$ is convex and piecewise affine, and the term $k=j$ gives
$f_z\ge0$.
Since $\mathbf X\ge\mathbf0$ is supported on $\mathcal I$ and
$\mathbf1^\top\mathbf X\mathbf1\le1$, maximizing over $\mathbf X$
gives $\max_{z\in\mathcal I}f_z(\lambda)$.
Strong LP duality therefore gives
\begin{equation}
  g=\min_{\lambda\ge0}\max_{z\in\mathcal I}f_z(\lambda).
  \label{eq:wasserstein-transfer-dual}
\end{equation}
The functions $f_z$ depend only on $\mathbf v,\mathbf C$.
We show that the minimum in~\eqref{eq:wasserstein-transfer-dual}
is unchanged if the inner maximum is restricted to a suitable pair
of functions. We define
\begin{equation}
  \delta_{zz'}=\min_{\lambda\ge0}\max\{f_z(\lambda),f_{z'}(\lambda)\},
  \qquad z,z'\in\mathcal S\times\mathcal S.
  \label{eq:wasserstein-pair-defect}
\end{equation}
These minima are attained because the objectives are continuous,
bounded below, and have finitely many affine pieces.
We set $\delta^\star=\max_{z,z'\in\mathcal I}\delta_{zz'}$.
Restricting the maximum in~\eqref{eq:wasserstein-transfer-dual}
to any two functions gives $\delta^\star\le g$.
For the reverse inequality, each set
$\{\lambda\ge0:f_z(\lambda)\le\delta^\star\}$, $z\in\mathcal I$,
is a closed interval by continuity and convexity of $f_z$.
By~\eqref{eq:wasserstein-pair-defect}, for any $z,z'\in\mathcal I$,
there is a multiplier $\lambda\ge0$ satisfying
\begin{equation*}
  \max\{f_z(\lambda),f_{z'}(\lambda)\}
  =\delta_{zz'}\le\delta^\star.
\end{equation*}
This multiplier belongs to both corresponding intervals,
so the intervals intersect pairwise.
We choose an interval whose left endpoint is largest.
Since it intersects every other interval, this endpoint is no greater
than any right endpoint and therefore belongs to every interval.
Using this common multiplier in~\eqref{eq:wasserstein-transfer-dual}
gives $g\le\delta^\star$, proving
\begin{equation}
  g=\max_{z,z'\in\mathcal I}\delta_{zz'}.
  \label{eq:wasserstein-pair-representation}
\end{equation}
If $\langle\mathbf C,\widehat{\boldsymbol\Pi}\rangle<\rho$,
the cost constraint in~\eqref{eq:wasserstein-transfer-certificate}
is absent, giving
$g=\max_{(i,j)\in\mathcal I}\{[\mathbf v]_j-\min_k[\mathbf v]_k\}$.
Both here and in~\eqref{eq:wasserstein-pair-representation},
$g$ equals one of the terms in the maximum. It therefore belongs to
the finite set
\begin{equation*}
  \mathcal G=\{0\}
       \cup\{\delta_{zz'}:z,z'\in\mathcal S\times\mathcal S\}
       \cup\{[\mathbf v]_j-\min_k[\mathbf v]_k:j\in\mathcal S\}.
\end{equation*}
This set depends only on $\mathbf v,\mathbf C$.
There are $S^2$ possible indices $z$, hence at most $S^4$ pair values
and at most $S$ additional coordinate differences.
Thus $\mathcal G$ contains at most $S^4+S\le2S^4$ positive numbers.
\end{proof}

\end{document}